\documentclass{article}

    \usepackage[final]{neurips_2021}

\usepackage[utf8]{inputenc} 
\usepackage[T1]{fontenc}    
\usepackage[colorlinks = true,
            linkcolor = RoyalPurple,
            citecolor = MidnightBlue,
            urlcolor  = blue]{hyperref}
\usepackage{url}            
\usepackage{booktabs}       
\usepackage{amsfonts}       
\usepackage{nicefrac}       
\usepackage{microtype}      
\usepackage[dvipsnames]{xcolor}

\usepackage{multirow}
\usepackage[pdftex]{graphicx}
\usepackage{xcolor}
\usepackage{dirtytalk}
\usepackage{algorithm}
\usepackage{algorithmicx}  
\usepackage{algpseudocode}
\usepackage{bm}
\usepackage{amssymb}
\usepackage{mathtools}
\usepackage{amsthm}
\usepackage{amsmath}
\usepackage{commath}
\usepackage{cancel}
\usepackage{hyperref}
\usepackage{tikz}
\usetikzlibrary{arrows}
\usepackage{capt-of}
\usepackage{todonotes}
\usepackage{caption}
\usepackage{subcaption}
\usepackage{wrapfig}
\usepackage{commands}
\usepackage{math}
\usepackage{listings}
\usepackage{enumitem}
\usepackage{dirtytalk}
\usepackage{verbatim}

\usepackage{longtable}

\setcitestyle{numbers,square,citesep={,}}

\title{Causal Effect Inference for Structured Treatments}

\author{%
Jean Kaddour\thanks{Correspondence to \href{jean.kaddour.20@ucl.ac.uk}{jean.kaddour.20@ucl.ac.uk}}\\
Centre for Artificial Intelligence\\
University College London
\And
Yuchen Zhu\\
Centre for Artificial Intelligence\\
University College London
\And
Qi Liu\\
Department of Computer Science\\
University of Oxford
\And
Matt J. Kusner\\
Centre for Artificial Intelligence\\
University College London
\And
Ricardo Silva\\
Department of Statistical Science\\
University College London
}

\begin{document}

\maketitle
\begin{abstract} 
We address the estimation of conditional average treatment effects (CATEs) for structured treatments (e.g., graphs, images, texts). Given a weak condition on the effect, we propose the \emph{generalized Robinson decomposition}, which (i) isolates the causal estimand (reducing regularization bias), (ii) allows one to plug in arbitrary models for learning, and (iii) possesses a quasi-oracle convergence guarantee under mild assumptions. In experiments with small-world and molecular graphs we demonstrate that our approach outperforms prior work in CATE estimation.
\end{abstract}
\section{Introduction}
Estimating feature-level causal effects, so-called \textit{conditional average treatment effects} (CATEs), from observational data is a fundamental problem across many domains. Examples include understanding the effects of non-pharmaceutical interventions on the transmission of COVID-19 in a specific region \cite{flaxman2020estimating}, how school meal programs impact child health \cite{national20072008}, and the effects of chemotherapy drugs on cancer patients \cite{schwab2020doseresponse}. Supervised learning methods face two challenges in such settings: (i) \textit{missing interventions}, the fact that we only observe one treatment for each individual means models must extrapolate to new treatments without access to ground truth, and (ii) \textit{confounding factors} that affect both treatment assignment and the outcome means that extrapolation from observation to intervention requires assumptions.
Many approaches have been proposed to overcome these issues \cite{limits, alaa2017bayesian,arbour2020permutation, athey2016recursive, athey2019estimating, grf, SCIGAN,  DML, curth2021nonparametric, hahn2020bayesian, hatt2021estimating,hill, imbens, jesson2020identifying, jesson2021quantifying,deepmatch, kennedy2017nonparametric, x-learner, BVNICE, nabi2020semiparametric,  nie2021vcnet, r-learner, pollmann2020causal, schwab2020doseresponse, shi2020invariant,  ricardo,  wager2018estimation, SITE, ZhuBoosting}.

In many cases, treatments are naturally \emph{structured}. For instance, a drug is commonly represented by its molecular structure (graph), the nutritional content of a meal as a food label (text), and geographic regions affected by a new policy as a map (image). Taking this structure into account can provide several advantages: (i) higher data-efficiency, (ii) capability to work with many treatments, and (iii) generalizing to unseen treatments during test time. However, the vast majority of prior work operates on either binary or continuous scalar treatments (structured treatments are rarely considered, a notable exception to this trend is \citet{harada2020graphite} which we describe in Section~\ref{rw:graphite}). 

To estimate CATEs with structured interventions, our contributions include: 

\begin{itemize}[ wide = 5pt, leftmargin = *]
    \item \textbf{Generalized Robinson decomposition (GRD):} A generalization of the Robinson decomposition \cite{r-decomposition} to treatments that can be vectorized as a continuous embedding.
    This GRD reveals a learnable pseudo-outcome target that isolates the causal component of the observed signal by partialling out confounding associations. Further, it allows one to learn the nuisance and target functions using any supervised learning method, thus extending recent work on \emph{plug-in estimators} \cite{r-learner,x-learner}. 
    \item \textbf{Quasi-oracle convergence guarantee:} A  result that shows that given access to estimators of certain nuisance functions, as long as the estimates converge at an $O(n^{-1/4})$ rate, the target estimator for the CATE achieves the same error bounds as an oracle who has ground-truth knowledge of both nuisance components, the propensity features, and conditional mean outcome. 
    \item \textbf{Structured Intervention Networks (SIN)}: A practical algorithm using GRD, representation learning, and alternating gradient descent. Our PyTorch \cite{pytorch} implementation is online.\footnote{\href{https://github.com/JeanKaddour/SIN}{https://github.com/JeanKaddour/SIN}}
    \item \textbf{Evaluation metrics} designed for structured treatments. Since previous evaluation protocols of CATE estimators have mostly focused on binary or scalar-continuous treatment settings, we believe that our proposed evaluation metrics can be useful for comparing future work.
    \item \textbf{Experimental results} with graph treatments in which SIN outperforms previous approaches. 
\end{itemize}

\begin{figure}
    \centering
    \includegraphics[width=.9\textwidth]{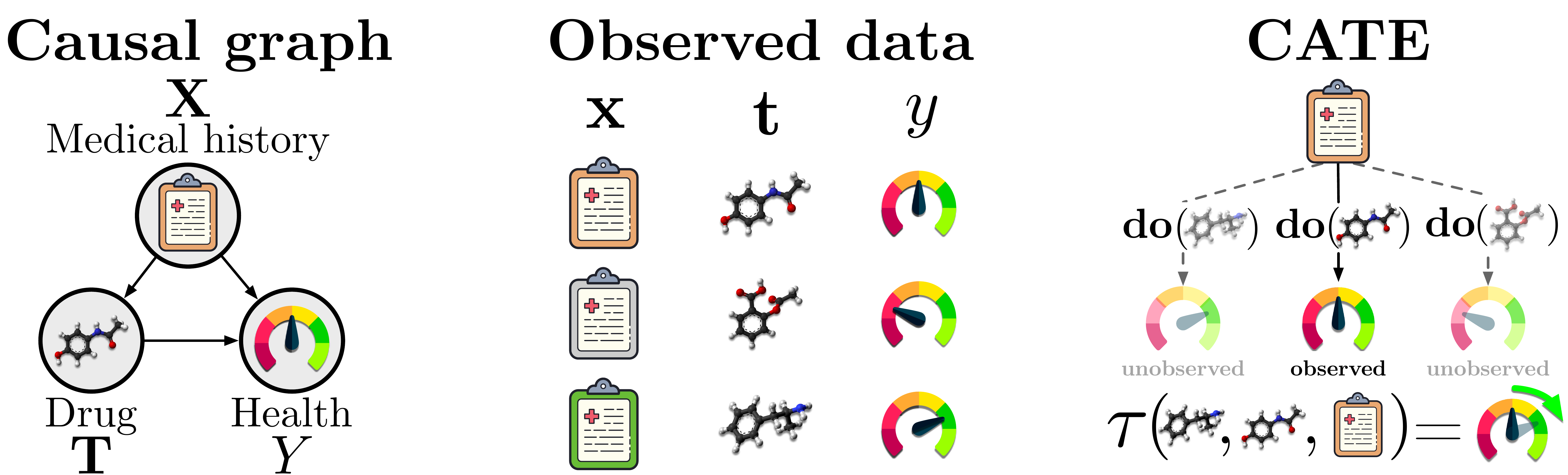}
    \caption{\textbf{Illustration of CATE estimation with structured treatments (e.g., molecular graphs).} \emph{Left}: Problem setup with features $\matr X$, treatment $\matr T$, and outcome $Y$. \emph{Center}: Observations the estimator has access to, typically containing only one outcome per individual. \emph{Right}: The CATE is the difference between the expected outcomes given a fixed individual and a pair of treatments.}
    \label{fig:problem_setup}
\end{figure}

\section{Related Work}
\label{rw:graphite}
Closest to our work is GraphITE \cite{harada2020graphite}, a method that learns representations of graph interventions for CATE estimation. They propose to minimize prediction loss plus a regularization term that aims to control for confounding based on the Hilbert-Schmidt Independence Criterion (HSIC) \cite{gretton2007kernel}. This technique suffers from two drawbacks: (i) the HSIC requires multiplication of kernel matrices and scales quadratically in the batch size; (ii) selecting the HSIC kernel hyper-parameter is not straightforward, as ground-truth CATEs are never observed, and empirical loss does not bound CATE estimation error \cite{limits}. We discuss other related work not on structured treatments in Appendix~\ref{sec:orw}.
\section{Preliminaries}
\subsection{Conditional Average Treatment Effects (CATEs)}\label{sec:CATE}
Imagine a dataset where each example $\left(\vect{x}_{i}, \vect t_{i}, y_{i}\right) \in \mathcal{D}$ represents a hospital patient's medical history record $\vect{x}_i$, prescribed drug treatment $\vect t_i$, and health outcome $y_i$, as illustrated in Figure~\ref{fig:problem_setup} (\emph{Center}). Further, we wish to understand how changing the treatment changes a patient's health outcome. The CATE, $\tau\left(\vect t^{\prime},\vect t_i, \vect{x}_i\right)$, describes the expected change in outcome for individuals with history $\vect{x}_i$, when treatment $\vect t_i$ is replaced by $\vect t^{\prime}$, depicted in Figure~\ref{fig:problem_setup} (\emph{Right}). In real-world scenarios, we only observe one outcome for each patient at one treatment level. Further, the patient's pre-treatment health conditions $\vect{x}_{i}$ influence both the doctor's treatment prescription and outcome, thereby \emph{confounding} the effect of the treatment on the outcome. 

Formally, we have the dataset $\mathcal{D}\!=\!\left\{\left(\vect{x}_{i}, \vect t_{i}, y_{i}\right)\right\}_{i=1}^{n}$ sampled from a joint distribution $p\left(\matr X, \matr T, Y\right)$, where $Y = f\left(\matr X, \matr T\right) + \varepsilon$, as depicted in Figure~\ref{fig:problem_setup} (\emph{Left}). We define the causal effect of fixing \emph{treatment} variable $\matr T \in \mathcal{T}$ to a value $\vect t$ on \emph{outcome} variable $Y \in \R$ using the do-operator \cite{pearl2000causality} as $\mathbb{E}\left[Y \mid \doo{\matr T = \vect t}\right]$. Crucially, this estimate differs from the conditional expectation $\mathbb{E}\left[Y \mid \matr T = \vect t\right]$ in that it describes the effect of an external entity \emph{intervening} on $\matr T$ by fixing it to a value $\vect t$ (removing the edge $\matr X \rightarrow \matr T$). 
We further condition on pre-treatment \emph{covariates} $\matr X$ to define the conditional causal estimand $\mathbb{E}\left[Y \mid \matr X = \vect x, \doo{\matr T = \vect t}\right]$. The \emph{conditional average treatment effect} (CATE) is the difference between expected outcomes at different treatment values $\vect t, \vect t^{\prime}$ for given covariates $\vect x$,
\begin{align}
\tau(\vect t^{\prime}, \vect t, \vect x) \triangleq \underbrace{\E\left[Y \mid \matr{X} = \vect x, \doo{\matr T = \vect t^{\prime}} \right]}_{=:\mu_{\vect t^{\prime}}(\vect x)} - \underbrace{\E\left[Y \mid \matr{X} = \vect x, \doo{\matr T = \vect t} \right]}_{=:\mu_{\vect t}(\vect x)},\label{eq:CATE}
\end{align}
where $\mu_{\vect t}\left(\vect x\right)$ is defined as the \emph{expected outcome} for a covariate vector $\vect x$ under treatment $\vect t$.

Because we do not observe both treatments $\vect t, \vect t'$ for a single covariate $\vect x$, we need to make assumptions that allow us to identify the CATE from observational data. 

\begin{assumption}(Unconfoundedness)
There are no confounders of the effect between $\mathbf T$ and $Y$ beyond $\matr X$. Therefore, $Pr\left(Y \leq y \mid \vect x, \doo{\vect t}\right) = Pr\left(Y \leq y~\mid~\vect x, \vect t \right)$, for all $(\vect x, \vect t, y)$.
\label{ass:unconfoundedness}
\end{assumption}

\begin{assumption}(Overlap) 
It holds that $0 < p\left(\vect t \mid \vect x\right)<1$, for all $(\vect x, \vect t)$.\label{ass:overlap}
\end{assumption}

Assumption~\ref{ass:overlap} means that all sub-populations have some probability of receiving any value of treatment (otherwise, some $\tau(\mathbf t', \mathbf t, \mathbf x)$ may be undefined or impossible to estimate.) 
These assumptions allow us to estimate the causal quantity $\tau(\vect t^{\prime}, \vect t, \vect x)$ through statistical estimands: 
\begin{align}
\tau\left(\vect t^{\prime}, \vect t, \vect x\right) = \mu_{\vect t^{\prime}}\left(\vect x\right) - \mu_{\vect t}\left(\vect x\right) = \E\left[Y\mid\matr X = \vect x, \matr T= \vect t^{\prime} \right] - \E\left[Y\mid \matr X = \vect x, \matr T= \vect t \right].
\end{align}

While one can model $\mu_t(\matr x)$ with regression models, such approaches suffer from bias \cite{DML, kennedy2020optimal,  x-learner} due to two factors: (i) associations between $\matr X$ and $\matr T$, due to confounding, makes it hard to identify the distinct contributions of $\matr X$ and $\matr T$ on $Y$, and (ii) regularization for predictive performance can harm effect estimation. Mitigating these biases relies on exposing and removing \emph{nuisance components}. This transforms the optimization into a (regularized) regression problem that isolates the causal effect.

\subsection{Robinson Decomposition}
\label{subsec:robinson}
One way to formulate such nuisance components is via the \emph{Robinson decomposition} \cite{r-decomposition}. Originally a reformulation of the CATE for binary treatments, it was used by the \textit{R-learner} \cite{r-learner} to construct a plug-in estimator. The R-learner exploits the decomposition by partialling out the confounding of $\matr X$ on $\matr T$ and $Y$. It also isolates the CATE, thereby removing regularization bias. 

Let the treatment variable be $T \in \{0,1 \}$ and the outcome model $p\left(y \mid \matr{x}, \matr{t}\right)$ parameterized as
\begin{equation}
 Y = f(\matr{X}, T) + \varepsilon \equiv \mu_0(\mathbf X) + T \times \tau_b(\mathbf X) + \varepsilon,
\label{eq:binmodel}
\end{equation}
\noindent where we define error term $\varepsilon$ such that $\mathbb E\left[\varepsilon \mid \vect x, \vect t\right] = \mathbb E\left[\varepsilon \mid \vect x\right] = 0$, and $\tau_b\left(\matr x\right) \triangleq \tau\left(1, 0, \matr x\right)$.

Define the \emph{propensity score} \cite{rosenbaum1983central} $e\left(\matr x\right) \triangleq p\left(T=1 \mid \boldsymbol{\matr x}\right)$ and the \emph{conditional mean outcome} as 
\begin{align}
    m\left(\matr x\right) \triangleq \E\left[Y \mid \matr x\right] = \mu_0\left(\mathbf x\right) + e\left(\matr x\right)\bincate\left(\mathbf x\right).\label{eq:cmo}
\end{align}

From model (\ref{eq:binmodel}) and the previous definitions, it follows that
\begin{equation}
Y - m\left(\matr X\right) = \left(T - e\left(\matr X\right)\right)\tau_{\text{b}}\left(\matr X\right) + \varepsilon,
\end{equation}
allowing us to define the estimator 
\begin{align}
    \bincateest\left(\vect \cdot \right)=\argmin_{\bincate} \Bigg \{\frac{1}{n} \sum_{i = 1}^n \left(\tilde{y}_i-\tilde t_i \times \bincate\left(\vect{x}_i\right)\right)^2 + \Lambda\left (\bincate\left (\cdot\right)\right ) \Bigg\}, \label{eq:r-learner}
\end{align} where $\tilde y_i \triangleq y_i - \widehat m\left(\matr x_i\right)$ and $\tilde t_i \triangleq t_i - \widehat e\left(\matr x_i\right)$ are pseudo-data points defined through estimated nuisance functions $\widehat m(\cdot), \widehat e(\cdot)$, which can be learned separately with any supervised learning algorithm. 

\section{The Generalized Robinson Decomposition}\label{sec:method}
Our goal is to estimate the CATE $\tau(\vect t^{\prime}, \vect t, \vect x)$ for structured interventions $\vect t^{\prime}, \vect t$ (e.g., graphs, images, text) while accounting for the confounding of $\matr X$ on $\matr T$ and $Y$. Inspired by the Robinson decomposition, which has enabled flexible CATE estimation for binary treatments \cite{grf, DML, BVNICE, r-learner}, we propose the \emph{Generalized Robinson Decomposition} from which we extract a pseudo-outcome that targets the causal effect. We demonstrate the usefulness of this decomposition from both a theoretical view (quasi-oracle convergence rate in Section~\ref{sec:qo}) and practical view (\emph{Structured Intervention Networks} in Section~\ref{sec:sin}). For details on its motivation and derivation, we refer the reader to Appendix~\ref{app:grd}. 

\subsection{Generalizing the Robinson Decomposition}
\label{sec:general_rd}
To generalize the Robinson decomposition to structured treatments, we introduce two concepts: (a) we assume that the causal effect is a \emph{product effect}: the outcome function $f^*\left(\matr X, \matr T\right)$ can be written as an inner product of two separate functionals, one over the covariates and one over the treatment, and (b) \emph{propensity features}, which partial out the effects from the covariates on the treatment features. Similar techniques have been previously shown to add to the robustness of estimation \cite{DML,r-learner}.
\begin{assumption}(Product effect) \label{ass:pipe} 
We consider the following partial parameterization of $p(y \mid \matr x, \matr t)$,
\begin{align}
    Y = g\left(\vect X\right)^{\top} h\left(\matr T\right) +   \varepsilon, 
    \label{eq:PE} 
\end{align} where $g: \mathcal{X} \rightarrow \mathbb{R}^{d}, h: \mathcal{T} \rightarrow \mathbb{R}^d$ and $\E[\varepsilon \mid \matr x, \matr t] = \E\left[\varepsilon \mid\matr x\right] = 0,$ for all  $\left(\matr x, \matr t\right) \in \mathcal X \times \mathcal T$.
\end{assumption}
This assumption is mild, as we can formally justify its universality. The following asserts that provided we allow the dimensionality of $g$ and $h$ to grow, we may approximate any arbitrary bounded continuous functions in $\mathcal{C}\left(\mathcal{X\times T}\right)$ where $\mathcal{X \times T}$ is compact. 

\begin{proposition}(Universality of product effect) \label{prop:prod_decomp}
Let $\mathcal{H_{X \times T}}$ be a Reproducing Kernel Hilbert Space (RKHS) on the set $\mathcal{X \times T}$ with universal kernel $k$. For any $\delta > 0$, and any $ f \in \mathcal{H_{X\times T}}$, there is a $d \in \mathbb{N}$ such that there exist two $d$-dimensional vector fields $g: \mathcal{X} \rightarrow \mathbb{R}^d$ and $h: \mathcal{T} \rightarrow \mathbb{R}^d$, where $\|f - g^{\top}h\|_{L_2(P_{\mathcal{X \times T}})} \leq \delta$. (Proof in Appendix~\ref{sec:universality})
\end{proposition}
 
This assumption allows us to simplify the expression of the CATE for treatments $\vect t^{\prime}, \vect t$, given  $\vect x$,
\begin{align}
\tau\left(\vect t^{\prime}, \vect t, \vect x\right) =  g\left(\matr x\right)^{\top} \left(h\left(\matr t^{\prime}\right) -  h\left(\matr t\right)\right).
\end{align}
Define \emph{propensity features} $e^h\left(\matr x\right) \triangleq \E \left[ h\left(\matr T\right) \mid \matr x \right ]$ and $m\left (\matr x\right) \triangleq \E \left [ Y \mid \matr x \right ] = g\left(\matr x\right)^{\top} e^h\left(\matr x\right).$

Following the same steps as in Section~\ref{subsec:robinson}, the Generalized Robinson Decomposition for \eqr{eq:PE} is 
\begin{empheq}[box=\eqbox]{align}
    Y -  m\left(\matr X\right)  = g\left(\matr X\right)^{\top} \left (h\left(\matr T\right) - e^h\left(\matr X\right)\right) + \varepsilon.
    \label{eq:robinson_decomp_gin}
\end{empheq}

Given nuisance estimates $\widehat m(\cdot), \widehat e^h(\cdot)$, we can use this decomposition to derive an optimization problem for $h(\cdot),g(\cdot)$ (note $\widehat e^h(\cdot)$ implicitly depends on $h(\cdot)$, we address this dependence in Section~\ref{sec:sin}).

\begin{align}
    \widehat g\left(\cdot\right), \widehat h\left(\cdot\right) \triangleq \argmin_{g,h} \left \{ \frac{1}{n} \sum_{i = 1}^n \left (Y_i -  \widehat m\left(\vect X_i\right) - g\left(\vect X_i\right)^{\top}\left(h\left(\vect T_i\right) - \widehat e^h\left(\vect X_i\right)\right) \right)^2 + \Lambda\left(g\left(\cdot\right)\right)\right \} \label{eq:generalized_robinson_loss} 
\end{align}

\subsection{Quasi-oracle error bound of Generalized Robinson Decomposition}\label{sec:qo}
We establish the main theoretical result of our paper: a \emph{quasi-oracle convergence guarantee} for the Generalized Robinson Decomposition under a finite-basis representation of the outcome function. This result is analogous to the R-learner for binary CATEs \cite{r-learner}: when the true $e\left(\vect \cdot \right), m\left(\vect \cdot \right)$ are unknown, and we only have access to the estimators $\widehat e\left(\vect \cdot \right), \widehat m\left(\vect \cdot \right)$, then as long as the estimates converge at $n^{-1/4}$ rate, the estimator $\bincateest\left(\vect \cdot \right)$ achieves the same error bounds as an \textit{oracle} who has ground-truth knowledge of these two nuisance components. 

More formally, provided the nuisance estimators $\widehat{m}(\cdot)$ and $\widehat{e}^h(\cdot)$ converge at an $O\left(n^{-1/4}\right)$ rate, our CATE estimator will converge at an $\widetilde{O}(n^{-\frac{1}{2(1+p)}})$ rate for arbitrarily small $p>0$, recovering the parametric convergence rate for when the true $m(\cdot)$ and $e^h(\cdot)$ are provided as oracle quantities. 

Our analysis assumes that the outcome $\E\left[Y\mid\mx=\mathbf{x}, \mt=\mathbf{t}\right]$ can be written as a linear combination of fixed basis functions. By Proposition~\ref{prop:prod_decomp}, as long as we have enough basis functions, this representation is flexible enough to capture the true outcome function. 

\begin{assumption}\label{assump:true_fun_approx}
Let $\boldsymbol{\alpha}(\vect X) \in \R^{d_{\boldsymbol{\alpha}}}$, $\boldsymbol{\beta}(\vect T) \in \R^{d_{\boldsymbol{\beta}}}$ be fixed, known orthonormal basis features on $\vect X \in \R^{d_{\mathbf x}}$, $\vect T \in \R^{d_{\mathbf t}}$, respectively. The true outcome function $f^*(\mathbf{x}, \mathbf{t}) = \E[Y\mid\mx=\mathbf{x}, \mt=\mathbf{t}]$ can be written as $f^*(\mathbf{x}, \mathbf{t})= \boldsymbol{\alpha}^{\top}(\mathbf{x})\thet^* \boldsymbol{\beta}(\mathbf{t})$ for some (unknown) matrix of coefficients $\thet^*$.
\end{assumption}
Note that by setting $g = \boldsymbol{\alpha}^{\top}\Theta^*$ and $h = \boldsymbol{\beta}$, we recover \eqr{eq:PE}. Additionally, we will need overlap in the basis features $\boldsymbol{\alpha}(\mathcal{X}), \boldsymbol{\beta}(\mathcal{T})$.
\begin{assumption}[Overlap in features] \label{assump:overlap_mainbody}
The marginal distribution of features $\mathcal{P}_{\boldsymbol{\alpha}(\mathcal{X})\times\boldsymbol{\beta}(\mathcal{T})}$ is positive, i.e. $\operatorname{supp}[\mathcal{P}_{\boldsymbol{\alpha}(\mathcal{X})\times \boldsymbol{\beta}(\mathcal{T})}]= \boldsymbol{\alpha}(\mathcal{X}) \times \boldsymbol{\beta}(\mathcal{T})$.
\end{assumption}
Assumption~\ref{assump:overlap_mainbody} is typically weaker than requiring overlap in $\vect X$ and $\vect T$, i.e., when $d_{\boldsymbol{\alpha}}, d_{\boldsymbol{\beta}} \ll d_{\mathbf x}, d_{\mathbf t}$. 

With further technical assumptions specified in Appendix \ref{sec:rates}, we establish the following theorem.

\begin{theorem}\label{theorem:main}
Let $\thet^*$ denote the representer of the true outcome function. Suppose Assumptions \ref{assump:overlap_mainbody}, \ref{assump:boundedness}, and \ref{assump:true_fun_approx} hold. Moreover, suppose that the propensity estimate $\widehat{e}^h$ is uniformly consistent, \begin{equation}
    \sup_{\mathbf{x} \in \mathcal{X}}\|\widehat{e}^h(\vect x) - e^h(\vect x)\| \rightarrow_p 0
\end{equation}
and the $L_2$ errors converge at rate
\begin{equation}
    \E\left[\left\{\estmean - \mean\right\}^2\right], \E\left[\|\widehat{e}^{h}\left(\mx\right) - e^h\left(\mx\right)\|^2\right] = \mathcal{O}(a^2_n) \label{eq:cf_rates}
\end{equation}
for some sequence $a_n \rightarrow 0$, where $(a_n)$ is such that $a_n = O(n^{-\kappa})$ with $\kappa > \frac{1}{4}$. Further, we define the regret as the excess risk \begin{align}
    R\left(\widehat{\thet}_n\right) \triangleq L\left(\widehat{\thet}_n\right) - L\left(\thet^*\right), \; L\left(\thet\right) \triangleq \E\left[\left\{\left(Y - m^*\left(\mx\right)\right) - \boldsymbol{\alpha}\left(\mx\right) \thet \left(\boldsymbol{\beta}\left(\mt\right) - e^h\left(\mx\right)\right)\right\}^2\right].
\end{align} Suppose that we obtain $\widehat \thet_n$ via a penalized basis function regression variant of the Generalized Robinson Decomposition, with a properly chosen penalty $\Lambda_n\left(\|\widehat \thet_n\|_2\right)$ (specified in the proof). Then, $\widehat \thet_n$ satisfies the regret bound: $R\left(\widehat{\thet}_n\right) = \widetilde{O}(r_n^2)$ with $r_n = n^{-\frac{1}{2(1+p)}}$ for arbitrarily small $p > 0$. 
\end{theorem}

\section{Structured Intervention Networks} \label{sec:sin}
We introduce \emph{Structured Intervention Networks} (SIN), a two-stage training algorithm for neural networks, which enables flexibility in learning complex causal relationships, and scalability to large data-sets. This implementation of GRD strikes a balance between theory and practice: while we assumed fixed basis-functions in Section~\ref{sec:qo}, in practice, we often need to learn the feature maps from data. We leave the convergence analysis of this representation learning setting for future work.
\subsection{Training Algorithm}
We propose to simultaneously learn feature maps $\widehat g\left(\matr X\right), \widehat h\left(\matr T\right)$ using alternating gradient descent, so that they can adapt to each other. A remaining challenge is that learning $\widehat e^h(\matr X)$ is now entangled with learning $\widehat h\left(\matr T\right)$. While the R-learner is based on the idea of \emph{cross-fitting}, where at each data point $i$ we pick estimates of the nuisances that do not use that data point,
we introduce a pragmatic representation learning approach for $(\widehat g, \widehat h)$ that does not use cross-fitting\footnote{We could in principle use cross-fitting for $\widehat e^h$, although the loop between fitting $\widehat h$ alternating with $\widehat e^h$ would break the overall independence between $\widehat e^h_i(\matr X)$ and data point $i$. While it is possible that cross-fitting for $\widehat e^h$ is still beneficial in this case, for simplicity and for computational savings, we did not implement it.}.

We learn surrogate models for the mean outcome and propensity features $\widehat m_{\mpa}(\matr X)$ and $\widehat e^h_{\vec \epa}(\matr X)$ with parameters $\mpa \in \R^{d_{\mpa}}, \epa \in \R^{d_{\epa}}$, as well as feature maps for covariates and treatments $\widehat g_{\xpa}(\matr X), \widehat h_{\tpa}(\matr T)$, parameterized by $\xpa \in \R^{d_{\xpa}}, \tpa \in \R^{d_{\tpa}}$. We denote regularizers by $\Lambda \left ( \cdot \right )$. Figure~\ref{fig:algorithm} summarizes the algorithm. As the mean outcome model $\com$ does not depend on the other components, we learn it separately in Stage 1. In Stage 2, we alternate between learning $\xpa, \tpa, \epa$.

\paragraph{Stage 1:} Learn parameters $\mpa$ of the mean outcome model $\com$ based on the objective 
\begin{align}
    J_{m}\left(\mpa\right) &= \frac{1}{m} \sum_{i = 1}^m \left (y_i - \widehat m_{\mpa}\left(\vect{x}_i\right)\right)^{2} +  \Lambda \left(\mpa \right ),
\end{align}which relies only on covariates and outcome data $\mathcal{D}_1 := \left \{\left (\vect x_i, y_i \right)\right \}^{m}_{i=1}$. 

\paragraph{Stage 2:} Learn parameters $\xpa, \tpa$ for the covariates and treatments feature maps $\xfeat, \tfeat$, as well as parameters $\epa$ for the propensity features $\pf$. 

\begin{align}
    J_{g, h}\left(\tpa, \xpa\right) &= \frac{1}{n} \sum_{i = 1}^n \left(y_i - \left \{ \widehat m_{\mpa}\left(\vect x_i\right) + \widehat g_{\xpa}\left(\vect x_i\right)^{\top}\left(\widehat h_{\tpa}\left(\vect t_i\right) - \widehat e^h_{\epa}\left(\vect x_i\right)\right) \right \}\right)^{2} + \Lambda\left( \xpa \right ) + \Lambda \left( \tpa \right)\label{eq:gin_loss}. 
\end{align} This loss hinges on $\pf$, which needs to be learned by 
\begin{align}
    J_{e^h}\left(\epa\right) &= \sum_{i = 1}^n \; \norm{\widehat h_{\tpa}\left(\vect t_i\right) - \widehat e^h_{\epa}\left(\vect{x}_i\right)}_{2}^{2} + \Lambda \left(\epa \right ),
\end{align} note again the dependence on $\tfeat$. While it may be tempting to learn $\xpa, \tpa$ and $\epa$ jointly, they have fundamentally different objectives ($\pf$ is defined as an estimate of the expectation $\E \left[ h\left(\matr T\right) \mid \matr x \right ]$). 
Therefore, we employ an alternating optimization procedure, where we take $k \in \{1, \dots, K\}$ optimization steps for $\xpa, \tpa$ towards $J_{g, h}\left(\xpa, \tpa\right)$ and one step for learning $\epa$. We observe that setting $K>1$, i.e. updating $\xpa, \tpa$ more frequently than $\epa$, stabilizes the training process.

\begin{figure}
\centering
\begin{subfigure}[t]{0.48\textwidth}
\begin{algorithm}[H]
\caption{SIN Training.}
\textbf{Input}: Stage 1 data $\mathcal{D}_1 := \{(\vect x_i, y_i ) \}^{m}_{i=1}$, Stage 2 data $\mathcal{D}_2 := \{(\vect x_i, \vect t_i, y_i ) \}^{n}_{i=1}$ Step sizes $\lambda_{\mpa}, \lambda_{\epa}, \lambda_{\xpa}, \lambda_{\tpa}$. Number of update steps $K$. Mini-batch sizes $B_1, B_2$.
\begin{algorithmic}[1]
\State Initialize parameters: $\mpa, \epa, \xpa, \tpa$
\While{not converged}  \Comment{\emph{Stage 1}}
    \State Sample mini-batch $\{(\vect x_b, y_b) \}^{m_{B_1}}_{b=1}$
    \State Evaluate $J_{m}\left(\mpa\right)$
    \State Update $\mpa \leftarrow \mpa-\lambda_{\mpa} \widehat{\nabla}_{\mpa} J(\mpa)$
\EndWhile 
\While{not converged} \Comment{\emph{Stage 2}}
    \State Sample mini-batch  $\{(\vect x_b, \vect t_b, y_b ) \}^{n_{B_2}}_{b=1}$
    \State Evaluate $J_{g, h}\left(\xpa, \tpa\right), J_{e^h}\left(\epa\right)$
    \For{$k = 1$ to $K$}
    \State Update $\tpa \leftarrow \tpa-\lambda_{\tpa} \widehat{\nabla}_{\tpa} J_{g, h}\left(\xpa, \tpa\right)$
    \State Update $\xpa \leftarrow \xpa-\lambda_{\xpa} \widehat{\nabla}_{\xpa} J_{g, h}\left(\xpa, \tpa\right)$
    \EndFor
    \State Update $\epa \leftarrow \epa-\lambda_{\epa} \widehat{\nabla}_{\epa} J_{e^h}\left(\epa\right)$
\EndWhile       
\end{algorithmic}
\end{algorithm}
\end{subfigure}
\hfill
\begin{subfigure}[t]{0.48\textwidth}
\begin{algorithm}[H]
\caption{Pseudocode in a PyTorch-like style.}
\definecolor{codeblue}{rgb}{0.25,0.5,0.5}
\definecolor{dkgreen}{rgb}{0,0.6,0}
\definecolor{gray}{rgb}{0.5,0.5,0.5}
\definecolor{mauve}{rgb}{0.58,0,0.82}
\lstset{
  language=Python,
  backgroundcolor=\color{white},
  basicstyle=\fontsize{8.25pt}{8.25pt}\ttfamily\selectfont,
  columns=fullflexible,
  breaklines=true,
  captionpos=b,
  commentstyle=\fontsize{8.25pt}{8.25pt}\color{gray},
  keywordstyle=\fontsize{8.25pt}{8.25pt}\color{codeblue},
}

\vspace{-1.2ex}
\begin{lstlisting}[language=python]
# Initialize submodels and optimizers
m, e, g, h = MLP(...), MLP(...), MLP(...), GNN(...)
m_opt, e_opt, g_opt, h_opt = Adam(m.params(), m_lr), Adam(e.params(), e_lr), ...

# Stage 1
for batch in train_loader:
    X, Y = batch.X, batch.Y
    m_opt.zero_grad()
    F.mse_loss(m(X), Y).backward()
    m_opt.step()

# Stage 2
for batch in train_loader:
    X, T, Y = batch.X, batch.T, batch.Y
    for _ in range(num_update_steps):
        g_opt.zero_grad()
        h_opt.zero_grad()
        F.mse_loss((g(X)*(h(T) - e(X))).sum(-1), (Y-m(X))).backward()
        g_opt.step()
        h_opt.step()
    e_opt.zero_grad()
    F.mse_loss(e(X), h(T)).backward()
    e_opt.step()
\end{lstlisting}
\vspace{-1.2ex}

\end{algorithm}
\end{subfigure}
\caption{The two-stage algorithm for training SIN.}
\label{fig:algorithm}
\vspace{-3ex}
\end{figure}

\subsection{Advantages of SIN} We conclude by describing the beneficial properties of SIN, particularly in finite-sample regimes:
\begin{enumerate}[ wide = 5pt, leftmargin = *]
\item \textbf{Targeted regularization}: Regularizing $\xfeat, \tfeat$ in \eqr{eq:gin_loss} after partialing out confounding is a type of targeted regularization of the isolated causal effect. 
In contrast, outcome estimation methods can suffer from regularization-induced confounding, e.g., regularizing the effect estimate away from zero in the service of trying to improve predictive performance \cite{x-learner}. 
\item \textbf{Propensity features}: Learning propensity features can help us to (i) partial out parts of $\matr X$ that cause the treatment but not the outcome, and (ii) dispose unnecessary components of $\matr T$.
\item \textbf{Data-efficiency}: In contrast to methods that split the data into disjoint models for each treatment group (known as \emph{T-learners} for binary treatments \cite{caron2020estimating, curth2021nonparametric}), sharing causal effect parameters between all covariates regardless of their assigned treatment increases data-efficiency.
\item \textbf{Partial data}: In settings without access to both the treatment assignment and the outcome but only access to one of them, one can leverage that data to improve the (nuisance) estimator further, e.g., when a patient's recovery is observed one year after a drug was administered \cite{BVNICE}.
\end{enumerate}
\section{Experiments}
Here we evaluate how CATE estimation with our proposed model SIN compares with prior methods.
\subsection{Experimental Setup}

\paragraph{Datasets.} To be able to compute CATE estimation error w.r.t. a ground truth, we design two causal models: a simpler synthetic model with small-world graph treatments and a more complex model with real-world molecular graph treatments and gene expression covariates. 
The Small-World (SW) simulation contains $1$,$000$ uniformly sampled covariates and $200$ randomly generated Watts–Strogatz small-world graphs \cite{watts1998collective} as treatments. \textit{The Cancer Genomic Atlas} (TCGA) simulation uses $9$,$659$ gene expression measurements of cancer patients for covariates \cite{weinstein2013cancer} and $10$,$000$ sampled molecules from the QM9 dataset \cite{ramakrishnan2014quantum} as treatments. Appendix~\ref{sec:exp_details} details the data-generating schemes.

\paragraph{Baselines.} 

We compare our method to (1) \textbf{Zero}, a sanity-check baseline that consistently predicts zero treatment effect and equals the mean squared treatment effect (poorly regularized models may perform worse than that due to confounding), (2) \textbf{CAT}, a categorical treatment variable model using one-hot encoded treatment indicator vectors, (3) \textbf{GNN}, a model that first encodes treatments with a GNN and then concatenates treatment and individual features for regression, (4) \textbf{GraphITE} \cite{harada2020graphite}, a CATE estimation method 
designed for graph treatments (more details in Section~\ref{rw:graphite}). GNN and CAT reflect the performance of standard regression models. The contrast between these two provides insight into whether the additional graph structure of the treatment improves CATE estimation. To deal with unseen treatments during the evaluation of CAT, we map such to the most similar ones seen during training based on their Euclidean distance in the embedding space of the GNN baseline. 

\begin{figure}[t!]
    \centering
    \includegraphics[width=\textwidth]{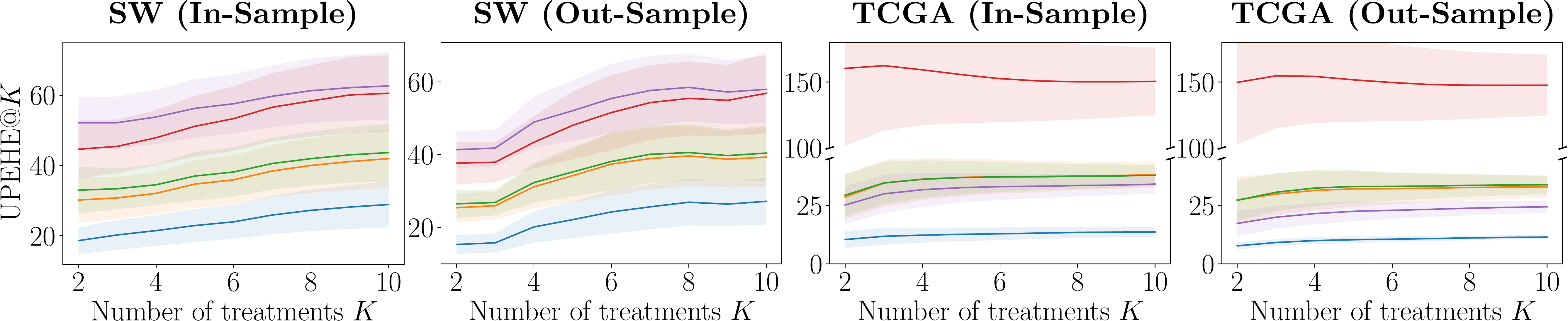}
     \includegraphics[width=.8\textwidth]{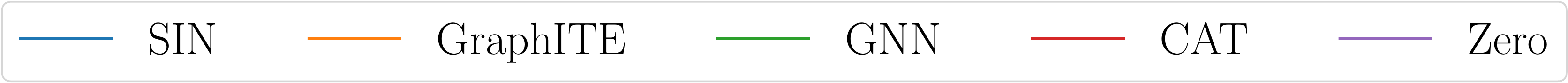}
    \caption{UPEHE@$K$ for $K \in \{2, \ldots, 10\}$.}
    \label{fig:upehe_results}
    \vspace{-3ex}
\end{figure}

\begin{table}[t!]
  \caption{Error of CATE estimation for all methods, measured by WPEHE@$6$. Results are averaged over 10 trials, $\pm$ denotes std. error (each trial samples treatment assignment matrix $\matr W$).}
  \label{tab:results}
  \centering
  \resizebox{0.75\textwidth}{!}{%
  \begin{tabular}{lcccccc}
    \toprule
    \multirow{2}{*}{\bfseries Method} &
      \multicolumn{2}{c}{\bfseries SW} &
    \multicolumn{2}{c}{\bfseries TCGA} \\ 
      & \text{In-sample} & \text{Out-sample} & \text{In-sample}&  \text{Out-sample}  \\
      \midrule
    Zero & $56.26 \pm 8.12$ & $53.77 \pm 8.93$ & $26.63 \pm 7.55$ & $17.94 \pm 4.86$\\
    CAT & $51.75 \pm 8.85 $ & $49.76 \pm 9.73 $& $155.88 \pm 52.82$ & $146.62 \pm 42.32$\\
    GNN & $37.10 \pm 6.84$ & $36.74 \pm 7.42 $ &$ 30.67 \pm 8.29$ & $27.57 \pm 7.95$\\
    GraphITE & $34.81 \pm 6.70$ & $35.94 \pm 8.07$ &$ 30.31 \pm 8.96$ & $27.48 \pm 8.95$\\
    \midrule
    \textbf{SIN} & $\bm{23.00} \pm \bm{4.56}$ & $\bm{23.19} \pm \bm{5.56}$  &$ \bm{10.98} \pm\bm{3.45}$ & $\bm{8.15} \pm \bm{1.46}$\\
    \bottomrule
  \end{tabular}%
  }
  \vspace{-3ex}
\end{table}

\paragraph{Graph models.} For small-world networks, we use \emph{k-dimensional GNNs} \cite{kGNNs}, as to distinguish graphs they take higher-order structures into account. To model molecular graphs, we use \emph{Relational Graph Convolutional Networks} \cite{schlichtkrull2017modeling}, where the nodes are atoms and each edge type corresponds to a specific bond type. We use the implementations of PyTorch Geometric \cite{pytorchgeometric}. 
\paragraph{Evaluation metrics.} We extend the \emph{expected Precision in Estimation of Heterogeneous Effect} (PEHE) commonly used in binary treatment settings \cite{hill} to arbitrary pairs of treatments $(\vect t, \vect t^{\prime})$ as follows. We denote the \emph{Unweighted PEHE} (UPEHE) and the \textcolor{ForestGreen}{\emph{Weighted PEHE} (WPEHE)} as 
\begin{align}
    \epsilon_{\text{UPEHE(\textcolor{ForestGreen}{WPEHE}})} &\triangleq \int_{\mathcal{X}}\left(\widehat{\tau}\left(\vect t^{\prime}, \vect t, \vect x\right)-\tau\left(\vect t^{\prime}, \vect t, \vect x\right)\right)^{2} \textcolor{ForestGreen}{p\left(\vect t \mid \vect x\right)p\left(\vect t^{\prime} \mid \vect x\right)}p\left(\vect x\right)  \dif \vect x,
\end{align} where the weighted version gives less importance to treatment pairs that are less likely; to account for the fact that such pairs will have higher estimation errors. In fact, as the reliability of estimated effects decreases by how likely they are in the observational study, we evaluate all methods on U/WPEHE truncated to the top $K$  treatments, which we call U/WPEHE@$K$. To compute this, for each $\vect x$, we rank all treatments by their propensity $p\left( \vect t \mid \vect x \right)$ (given by the causal model) in descending order. We take the top $K$ treatments and compute the U/WPEHE for all $\binom{K}{2}$ treatment pairs.

\paragraph{In-sample vs. out-sample.} A common benchmark for causal inference methods is the \emph{in-sample} task, which we include here for completeness: estimating CATEs for covariate values $\vect x$ found in the training set.
This task is still non-trivial, as the outcome of only one treatment is observed during training \footnote{The original motivation comes from Fisherian designs where the only source of randomness is on the treatment assignment \cite{imbens:2015}. Our motivation is simpler: rule out the extra variability from different covariates, highlighting the difference between methods due to different loss functions and less due to smoothing abilities.}. In contrast, and arguably of more relevance to decision making, the goal of the \emph{out-sample} task is to estimate CATEs for completely unseen covariate realizations $\vect x'$. 

\paragraph{Hyper-parameter tuning.} To ensure a fair comparison, we perform hyper-parameter optimization with random search for all models on held-out data and select the best hyper-parameters over 10 runs.

\paragraph{Propensity.}\label{para:propensity} We define the propensity (or \emph{treatment selection bias}) as $p\left( \matr T \mid \vect x \right) = \operatorname{softmax}\left(\kappa \matr W^{\top} \matr X\right)$, where $\matr W \in \R^{|\mathcal{T}| \times d}, \forall i, j: W_{ij}\sim \mathcal{U}\left[0,1\right]$ is a random matrix (sampled then fixed for each run). Recall $|\mathcal{T}|$ is the number of available treatments and let $d$ be the dimensionality of the covariates. Here the \emph{bias strength }$\kappa$ is a temperature parameter that determines the flatness of the propensity (the lower the flatter, i.e., $\kappa\!=\!0$ corresponds to the uniform distribution). 

\subsection{Comparison of Performances on different $K$ Treatments}\label{sec:exp_performance}
Figure~\ref{fig:upehe_results} shows the UPEHE@$K$ of all methods for $K \in \{2,\ldots,10\}$. We also report the  WPEHE@$6$ of all methods in Table~\ref{tab:results}. Unless stated otherwise, we report results for bias strengths $\kappa=10$ and $\kappa=0.1$ in the SW and TCGA datasets, respectively across $10$ random trials. 

The results indicate that the relative performance of each method, for both the in-sample and out-sample estimation tasks, is consistent. Further, they suggest that, overall, the performance of SIN is best due to a better isolation of the causal effect from the observed data compared to other methods. The performance difference between CAT and GNN across all results indicate that accounting for graph information significantly improves the estimates. We observe from the SW experiments that GraphITE \cite{harada2020graphite} performs slightly better than GNN, while it is nearly the same as GNN on TCGA.

Surprisingly, the results of the TCGA experiments with low bias strength $\kappa=0.1$ expose that all models but SIN fail to isolate causal effects better than the Zero baseline. These results confirm that confounding effects of $\matr X$ on $Y$ combined with moderate causal effects can cause severe regularization bias for black-box regression models, while SIN partials these out from the outcome by $\com$. We include additional results on convergence and larger values of $K$ in Appendix~\ref{app:add_results_comparison}.

\begin{figure}[t!]
    \centering
    \includegraphics[width=\textwidth]{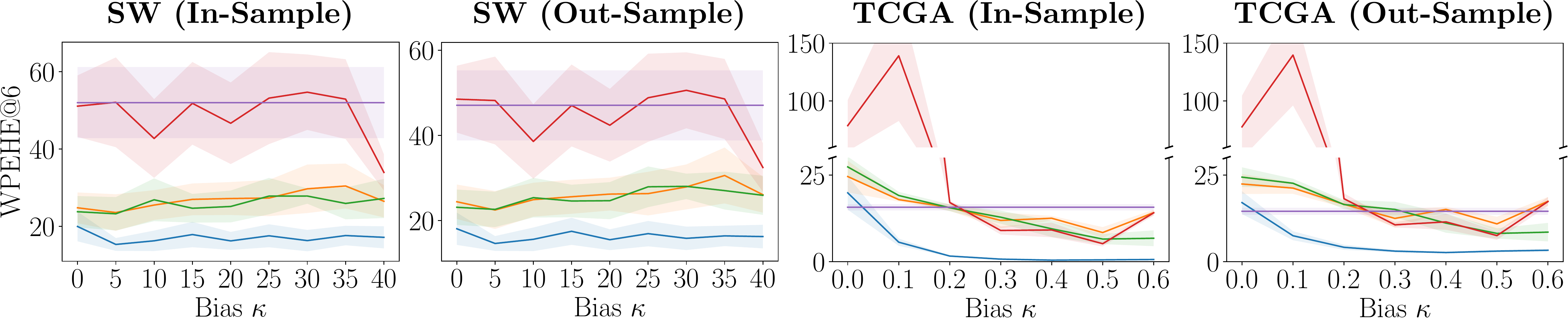}
     \includegraphics[width=.8\textwidth]{figures/legend.pdf}
    \caption{WPEHE@$6$ over increasing bias strength $\kappa$.}
    \label{fig:bias_results}
    \vspace{-4ex}
\end{figure}

\subsection{Comparison of Robustness to different Bias Strengths $\kappa$}\label{sec:exp_bias} A strong selection bias (i.e. large $\kappa$) in the observed data makes CATE estimation more difficult, as it becomes unlikely to see certain treatments $\vect t \in \mathcal{T}$ for particular covariates $\vect x$. Here, we assess each model's robustness to varying levels of selection bias, determined by $\kappa$, across $5$ random seeds. In Figure~\ref{fig:bias_results}, we see that SIN outperforms the baselines across the entire range of considered biases. Interestingly, SIN performs competitively even in a case with no selection bias ($\kappa\!=\!0$, which corresponds to a randomized experiment). Importantly, all performances seem to either stagnate (SW) or to increase (TCGA) with increasing biases. Notably, the poor performance of CAT suddenly improves on datasets with high bias. We believe this is because, in high bias regimes, we see fewer distinct treatments overall, which allows the CAT model to approach the performance of GNN. 
\section{Limitations, Future Work and Potential Negative Societal Impacts} \label{sec:limitations}

\vspace{-1ex}
\paragraph{Limitations and future work.} 
Firstly, in some real-life domains, Assumption~\ref{ass:unconfoundedness} (Unconfoundedness) can be too strong, as there may exist \emph{hidden confounders}. There are two common strategies to deal with them: utilizing \emph{instrumental variables} \cite{deepiv, kiv, xu2021learning} or \emph{proxy variables} \cite{mastouri2021proximal, miao2018identifying, tchetgen2020introduction}. Developing new approaches for structured interventions in such settings is a promising future direction. Secondly, SIN is based on neural networks; however, neural network initialization can impact final estimates. To obtain consistency guarantees, GRD can be combined with kernel methods \cite{mastouri2021proximal, kiv}. 

\vspace{-1ex}
\paragraph{Potential negative societal impacts.}
Because causal inference methods make recommendations about interventions to apply in real-world settings, misapplying them can have a negative real-world impact. It is crucial to thoroughly test these methods on realistic simulations and alter aspects of them to understand how violations of assumptions impact estimation. We have aimed to provide a comprehensive evaluation of structured treatment methods by showing how estimation degrades as less likely treatments are considered (Figure~\ref{fig:upehe_results}) and as treatment bias increases (Figure~\ref{fig:bias_results}). 
\section{Conclusion} \label{sec:conclusion}
The main contributions of this paper are two-fold: (i) the generalized Robinson decomposition that yields a pseudo-outcome targeting the causal effect while possessing a quasi-oracle convergence guarantee under mild assumptions, and (ii) Structured Intervention Networks, a practical algorithm using representation learning that outperforms prior approaches in experiments with graph treatments.

\section*{Acknowledgements}

We thank Antonin Schrab, David Watson, Jakob Zeitler, Limor Gultchin, Marc Deisenroth and Shonosuke Harada for useful discussions and constructive feedback on the paper. JK and YZ acknowledge support by the Engineering and Physical Sciences Research Council with grant number EP/S021566/1. This work was partially supported by an ONR grant number N62909-19-1-2096 to RS. We thank the Alan Turing Institute for the provision of Azure cloud computing resources.

\bibliographystyle{icml2020}
\bibliography{citations}
\newpage
\appendix

\section{Other Related Work} \label{sec:orw}

\paragraph{Plug-in estimators.} 
A recent line of work for CATE estimation derives \emph{plug-in estimators} \cite{caron2020estimating}.\footnote{These are also called \emph{meta-learners}. To avoid confusion with \emph{meta-learning}, we call these \emph{plug-in estimators}.} These work by decomposing CATE estimation into multiple sub-problems (so-called \emph{nuisance components}), each solvable using any supervised learning method \cite{curth2021nonparametric, hahn2020bayesian, kennedy2020optimal, x-learner, r-learner}. Currently, these approaches are limited to binary treatment setups. Our approach is inspired by these methods, extending plug-in estimation to structured treatment settings.

\paragraph{CATE estimation with neural networks.}
Neural network CATE estimators 
typically use separate prediction heads for each treatment option  \cite{johansson2016learning, CEVAE, nie2021vcnet, schwab2019perfect, schwab2020doseresponse, shalit2017estimating, dragonnet}. This architectural design reduces one source of regularization bias: the influence of the treatment indicator variable might be lost in the high-dimensional network representations. Extending this idea directly to structured treatments would not only be computationally expensive, but would also not be able to make use of treatment features or learn treatment representations. 

\paragraph{Multiple treatments.}
While Inverse Probability Weighting (IPW) \cite{li2019propensity, lopez2017estimation, zhou2018offline} is a popular technique for estimating effects with multiple, categorical treatments, it requires estimating the propensity density which is infeasible in settings with hundreds or thousands of treatments; some of which may have not been seen during training. 
\citet{nabi2020semiparametric} propose a framework for sufficient dimensionality reduction of high-dimensional treatments based on semiparametric inference theory. Besides relying on IPW, this approach is designed for average treatment effects (not CATEs).

\section{The Generalized Robinson Decomposition} \label{app:grd}
\subsection{Motivation}
frequently the influence of $\matr T$ on $Y$ is very different from the influence of $\matr X$ on $Y$. Specifically, $f(\matr X, \matr T)$ often has different smoothness in $\matr X$ and $\matr T$. For instance, different health histories $\matr X$ for a fixed treatment $\matr t$ will have a much more variable effect on $Y$ than different treatments $\matr t$ for a given history $\matr X$. This is why methods like the R-learner \cite{r-learner} have carefully separated estimation functions of $\matr X$ from functions of $\matr T$ \cite{caron2020estimating}.

A generic way to extend the Robinson decomposition to arbitrary treatments is to learn a model $\hat f(\matr X, \matr T)$ defined over the entire outcome surface, via mean outcome $\hat m(\matr X)$ and treatment conditional density $p(\matr T \mid  \matr X)$. In this case, we fit the relationship
\begin{align}
   Y - m\left(\matr X\right) = f\left (\matr X, \matr T\right ) -  e^p\left (\matr X\right ) + \varepsilon, \quad \text{where} \quad  e^p\left(\matr x\right) \triangleq \E\left[f\left(\matr X, \matr T\right) \mid \matr x\right].
\end{align} To learn $f(\cdot, \cdot)$ from a dataset $\mathcal{D}\!=\!\left\{\left(\vect{x}_{i}, \vect t_{i}, y_{i}\right)\right\}_{i=1}^{n}$ we need to solve,
\begin{align}
\hat{f} = \argmin_{f \in \mathcal{F}} \frac{1}{N}\sum_{i=1}^N \left[\left\{y_i - \hat{m}\left(\vect x_i\right) \right\} - \left(f\left(\vect x_i, \vect t_i\right) - \hat{e}^p\left(\vect x_i\right)\right)\right]^2, \label{eq:naive}
\end{align}
where $\mathcal{F}$ is some function space and $\hat{m}$ is a plug-in finite sample estimate of $m$. Because $e^p$ contains $f$, we need to estimate it, which we denote $\hat{e}^p$. 

One solution is to estimate the propensity $p(\matr T \mid \matr X)$ and use it to compute $\hat{e}^p$. However, this approach requires conditional density estimation over potentially high-dimensional, structured treatments, which remains an open research question \cite{zhang2021set}, and is prone to high variance \cite{kiv}. Further, to compute $\hat{e}^p$ from it, one has to resort to Monte Carlo evaluation. By learning propensity features $\pf$ instead of $p(\matr T \mid \matr X)$, we avoid these issues.

Another option is to solve for $f$, fix it, then estimate $\hat{e}^p$ using regression from finite samples, and iterate to a fixed point. However, there is a fundamental issue with this approach: we are typically interested in regularizing the causal effect directly as opposed to the generic regression function. This is why, for instance, the R-learner parameterizes $\mu_1(\matr x)$ as a (nuisance) baseline $\mu_0(\matr x)$ plus the CATE $\tau_{\text{b}}(\matr x)$. The black-box $f(\matr x, \matr t)$ does not capture the asymmetry between $\matr x$ and $\matr t$ in the implied CATE $f(\matr x, \matr t) - f(\matr x, \matr t')$. Further, unlike the binary case, in many applications, we do not have a baseline treatment $\matr t_0$  with respect to which we could parameterize $f$ in terms of some $\tau(\matr t, \matr t_0, \matr x)$. To regularize the causal effect more directly, we make the product effect assumption, which allows us to partial out confounding. 
\subsection{Derivation in detail}
We consider the product effect parameterization of $p(y \mid \matr x, \matr t)$,
\begin{align}
    Y = \quad \underbrace{g\left(\vect X\right)^{\top} h\left(\matr T\right)}_{=:f(\matr X, \matr T)} \quad + \quad   \varepsilon, 
    \label{eq:PE_copy}
\end{align} where $g: \mathcal{X} \rightarrow \mathbb{R}^{d}, h: \mathcal{T} \rightarrow \mathbb{R}^d$ and $\E[\varepsilon \mid \matr x, \matr t] = \E\left[\varepsilon~\mid~\matr x\right] = 0,$ for all  $\left(\matr x, \matr t\right) \in \mathcal X \times \mathcal T$.
Rearranging \eqr{eq:PE_copy} yields the Robinson residual
\begin{align}
    \varepsilon = Y -  g\left(\vect X\right)^{\top} h\left(\matr T\right), \label{eq:pipe_robinson}
\end{align} which we aim to rewrite in terms of $m(\matr X)$. To this end, we define \textit{propensity features} $e^h\left(\matr X\right)$ as
\begin{align} 
e^h\left(\matr X\right) \triangleq \E \left[ h\left(\matr T\right) \mid \matr X \right ], \quad \text{such that} \quad
    m\left (\matr X\right) = \E \left [ Y \mid \matr X \right ] = g\left(\matr X\right)^{\top} e^h\left(\matr X\right). 
\end{align}
To obtain the generalized Robinson decomposition, one rewrites \eqr{eq:pipe_robinson} as 
\begin{align}
    \varepsilon &= Y - \left (  g\left(\matr X\right)^{\top} \left [h\left(\matr T\right) + \cancel{e^h\left(\matr X\right)}  - \cancel{e^h\left(\matr X\right)} \right ] \right ) \\
    &= Y - \left (   g\left(\matr X\right)^{\top} e^h\left(\matr X\right) + g\left(\matr X\right)^{\top} \left (h\left(\matr T\right) - e^h\left(\matr X\right)\right) \right ) \\
    &= Y -  \underbrace{ \left(  g\left(\matr X\right)^{\top} e^h\left(\matr X\right) \right )}_{m(\matr X)}  - g\left(\matr X\right)^{\top} \left (h\left(\matr T\right) - e^h\left(\matr X\right)\right). \label{eq:generalized_r_decomp}
\end{align}
Hence, the generalized Robinson decomposition is
\begin{align}
    Y -  m(\matr X)  = g\left(\matr X\right)^{\top} \left (h\left(\matr T\right) - e^h\left(\matr X\right)\right) + \varepsilon.
\end{align}

\section{Universality of Product Decomposition}
\label{sec:universality}

\textbf{Proof of Proposition \ref{prop:prod_decomp}.}

\begin{proof}

Define ${\mathcal{H}_{0}}_{\mathcal{X\times T}}$ as $\left\{f\left(\mathbf{x}, \mathbf{t}\right)=\sum_{i=1}^n \alpha_i k\left(\left(\mathbf{x}_i, \mathbf{t}_i\right), \left(\mathbf{x}, \mathbf{t}\right)\right) |n \in \mathbb{N}, \alpha_{i=1,\cdots,n} \in \mathbb{R} \right\}$.
By definition, the RKHS $\mathcal{H_{X\times T}}$ is the set of pointwise limits of Cauchy sequences $(f_n)_n \in {\mathcal{H}_{0}}_{\mathcal{X\times T}}$. By Lemma 41 of \cite{rkhs_notes}, the Cauchy sequences also converges in the $\mathcal{H_{X\times T}}$ norm.

For any $f \in \mathcal{H_{X\times T}}$, pick its Cauchy sequence $(f_n) in {\mathcal{H}_{0}}_{\mathcal{X\times T}}$. Since $\sum_{i=1}^\infty \alpha_i k\left(\left(\mathbf{x}_i, \mathbf{t}_i\right), \left(\mathbf{x}, \mathbf{t}\right)\right)$ converges in $\|\cdot\|_{\mathcal{H_{X\times T}}}$, for any $\tilde{\epsilon}$ there exist a $\tilde{d}$ such that let $f_{\tilde{d}} = \sum_{i=1}^{\tilde{d}} \alpha_i k\left(\left(\mathbf{x}_i, \mathbf{t}_i\right), \left(\mathbf{x}, \mathbf{t}\right)\right)$, then
\begin{align}
  \|f_{\tilde{d}} - f\|_{\mathcal{H_{X\times T}}} \leq \tilde{\epsilon}
\end{align}
Since for any RKHS with kernel $k$, the RKHS norm is always an upper bound on the $L_2$ norm up to scaling by a constant $C_k$, 
\begin{align}
\|f_{\tilde{d}}-f\|_{L_2(P_{\mathcal{(X\times T)}})} \leq C_k \|f_{\tilde{d}} - f\|_{\mathcal{H_{X\times T}}} \leq C_k \tilde{\epsilon}
\end{align}
Then for any $\epsilon$, we can choose $d \in \mathbb{N}$ such that $\|f_d - f\|_{L_2(P_{\mathcal{{X\times T}}})} \leq C_k \cdot \frac{\epsilon}{C_k} = \epsilon$.

It remains to show that $f_d$ can be written as $g^{\top}h$ as required. $\mathcal{H_{X\times T}}$ is isometrically isomorphic to $\mathcal{H_{X}} \times \mathcal{H_{T}}$; we can decompose $k$ into the product kernel 
\begin{align}
   k\left(\left(\mathbf{x}, \mathbf{t}\right), \left(\mathbf{x}', \mathbf{t}'\right)\right)=k_{\mathcal{X}}\left(\mathbf{x}, \mathbf{x}'\right)k_{\mathcal{T}}\left(\mathbf{t}, \mathbf{t}'\right).
\end{align} Thus $f_d\left(\mathbf{x}, \mathbf{t}\right) = \sum_{i=1}^d \alpha_i k_{\mathcal{X}}(\mathbf{x}, \mathbf{x}_i)k_{\mathcal{T}}(\mathbf{t}, \mathbf{t}_i) $. Set $g(\mathbf{x}) = \left(\alpha_1 k_{\mathcal{X}}\left(\mathbf{x}, \mathbf{x}_1\right),\cdots, \alpha_d k_{\mathcal{X}}\left(\mathbf{x}, \mathbf{x}_d\right)\right)^{\top}$, $h(\mathbf{t})=\left(k_{\mathcal{T}}\left(\mathbf{t}, \mathbf{t}_1\right), \cdots, k_{\mathcal{T}}\left(\mathbf{t}, \mathbf{t}_d\right)\right)^{\top}$, we obtain $f_d=g^{\top}.$
\end{proof}

\section{Experimental Details} \label{sec:exp_details}
\subsection{Simulations}
\paragraph{Baseline effect}
Similarly as in \cite{SCIGAN, curth2021nonparametric, nie2021vcnet}, for each run of the experiment, we randomly sample a vector $\vect u_{0} \sim \mathcal{U}(\mathbf{0},\mathbf{1})$, and set $\mathbf{v}_{0}=\mathbf{u}_{0} /\left\|\mathbf{u}_{0}\right\|$ where $\|\cdot\|$ is the Euclidean norm. We then model the baseline effect as
\begin{align}
    \mu_0\left(\vect x \right) = \mathbf{v}_{0}^{\top} \mathbf{x}.
\end{align}
\subsubsection{Small-World Networks}
\paragraph{Covariates} We uniformly sample $20$-dimensional multivariate covariates $\matr X \sim \mathcal{U}\left(\mathbf{-1}, \mathbf{1}\right)$. The in-sample dataset consists of $1$,$000$ units, and the out-sample one of $500$. For the treatment assignment, we square the covariates element-wise; i.e., we sample treatment assignments according to $p\left( \matr T \mid \vect x^2 \right)$. 

\paragraph{Graph interventions}
For each graph intervention, we uniformly sample a number of nodes between $10$ and $120$, number of neighbors for each node between $3$ and $8$, and the probability of rewiring each edge between $0.1$ and $1$ Then, we repeatedly generate Watts–Strogatz small-world graphs until we get a connected one. Each vertex has one feature, which is its degree centrality. We denote a graph's node connectivity as $\nu\left(\gG\right)$ and its average shortest path length as $l\left(\gG\right)$.

\paragraph{Outcomes}
Analogously as for the baseline effect, we generate two randomly sampled vectors $\vect v_{\nu}$ and $\vect v_{l}$. Then, given an assigned graph treatment $\gG$ and a covariate vector $\vect x$, we generate the outcome as \begin{equation}
    Y = 100 \mu_0\left(\vect x\right) + 0.2 \nu\left(\gG\right)^2 \cdot \mathbf{v}_{\nu}^{\top} \mathbf{x} + l\left(\gG\right) \cdot \mathbf{v}_{l}^{\top} \mathbf{x} + \epsilon, \quad \epsilon \sim \mathcal{N}\left(0,1\right).
\end{equation}

\subsubsection{TCGA}

\paragraph{Covariates} 
The \textit{The Cancer Genomic Atlas} (TCGA) simulation uses $4$,$000$-dimensional $9$,$659$ gene expression measurements of cancer patients for covariates \cite{weinstein2013cancer}, i.e., each unit is a covariate vector $\matr X \in \R^{4000}$. The in-sample and out-sample datasets consist of $5$,$000$ and $4$,$659$ units, respectively. In each run, the units are split randomly into in- and out-sample datasets. We used
the same version of the TCGA dataset as used by \citet{SCIGAN} and \citet{schwab2020doseresponse}. 

\paragraph{Graph interventions} In each run, we randomly sample $10$,$000$ molecules from the Quantum Machine 9 (QM9) dataset \cite{ramakrishnan2014quantum, gdb17} (with $133$k molecules in total). For each molecule, we create a relational graph, where each node corresponds to an atom and consist of $78$ atom features. An edge corresponds to the chemical bond type, where we label each edge correspondingly, considering \emph{single}, \emph{double}, \emph{triple} and \emph{aromatic} bonds. Furthermore, for each molecule, we obtain 8 of its properties \emph{mu, alpha, homo, lumo, gap, r2, zpve, u0}, which we collect in the vector $\vect z \in \R^8$.

\paragraph{Outcomes}
For each covariate vector $\vect x$, we compute its $8$-dimensional PCA components, denoted by $\vect x^{(\text{PCA})} \in \R^{8}$. Then, given the molecular properties of the assigned molecule treatment $\vect z$, we generate outcomes by

\begin{align}
    Y = 10 \mu_0\left(\vect x\right) + 0.01 \vect z^{\top} \vect x^{\left(\text{PCA}\right)} + \epsilon, \quad \epsilon \sim \mathcal{N}\left(0,1\right).
\end{align}

\subsection{Hyper-parameters}
To ensure a fair comparison between all models, we perform hyper-parameter optimization with random search for all models on held-out data and select the best hyper-parameters over 10 runs. While conceptually, choosing hyper-parameters based on predictive metrics may not necessarily lead to good CATE estimation performance, \citet{neal2021realcause} provide empirical evidence that doing so indeed often does in practice. 

Table~\ref{tab:hyp_sw} and Table~\ref{tab:hyp_tcga} include the hyper-parameter search ranges we set in the SW and TCGA experiments, respectively. Table~\ref{tab:hyp_sw_fixed} and Table~\ref{tab:hyp_tcga_fixed} include the fixed hyper-parameter values across all SW and TCGA experiments, respectively. We restricted the number of hyper-parameter optimization trials to $10$ in all experiments. We observed that all models' performances are rather insensitive to hyper-parameter values in the considered search ranges, i.e., the performances across trials have not varied much. The search ranges for the HSIC penalty $\lambda$ are taken from the experimental section of the GraphITE paper \cite{harada2020graphite}, where the authors also argue that their model's performance is insensitive to this weight. In consultation with \citet{harada2020graphite}, we use \citet{hsic_implementation}'s implementation of the normalized HSIC. We use early stopping for all models based on their training loss. We noticed that a patience value below 10 often leads to pre-convergence stopping with subsequent sub-optimal performance for all models but GIN. 

\subsubsection{SW}

\begin{table}[H]
    \centering
    \begin{tabular}{lc}
    \toprule
    \textbf{Hyper-parameter} & \textbf{Search range} \\
    \midrule
        Num. of layers for covariates representations & 2-4 \\
        Num. of layers for treatment representations & 3-6 \\ 
        Num. of layers for $\com{}^*$ & 3-6 \\
        Num. of layers for $\pf{}^*$ & 3-6 \\
        Num. of layer for final feed-forward network ${}^\dagger$ & 2-6 \\
        Dim. of hidden layers for covariates representations & 50-300\\
        Dim. of hidden layers for treatment representations & 50-300\\
        Dim. of hidden layers for $\com{}^*$ & 200-300\\
        Dim. of hidden layers for $\pf{}^*$ & 50-150\\
        Dim. of $\xfeat, \tfeat{}^*$ & 50-250\\
        Dim. of final covariates/treatment layer & 2-200\\
        Dim. of hidden layers for final feed-forward network & 50-300\\
        Num. update steps $K{}^*$ & 10-20\\
        Early stopping patience for $\com{}^*$ & \{5, 10\}\\
        Early stopping patience for $\xfeat, \tfeat, \pf{}^*$ & \{1, 5\} \\
        Learning rates $\lambda_{\xpa}, \lambda_{\tpa} {}^*$ & \{5e-4, 1e-3\} \\ 
        Learning rate ${}^\dagger$ & \{5e-4, 1e-3\} \\ 
        Dropout for $\com {}^*$ & \{0, 0.2\} \\
        Dropout for $\pf {}^*$ & \{0, 0.2 \} \\
        Weight of HSIC penalty $\lambda{}^{\ddagger}$ & \{0.001, 0.01, 1, 10, 100, 1000\}
        \\ \bottomrule
    \end{tabular}
    \caption{Hyper-parameter search ranges for SW experiments. ${}^*$ denotes hyper-parameter only applicable for GIN; ${}^\dagger$ applicable for all models but GIN, ${}^\ddagger$ applicable only for GraphITE.}
    \label{tab:hyp_sw}
\end{table}

\begin{table}[H]
    \centering
    \begin{tabular}{lc}
    \toprule
    \textbf{Hyper-parameter} & \textbf{Value} \\
    \midrule
        Optimizer & Adam \cite{adam} \\
        Batch size & 500 \\ 
        Weight decay (all optims.) & 0 \\
        $\lambda_{\mpa}, \lambda_{\epa}$ & 1e-3 \\
        Early stopping patience ${}^\dagger$ & 10\\
        GNN Batch Norm & True \\
        MLP Batch Norm (all MLPs) & False \\
        Activation functions (all layers) & ReLU \\
        Validation set size (in \%) & 20\% \\
        \bottomrule
    \end{tabular}
    \caption{Fixed hyper-parameter values across all SW experiments. ${}^*$ denotes hyper-parameter only applicable for GIN; ${}^\dagger$ applicable for all models but GIN, ${}^\ddagger$ applicable only for GraphITE.}
    \label{tab:hyp_sw_fixed}
\end{table}

\newpage
\subsubsection{TCGA}
\begin{table}[ht]
    \centering
    \begin{tabular}{lc}
    \toprule
    \textbf{Hyper-parameter} & \textbf{Search range} \\
    \midrule
        Num. of layers for covariates representations & 2-5 \\
        Num. of layers for treatment representations & 3-6 \\ 
        Num. of layers for $\com{}^*$ & 2-4 \\
        Num. of layers for $\pf{}^*$ & 1-6 \\
        Num. of layer for final feed-forward network ${}^\dagger$ & 1-5 \\
        Dim. of hidden layers for covariates representations & 100-400\\
        Dim. of hidden layers for treatment representations & 100-400\\
        Dim. of hidden layers for $\com{}^*$ & 100-300\\
        Dim. of hidden layers for $\pf{}^*$ & 10-50\\
        Dim. of $\xfeat, \tfeat{}^*$ & 200-600 \\
        Dim. of final covariates/treatment layer & 2-800\\
        Dim. of hidden layers for final feed-forward network & 100-400\\
        Num. update steps $K{}^*$ & 10-20\\
        Early stopping patience for $\xfeat, \tfeat, \pf{}^*$ & \{5, 10\} \\
        Learning rates $\lambda_{\xpa}, \lambda_{\tpa} {}^*$ & \{5e-4, 1e-3\} \\ 
        Learning rate ${}^\dagger$ & \{5e-4, 1e-3\} \\ 
        Weight of HSIC penalty $\lambda{}^{\ddagger}$ & \{0.001, 0.01, 1, 10, 100, 1000\}
        \\ \bottomrule
    \end{tabular}
    \caption{Hyper-parameter search ranges for TCGA experiments. ${}^*$ denotes hyper-parameter only applicable for GIN; ${}^\dagger$ applicable for all models but GIN, ${}^\ddagger$ applicable only for GraphITE.}
    \label{tab:hyp_tcga}
\end{table}

\begin{table}[ht]
    \centering
    \begin{tabular}{lc}
    \toprule
    \textbf{Hyper-parameter} & \textbf{Value} \\
    \midrule
        Optimizer & Adam \cite{adam}\\
        Batch size & 1000 \\ 
        Weight decay (all optims.) & 0 \\
        $\lambda_{\mpa}, \lambda_{\epa}$ & 1e-3 \\
        Early stopping patience ${}^\dagger$ & 10\\
        GNN Batch Norm & True \\
        MLP Batch Norm (all MLPs) & False \\
        Activation functions (all layers) & ReLU \\
        Validation set size (in \%) & 20\% \\
        \bottomrule
    \end{tabular}
    \caption{Fixed hyper-parameter values across all TCGA experiments. ${}^*$ denotes hyper-parameter only applicable for GIN; ${}^\dagger$ applicable for all models but GIN, ${}^\ddagger$ applicable only for GraphITE.}
    \label{tab:hyp_tcga_fixed}
\end{table}

\subsubsection{Hardware details}
All experiments were run on Microsoft Azure Virtual Machines with 12 Intel Xeon E5-2690 v4 CPUs and 2 NVIDIA Tesla K80 GPUs. No single trial took longer than $\sim30$ minutes to run. 
\newpage
\section{Additional Results}
\subsection{Comparison of Performances on different $K$ Treatments}\label{app:add_results_comparison}
We present additional WPEHE@$K$ results for the experiments in Section~\ref{sec:exp_performance} with varying $K$.

\begin{longtable}{lcccccc}
    \toprule
    \multirow{2}{*}{\bfseries Method} &
      \multicolumn{2}{c}{\bfseries SW} &
    \multicolumn{2}{c}{\bfseries TCGA} \\ 
      & \text{In-sample} & \text{Out-sample} & \text{In-sample}&  \text{Out-sample}  \\
    \midrule
    \textbf{WPEHE@2} & & & & \\
    \midrule
    Zero & $52.17 \pm 7.37$ & $41.36 \pm 5.04$ & $25.17 \pm 8.12$ & $17.33 \pm 5.41$\\
    CAT & $44.63 \pm 8.18 $ & $37.65 \pm 5.90 $& $160.35 \pm 58.56$ & $149.75 \pm 46.86$\\
    GNN & $32.98 \pm 6.63$ & $26.47 \pm 3.87 $ &$ 29.35 \pm 8.90$ & $27.17 \pm 8.67$\\
    GraphITE & $30.18 \pm 6.45$ & $25.39 \pm 4.04$ &$ 28.60 \pm 9.44$ & $27.37 \pm 9.87$\\
    GIN & $\bm{18.00} \pm \bm{3.83}$ & $\bm{15.30} \pm \bm{2.60}$  &$ \bm{10.44} \pm\bm{3.62}$ & $\bm{7.76} \pm \bm{1.56}$\\
    \midrule
    \textbf{WPEHE@3} & & & & \\
    \midrule
    Zero & $51.61 \pm 7.24$ & $41.53 \pm 4.96$ & $25.97 \pm 7.96$ & $17.50 \pm 5.11$\\
    CAT & $44.87 \pm 7.53 $ & $37.59 \pm 5.46 $& $159.48 \pm 56.46$ & $148.80 \pm 44.87$\\
    GNN & $32.97 \pm 5.75$ & $26.60 \pm 3.70 $ &$ 30.22 \pm 8.77$ & $27.29 \pm 8.30$\\
    GraphITE & $30.39 \pm 5.89$ & $25.70 \pm 3.70$ &$ 29.71 \pm 9.43$ & $27.27 \pm 9.38$\\
    GIN & $\bm{19.79} \pm \bm{4.06}$ & $\bm{15.54} \pm \bm{2.56}$  &$ \bm{10.62} \pm\bm{3.56}$ & $\bm{7.94} \pm \bm{1.51}$\\
    \midrule
    \textbf{WPEHE@4} & & & & \\
    \midrule
    Zero & $52.92 \pm 7.47$ & $47.93 \pm 6.68$ & $26.35 \pm 7.79$ & $17.76 \pm 5.05$\\
    CAT & $46.95 \pm 7.65 $ & $42.47 \pm 6.91 $& $158.02 \pm 54.76$ & $148.08 \pm 43.71$\\
    GNN & $33.89 \pm 5.73$ & $31.51 \pm 5.27 $ &$ 30.51 \pm 8.57$ & $27.53 \pm 8.23$\\
    GraphITE & $31.43 \pm 5.75$ & $30.39 \pm 5.71$ &$ 30.07 \pm 9.22$ & $27.48 \pm 9.28$\\
    GIN & $\bm{20.78} \pm \bm{4.11}$ & $\bm{19.50} \pm \bm{4.12}$  &$ \bm{10.76} \pm\bm{3.51}$ & $\bm{8.08} \pm \bm{1.51}$\\
    \midrule
    \textbf{WPEHE@5} & & & & \\
    \midrule
    Zero & $55.02 \pm 8.00$ & $50.75 \pm 7.92$ & $26.53 \pm 7.66$ & $17.91 \pm 4.96$\\
    CAT & $49.78 \pm 8.37 $ & $46.65 \pm 8.86 $& $156.77 \pm 53.58$ & $147.20 \pm 42.86$\\
    GNN & $36.06 \pm 6.69$ & $34.16 \pm 6.41 $ &$ 30.61 \pm 8.41$ & $27.61 \pm 8.10$\\
    GraphITE & $33.69 \pm 6.56$ & $33.13 \pm 6.92$ &$ 30.22 \pm 9.08$ & $27.53 \pm 9.12$\\
    GIN & $\bm{22.06} \pm \bm{4.40}$ & $\bm{21.19} \pm \bm{4.80}$  &$ \bm{10.90} \pm\bm{3.47}$ & $\bm{8.13} \pm \bm{1.49}$\\
    \midrule
    \textbf{WPEHE@6} & & & & \\
    \midrule
    Zero & $56.26 \pm 8.12$ & $53.77 \pm 8.93$ & $26.63 \pm 7.55$ & $17.94 \pm 4.86$\\
    CAT & $51.75 \pm 8.85 $ & $49.76 \pm 9.73 $& $155.88 \pm 52.82$ & $146.62 \pm 42.32$\\
    GNN & $37.10 \pm 6.84$ & $36.74 \pm 7.42 $ &$ 30.67 \pm 8.29$ & $27.57 \pm 7.95$\\
    GraphITE & $34.81 \pm 6.70$ & $35.94 \pm 8.07$ &$ 30.31 \pm 8.96$ & $27.48 \pm 8.95$\\
    GIN & $\bm{23.00} \pm \bm{4.56}$ & $\bm{23.19} \pm \bm{5.56}$  &$ \bm{10.98} \pm\bm{3.45}$ & $\bm{8.15} \pm \bm{1.46}$\\
    \midrule
    \textbf{WPEHE@7} & & & & \\
    \midrule
    Zero & $58.16 \pm 8.38$ & $55.73 \pm 9.01$ & $26.66 \pm 7.48$ & $17.97 \pm 4.81$\\
    CAT & $54.62 \pm 9.27 $ & $52.21 \pm 9.74 $& $155.24 \pm 52.25$ & $146.15 \pm 41.90$\\
    GNN & $39.21 \pm 7.05$ & $38.51 \pm 7.50 $ &$ 30.67 \pm 8.21$ & $27.56 \pm 7.86$\\
    GraphITE & $37.00 \pm 7.10$ & $37.34 \pm 8.05$ &$ 30.33 \pm 8.88$ & $27.47 \pm 8.86$\\
    GIN & $\bm{24.71} \pm \bm{5.07}$ & $\bm{24.46} \pm \bm{5.79}$  &$ \bm{11.02} \pm\bm{3.43}$ & $\bm{8.17} \pm \bm{1.45}$\\
    \midrule
    \textbf{WPEHE@8} & & & & \\
    \midrule
    Zero & $59.57 \pm 8.74$ & $56.61 \pm 8.94$ & $26.73 \pm 7.43$ & $18.03 \pm 4.76$\\
    CAT & $56.24 \pm 9.71 $ & $53.33 \pm 9.71 $& $154.86 \pm 51.85$ & $145.94 \pm 41.61$\\
    GNN & $40.44 \pm 7.36$ & $39.04 \pm 7.33 $ &$ 30.72 \pm 8.16$ & $27.49 \pm 8.78$\\
    GraphITE & $38.42 \pm 7.46$ & $38.06 \pm 7.89$ &$ 30.39 \pm 8.82$ & $27.49 \pm 8.78$\\
    GIN & $\bm{25.90} \pm \bm{5.51}$ & $\bm{25.63} \pm \bm{6.03}$  &$ \bm{11.10} \pm\bm{3.43}$ & $\bm{8.20} \pm \bm{1.44}$\\
    \midrule
    \textbf{WPEHE@9} & & & & \\
    \midrule
    Zero & $60.39 \pm 8.94$ & $55.72 \pm 8.44$ & $26.75 \pm 7.40$ & $18.06 \pm 4.73$\\
    CAT & $57.78 \pm 10.27 $ & $53.06 \pm 9.36 $& $154.60 \pm 51.57$ & $145.73 \pm 41.37$\\
    GNN & $41.45 \pm 7.60$ & $38.47 \pm 6.92 $ &$ 30.72 \pm 8.11$ & $27.60 \pm 7.74$\\
    GraphITE & $39.43 \pm 7.69$ & $37.43 \pm 7.48$ &$ 30.39 \pm 8.78$ & $27.50 \pm 8.72$\\
    GIN & $\bm{26.76} \pm \bm{5.80}$ & $\bm{25.30} \pm \bm{5.75}$  &$ \bm{11.12} \pm\bm{3.42}$ & $\bm{8.22} \pm \bm{1.43}$\\
    \midrule
    \textbf{WPEHE@10} & & & & \\
    \midrule
    Zero & $60.92 \pm 9.10$ & $56.44 \pm 8.91$ & $26.78 \pm 7.35$ & $18.09 \pm 4.71$\\
    CAT & $58.32 \pm 10.29 $ & $54.76 \pm 10.56 $& $154.39 \pm 51.32$ & $145.57 \pm 41.21$\\
    GNN & $42.08 \pm 7.82$ & $39.11 \pm 7.24 $ &$ 30.73 \pm 8.07$ & $27.61 \pm 7.70$\\
    GraphITE & $40.26 \pm 7.94$ & $37.99 \pm 7.80$ &$ 30.41 \pm 8.74$ & $27.51 \pm 8.69$\\
    GIN & $\bm{27.47} \pm \bm{6.07}$ & $\bm{26.01} \pm \bm{6.06}$  &$ \bm{11.13} \pm\bm{3.41}$ & $\bm{8.23} \pm \bm{1.43}$\\
    \bottomrule
\caption{Error of CATE estimation for all methods, measured by WPEHE@$1-10$. Results are averaged over 10 trials, $\pm$ denotes std. error.}
  \label{tab:results_1}
  \end{longtable}%

\subsection{Comparison of Robustness to different Bias Strengths $\kappa$}
We present additional WPEHE@$K$ results for the experiments in Section~\ref{sec:exp_bias} over increasing bias strength $\kappa$ and varying $K$.  
\begin{figure}[ht]

    \begin{subfigure}[b]{.24\textwidth}
    \centering
    \textbf{SW In-Sample} 
    \includegraphics[width=\textwidth]{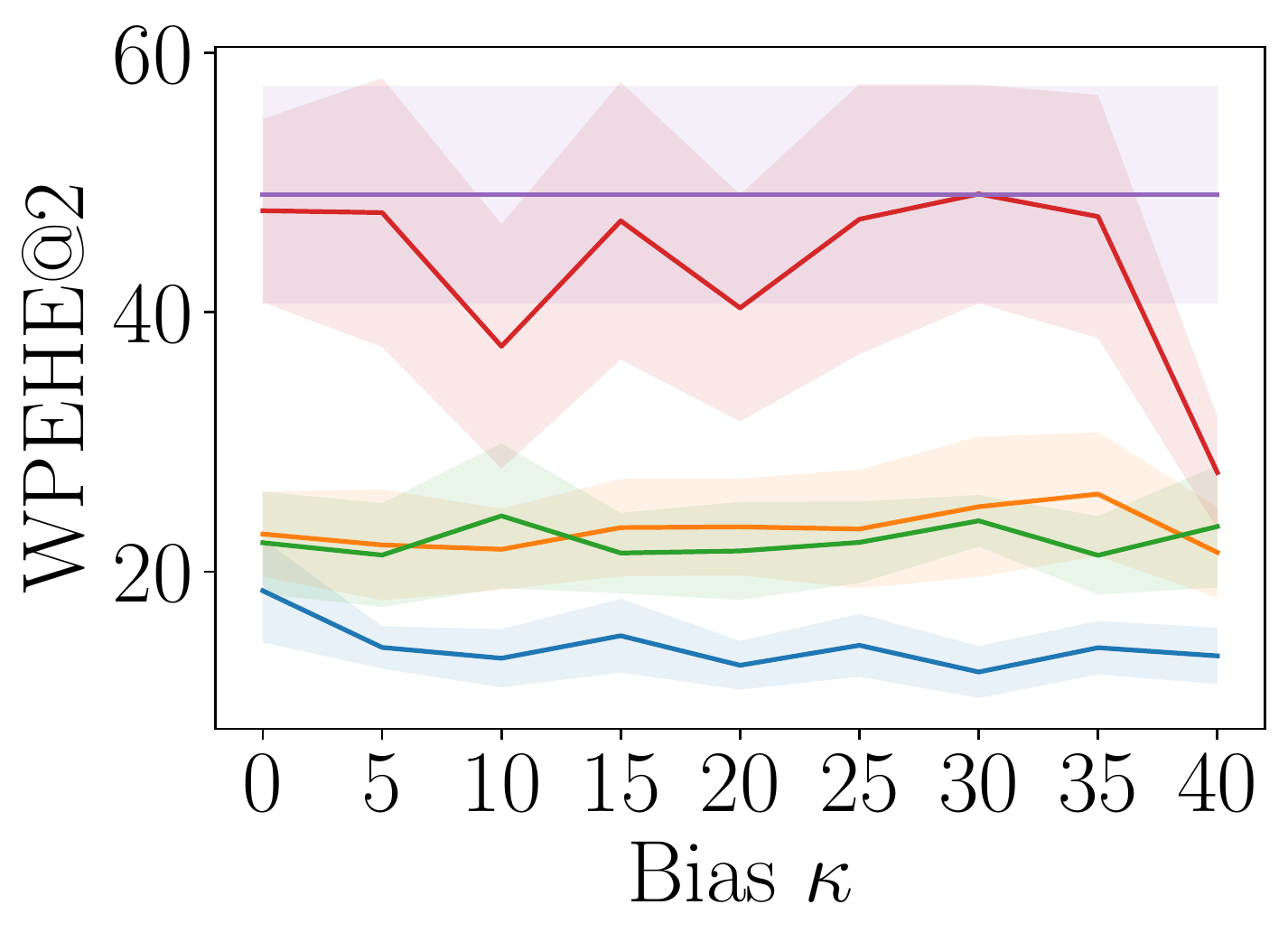}
    \end{subfigure}
    \begin{subfigure}[b]{.24\textwidth}
    \centering
    \textbf{SW Out-Sample} 
    \includegraphics[width=\textwidth]{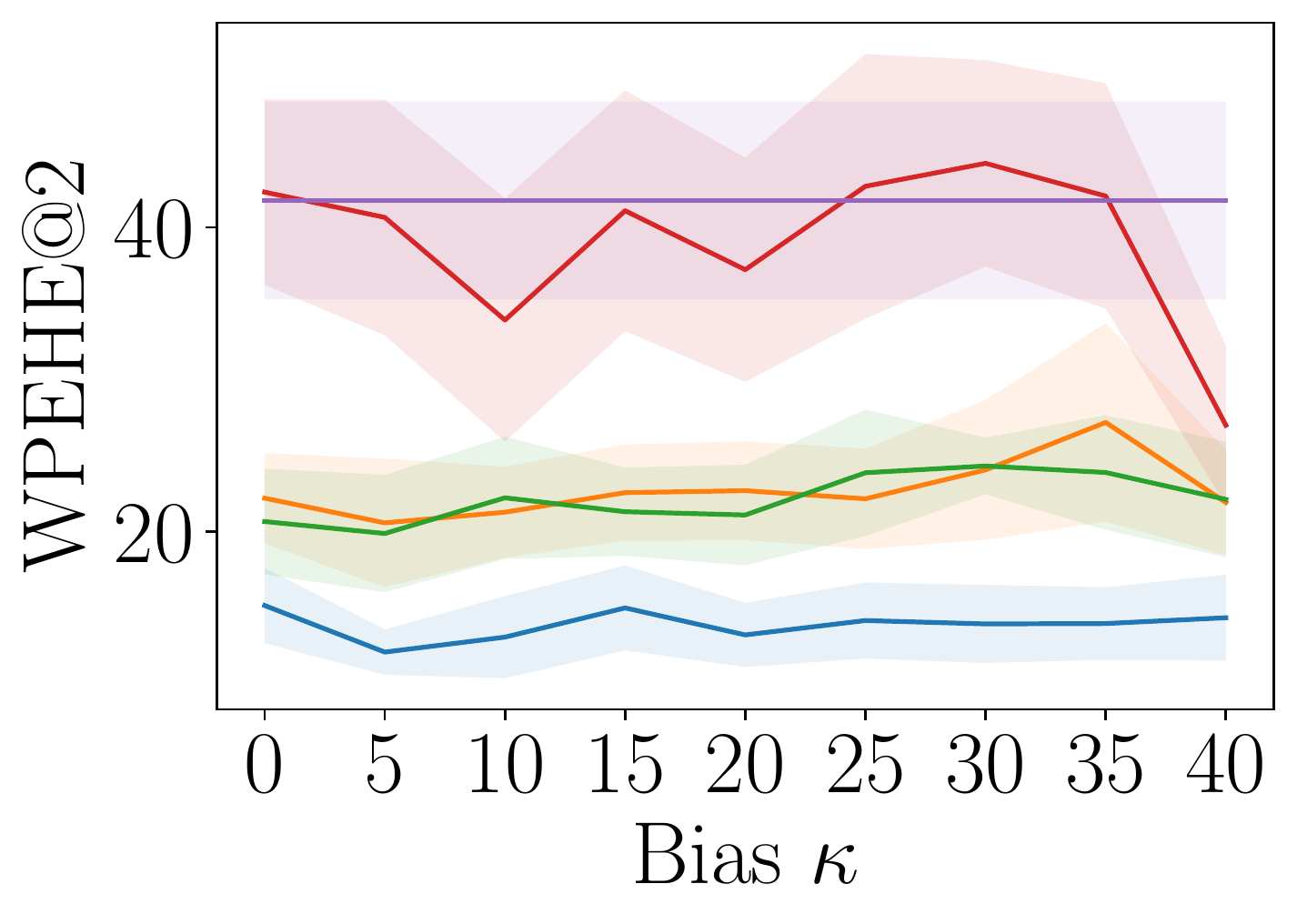}
    \end{subfigure}
    \begin{subfigure}[b]{.24\textwidth}
    \centering
    \textbf{TCGA In-Sample} 
    \includegraphics[width=\textwidth]{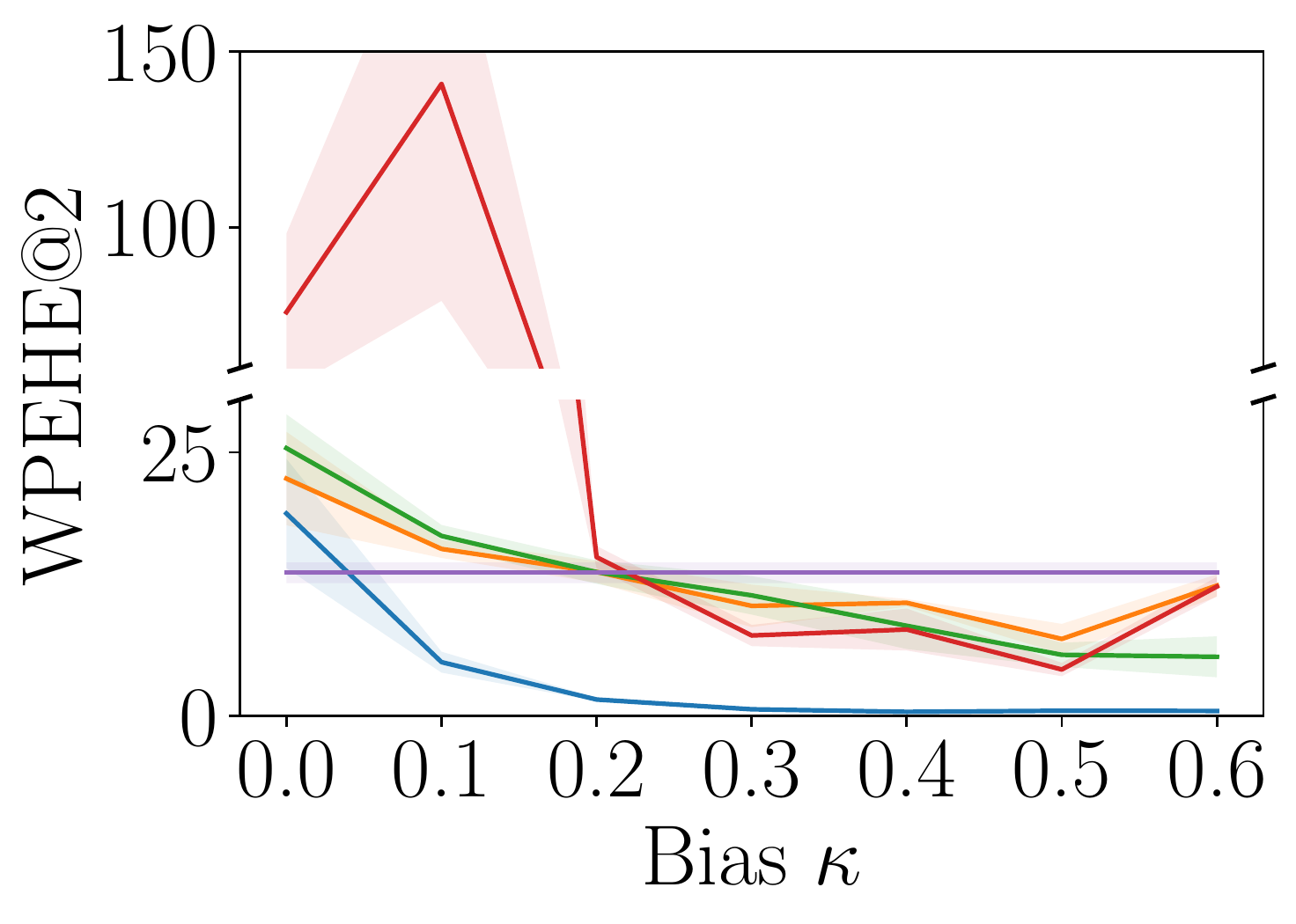}
    \end{subfigure}
    \begin{subfigure}[b]{.24\textwidth}
    \centering
    \textbf{TCGA Out-Sample} 
    \includegraphics[width=\textwidth]{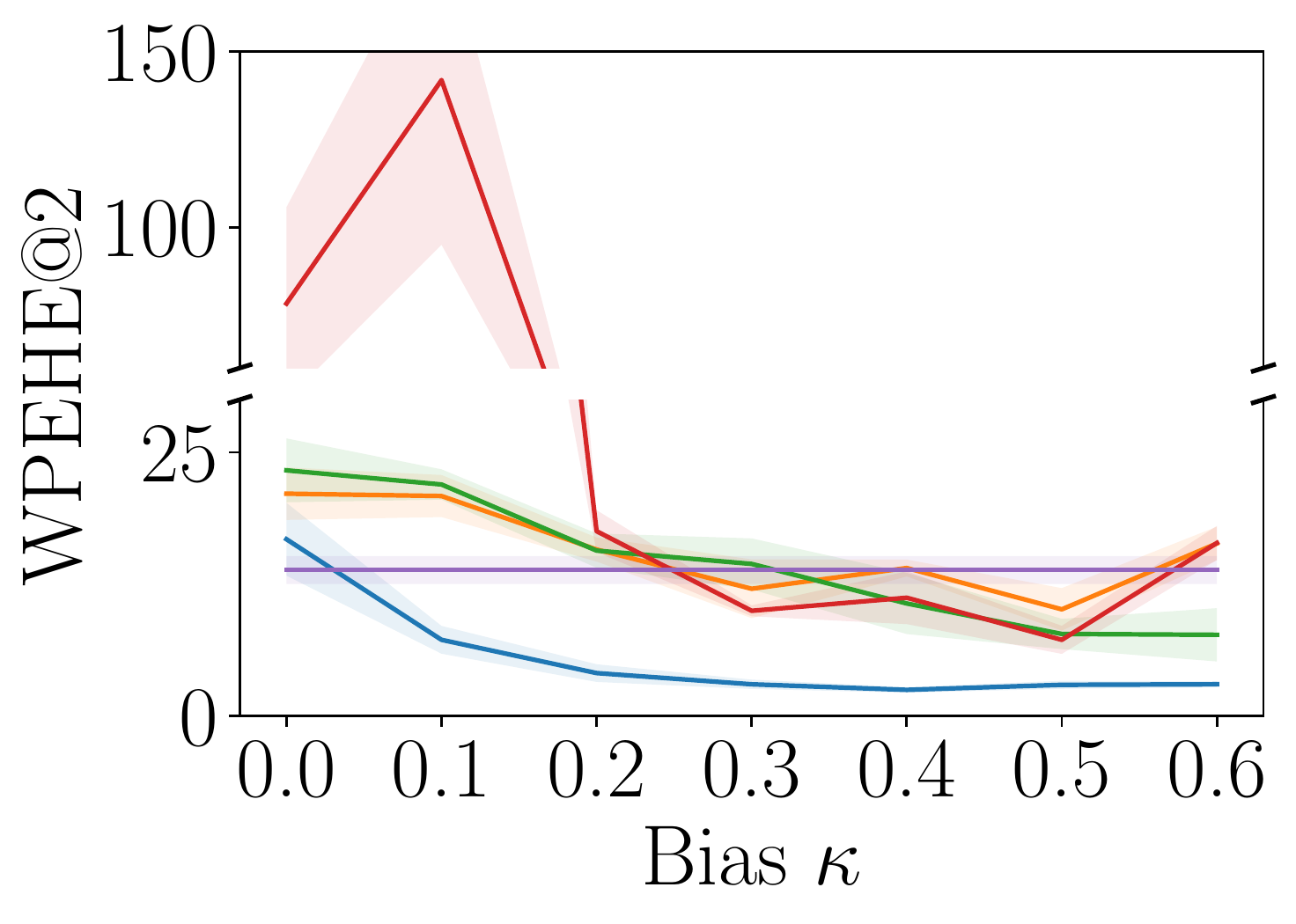}
    \end{subfigure}

    \begin{subfigure}[b]{.24\textwidth}
    \includegraphics[width=\textwidth]{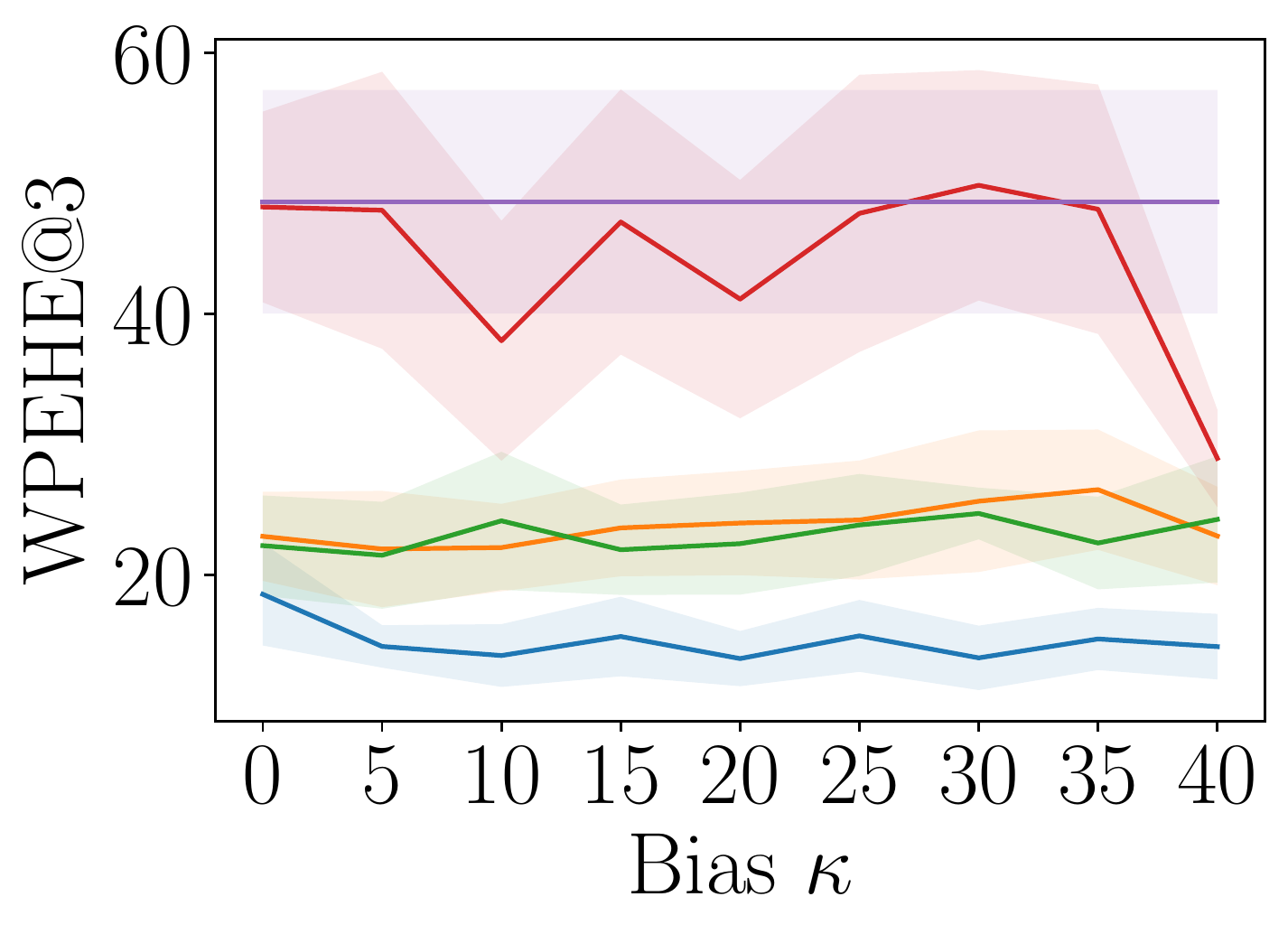}
    \end{subfigure}
    \begin{subfigure}[b]{.24\textwidth}
    \includegraphics[width=\textwidth]{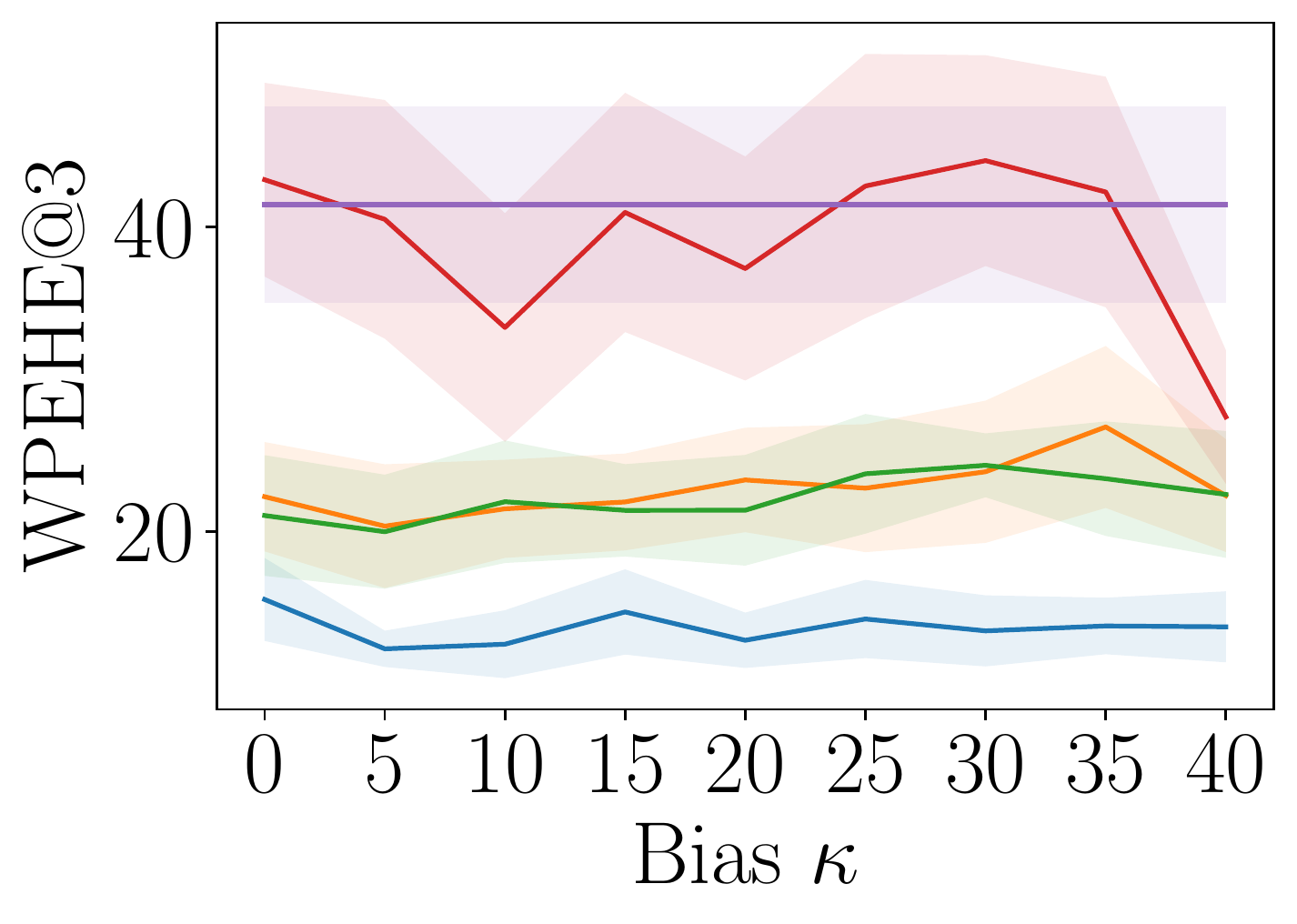}
    \end{subfigure}\begin{subfigure}[b]{.24\textwidth}
    \includegraphics[width=\textwidth]{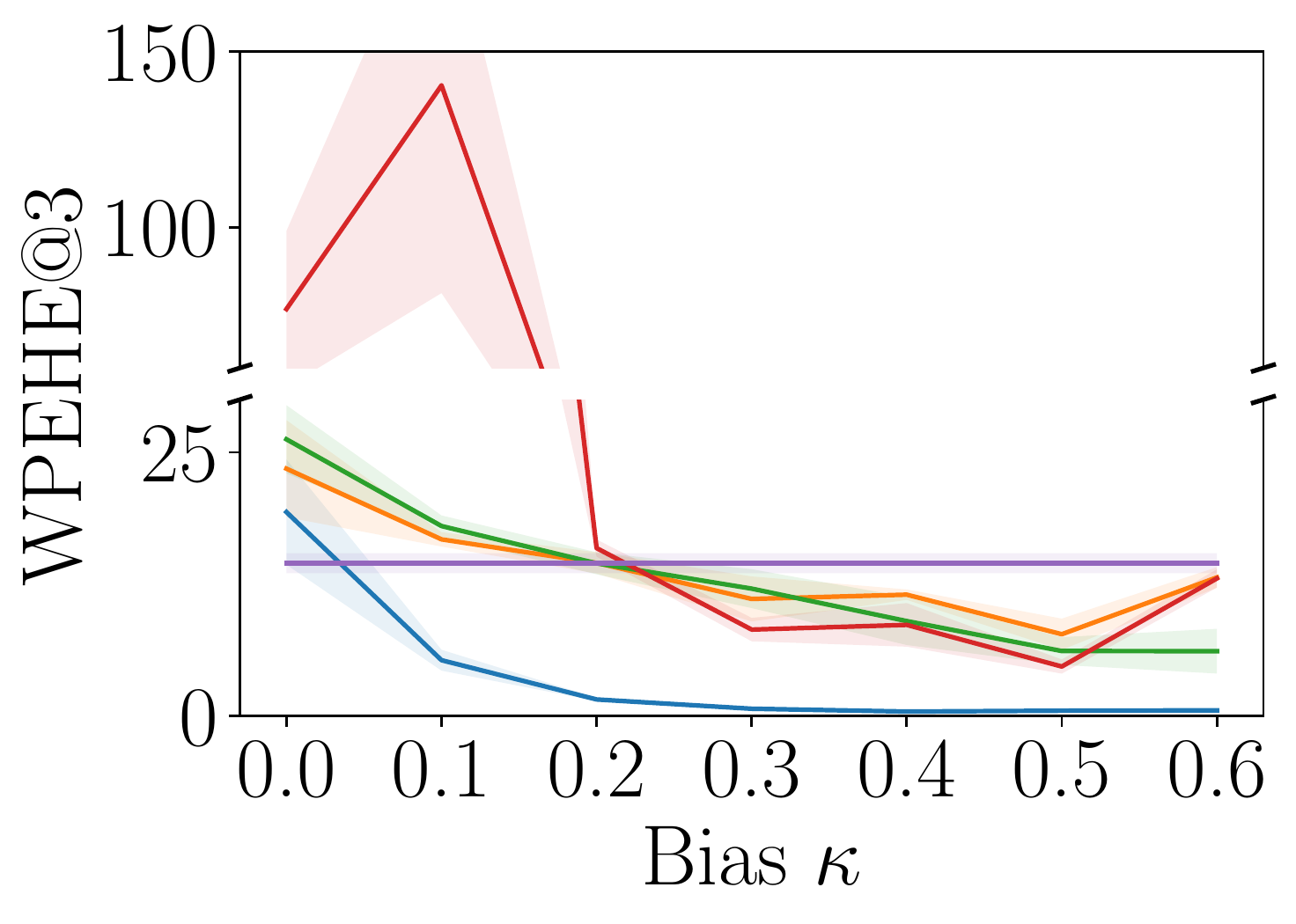}
    \end{subfigure}
    \begin{subfigure}[b]{.24\textwidth}
    \includegraphics[width=\textwidth]{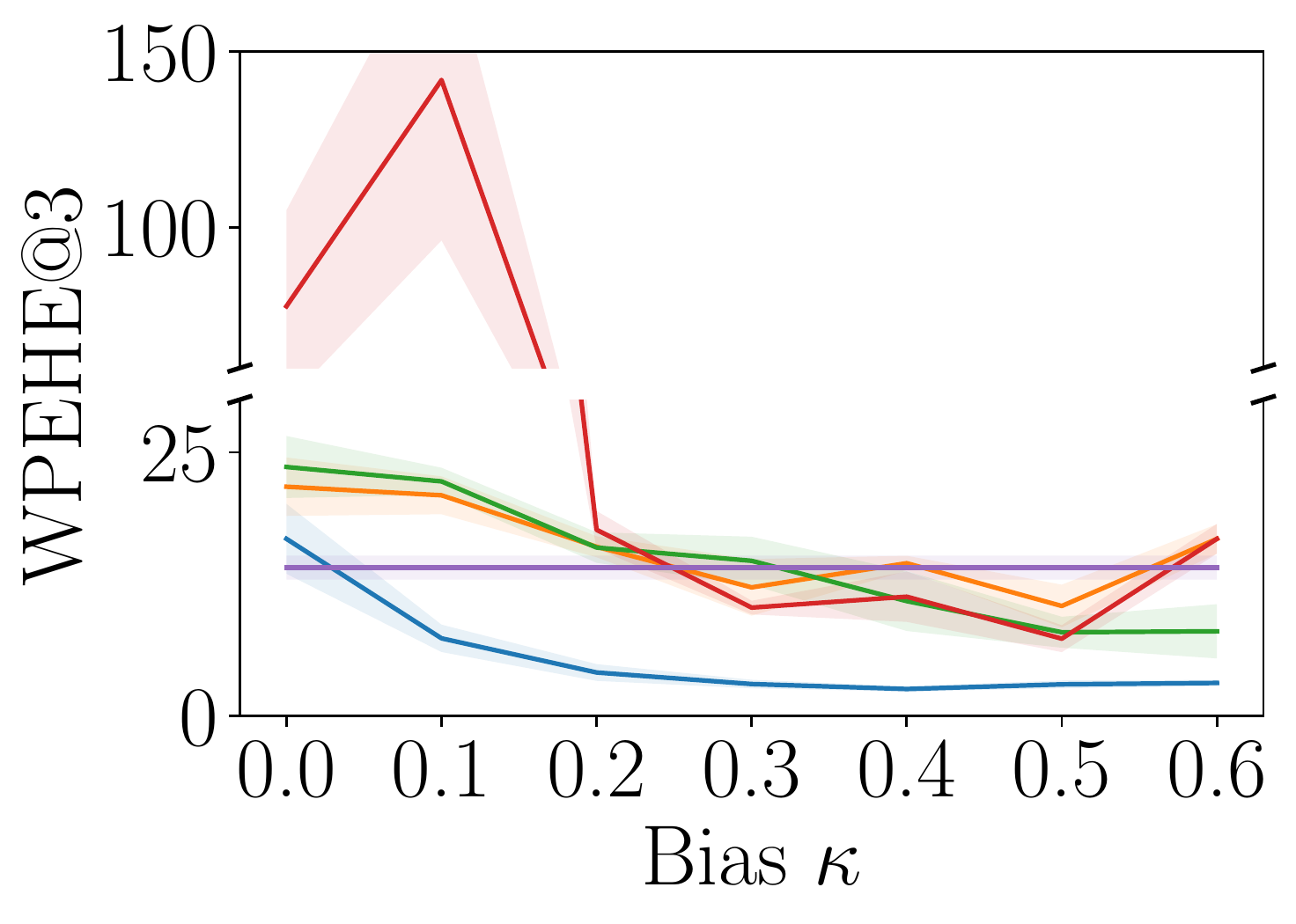}
    \end{subfigure}

    \begin{subfigure}[b]{.24\textwidth}
    \includegraphics[width=\textwidth]{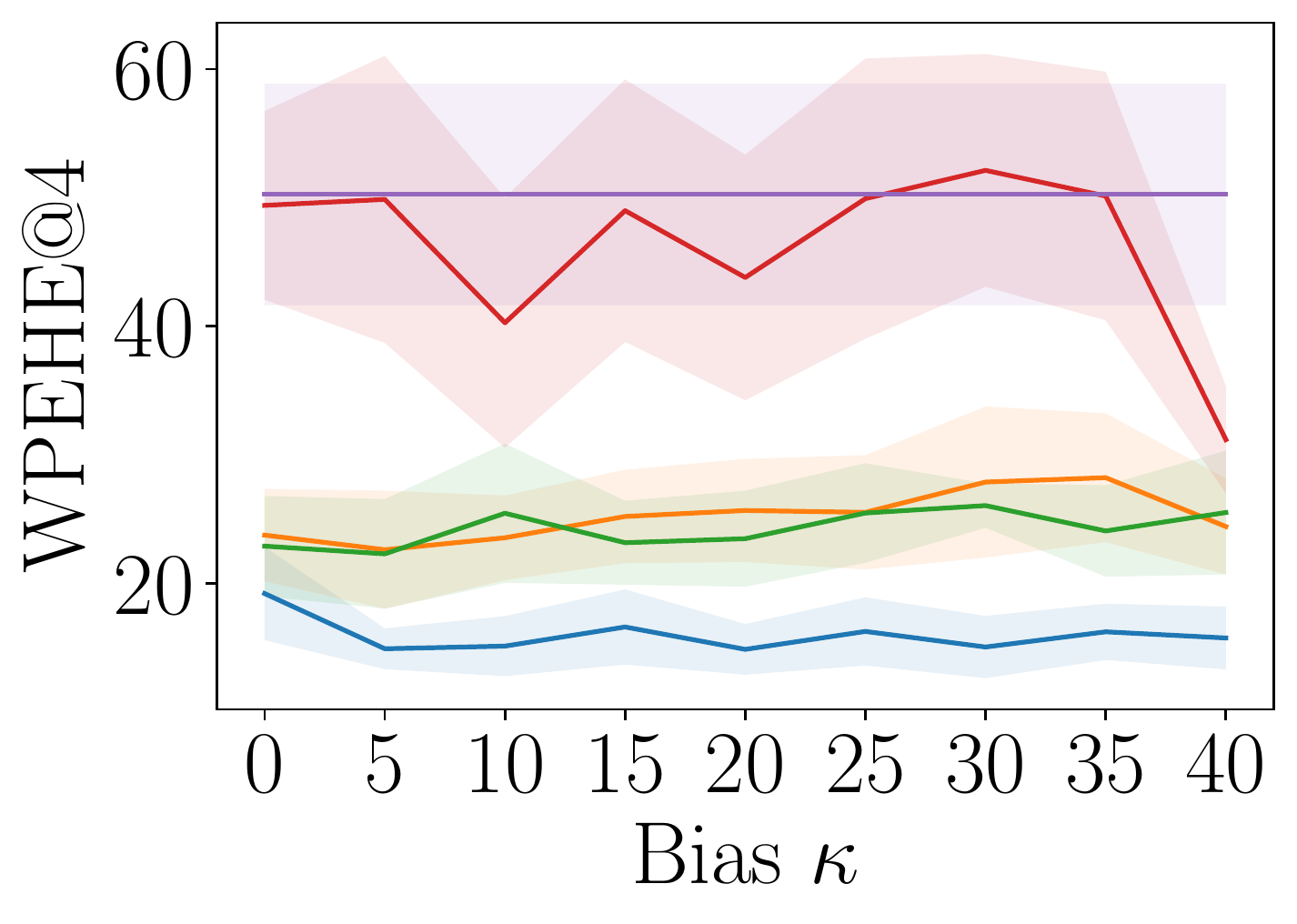}
    \end{subfigure}
    \begin{subfigure}[b]{.24\textwidth}
    \includegraphics[width=\textwidth]{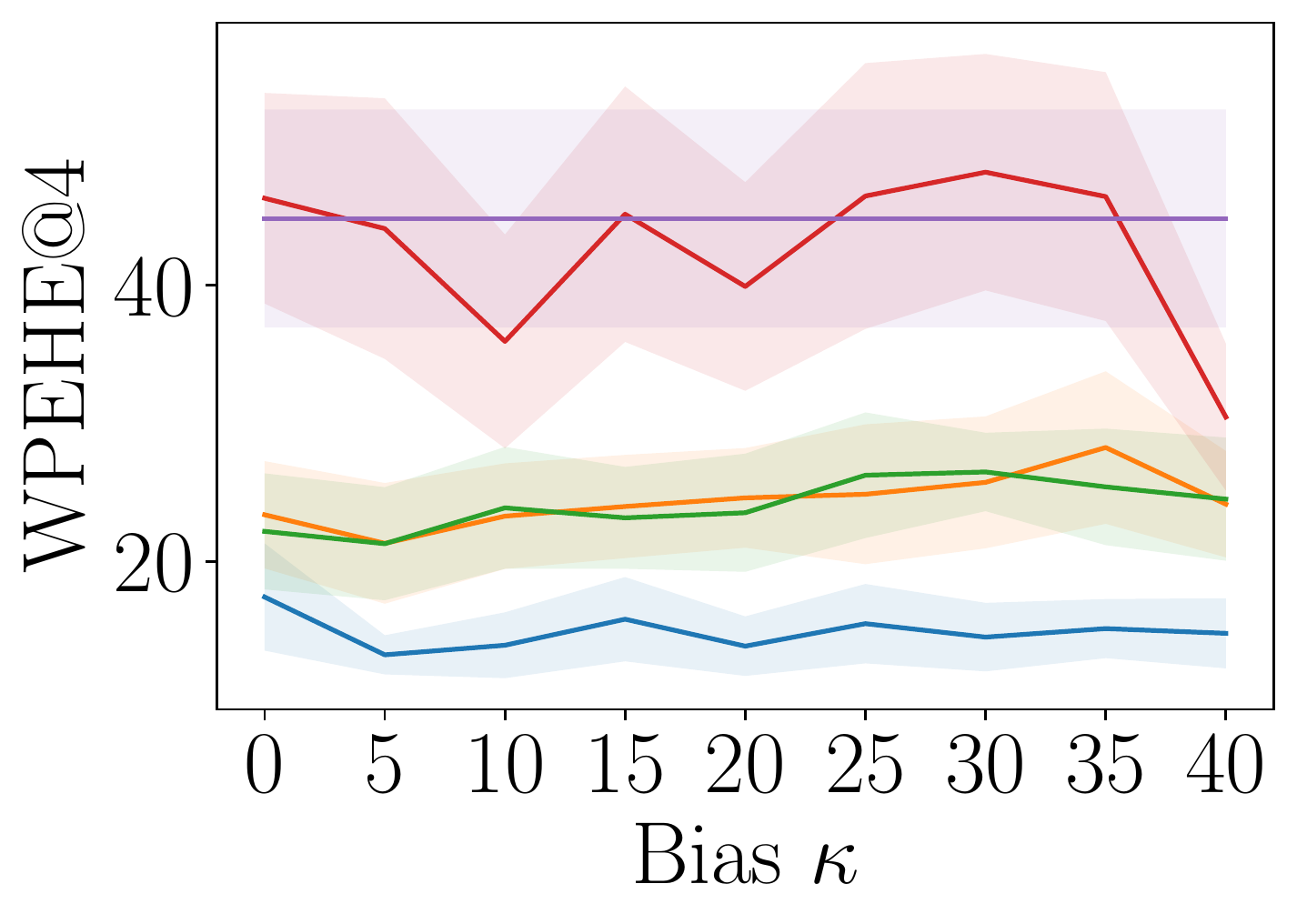}
    \end{subfigure}
    \begin{subfigure}[b]{.24\textwidth}
    \includegraphics[width=\textwidth]{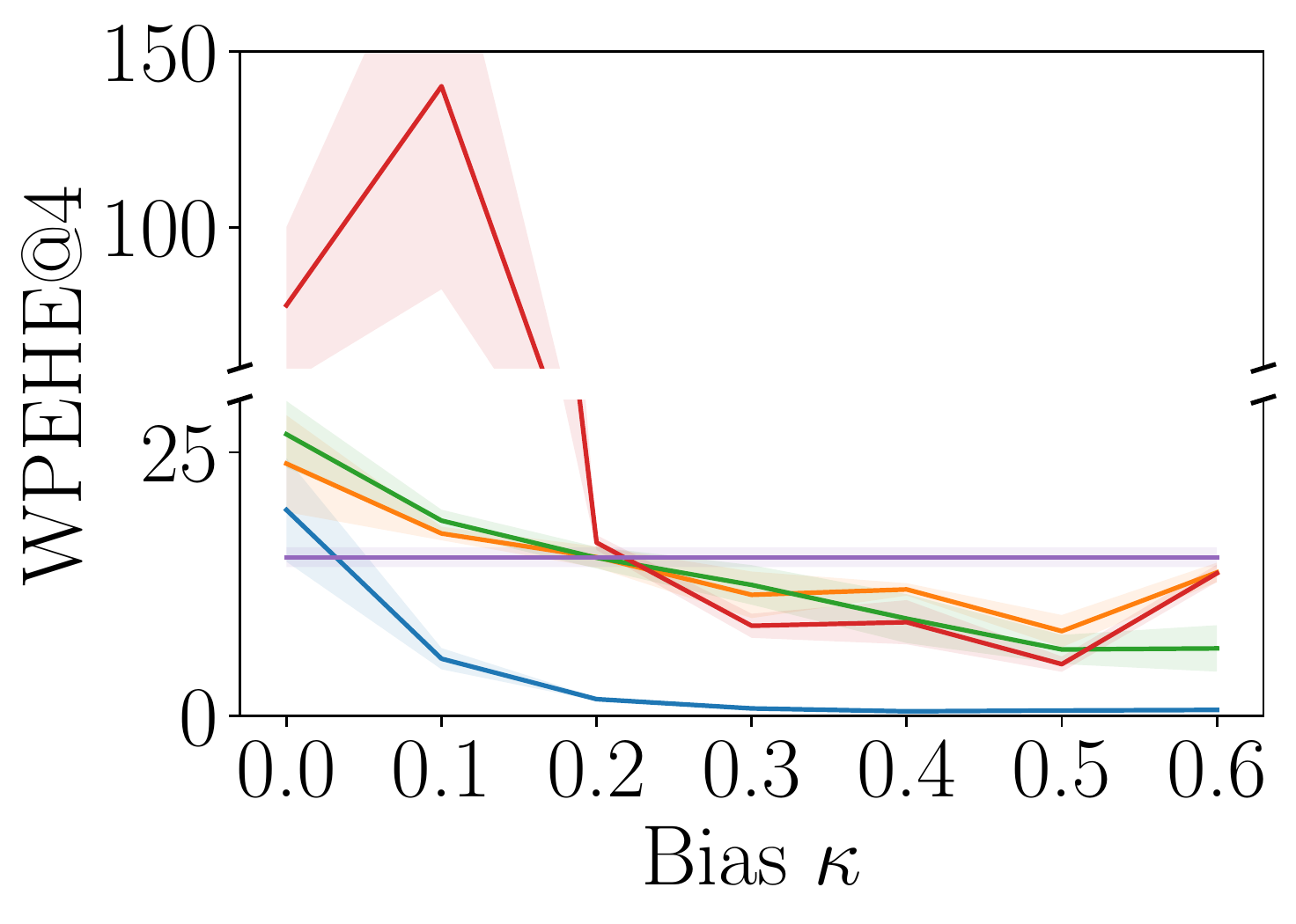}
    \end{subfigure}
    \begin{subfigure}[b]{.24\textwidth}
    \includegraphics[width=\textwidth]{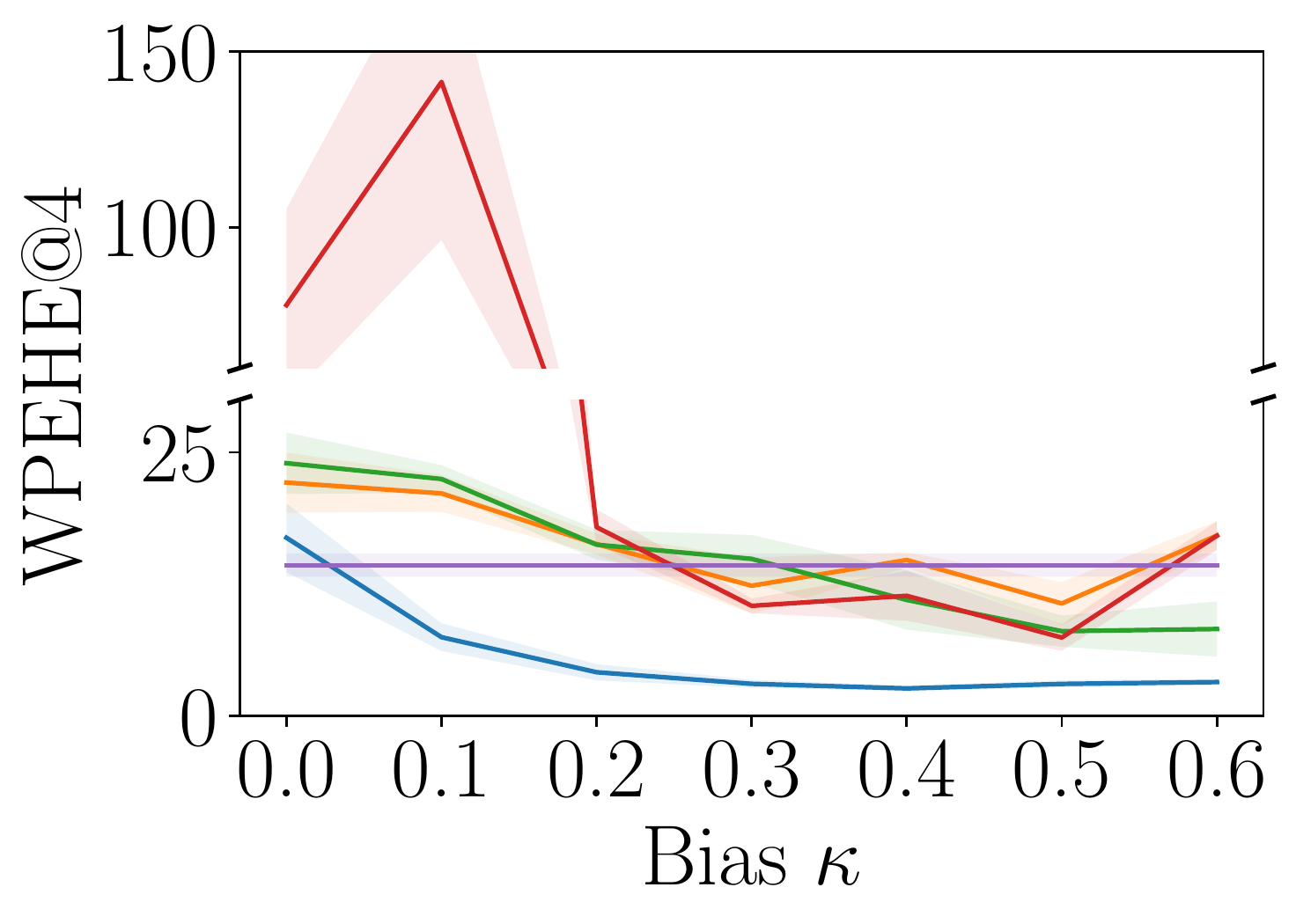}
    \end{subfigure}

    \begin{subfigure}[b]{.24\textwidth}
    \includegraphics[width=\textwidth]{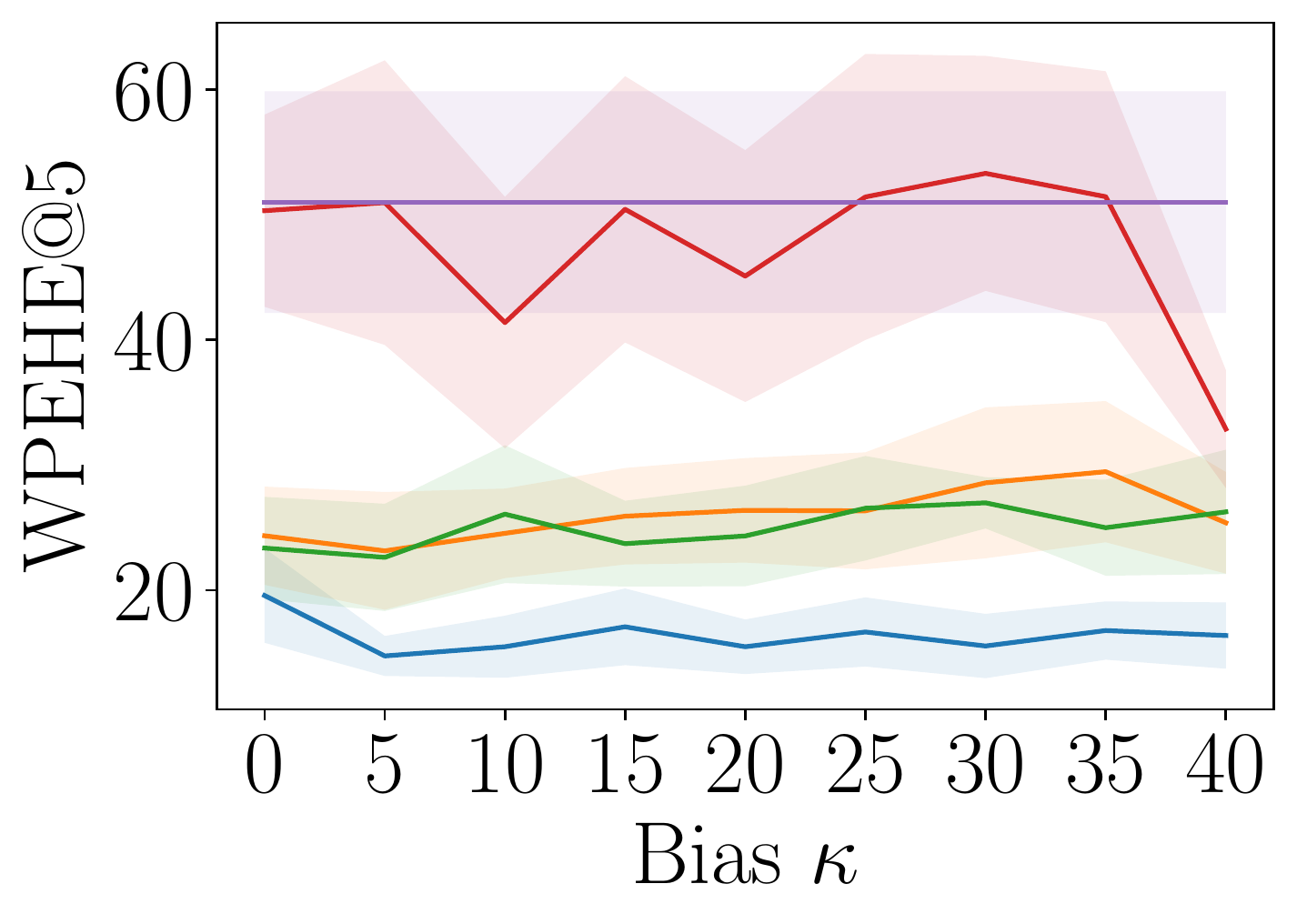}
    \end{subfigure}
    \begin{subfigure}[b]{.24\textwidth}
    \includegraphics[width=\textwidth]{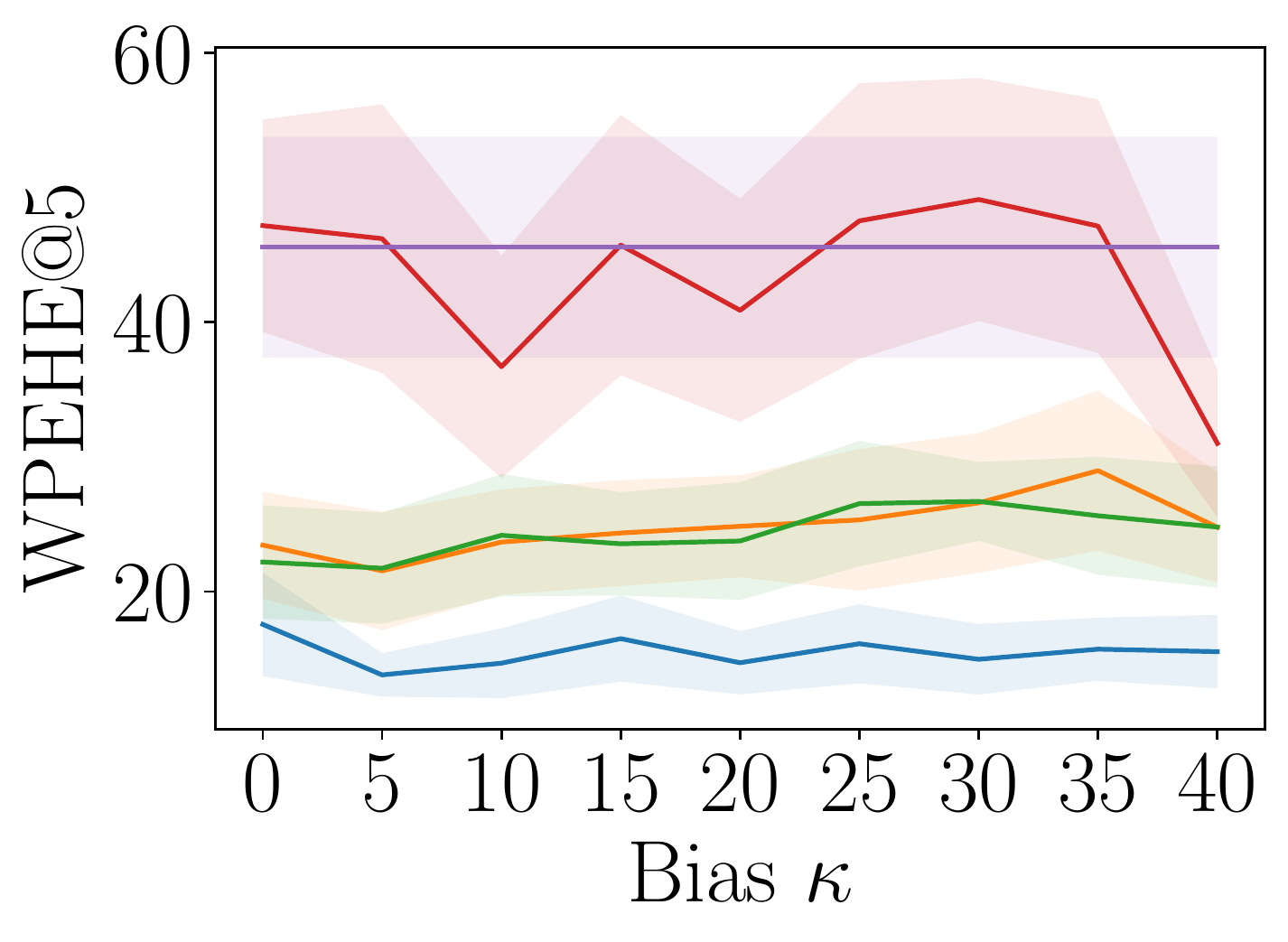}
    \end{subfigure}
    \begin{subfigure}[b]{.24\textwidth}
    \includegraphics[width=\textwidth]{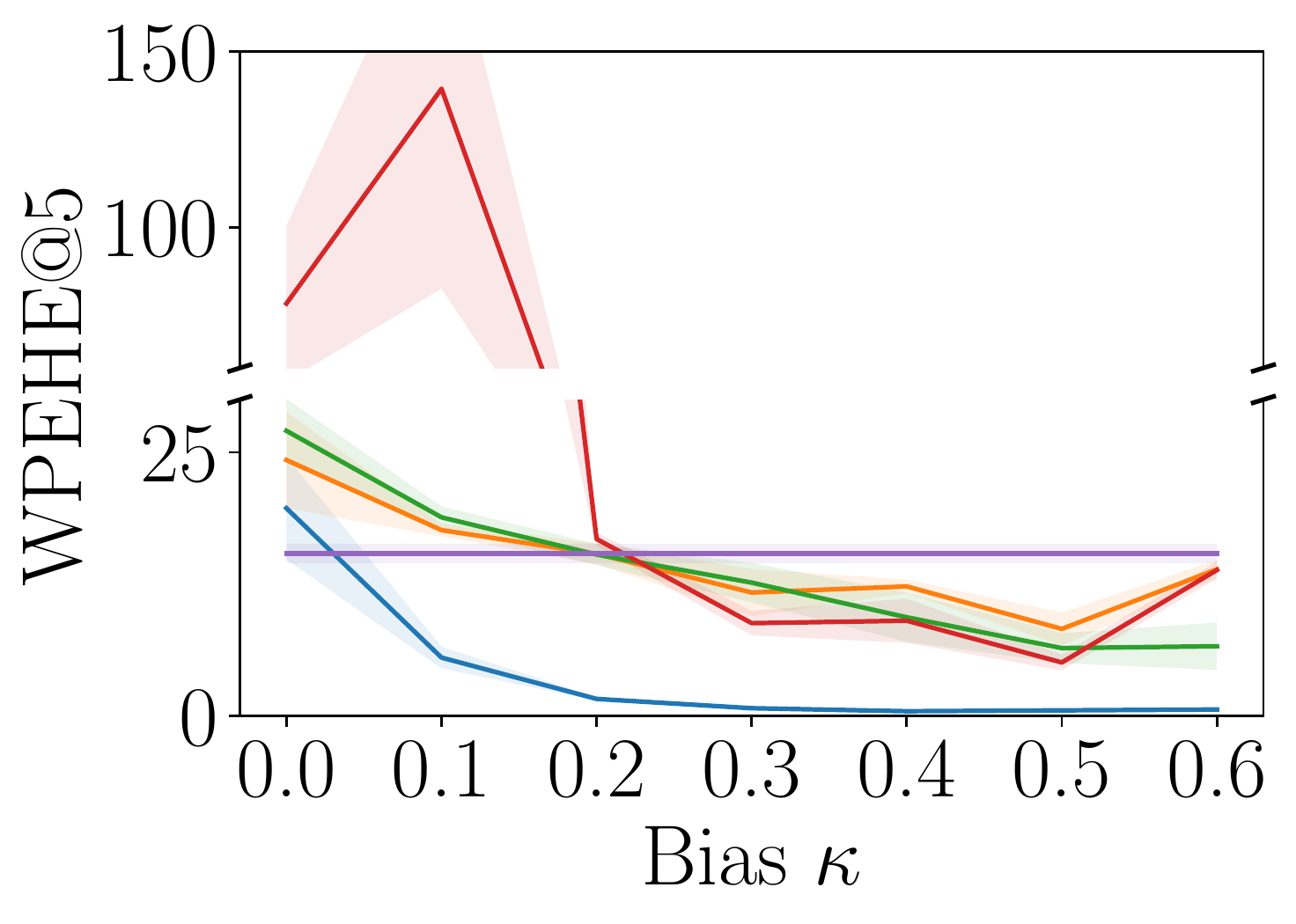}
    \end{subfigure}
    \begin{subfigure}[b]{.24\textwidth}
    \includegraphics[width=\textwidth]{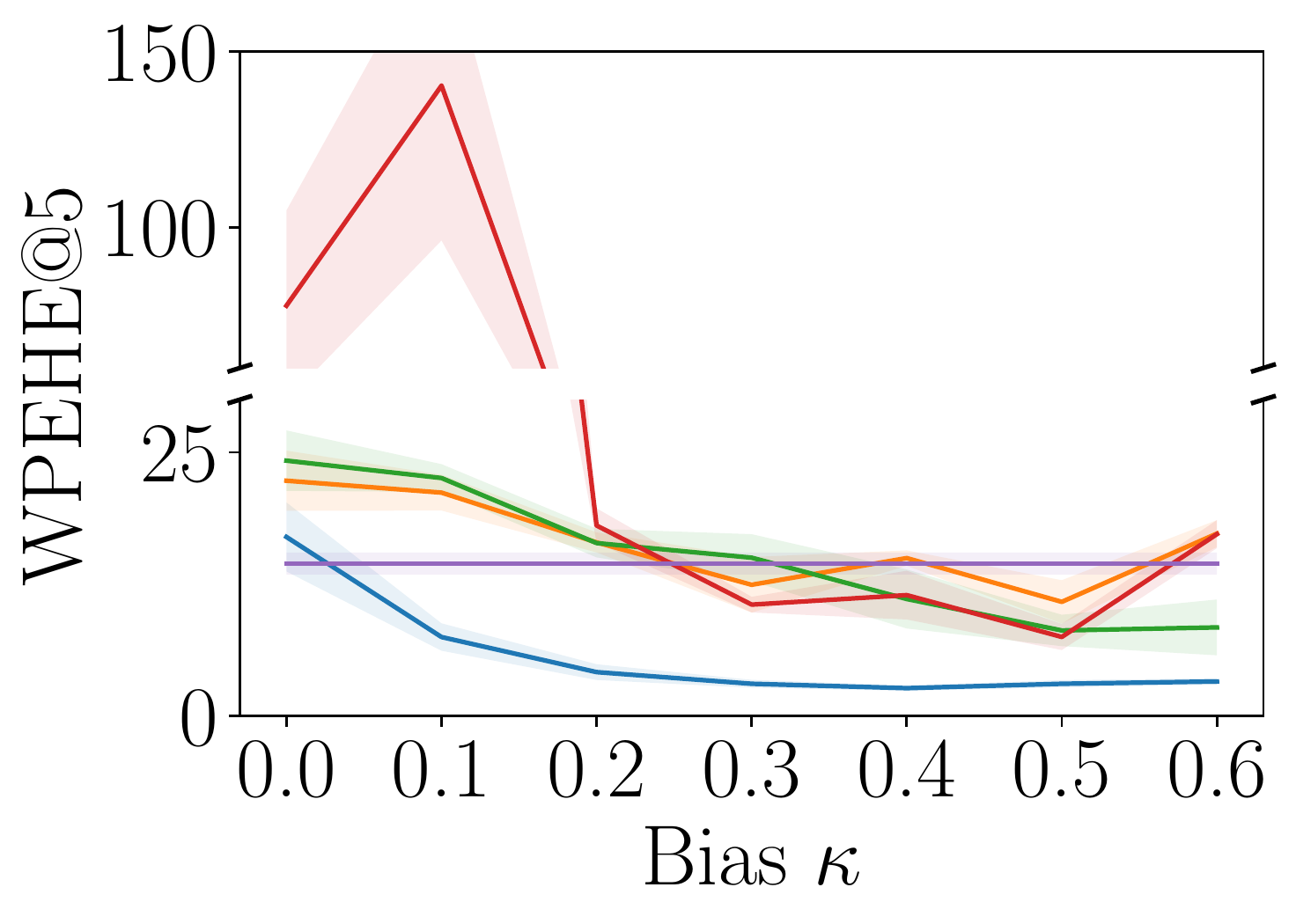}
    \end{subfigure}
    
    \begin{subfigure}[b]{.24\textwidth}
    \includegraphics[width=\textwidth]{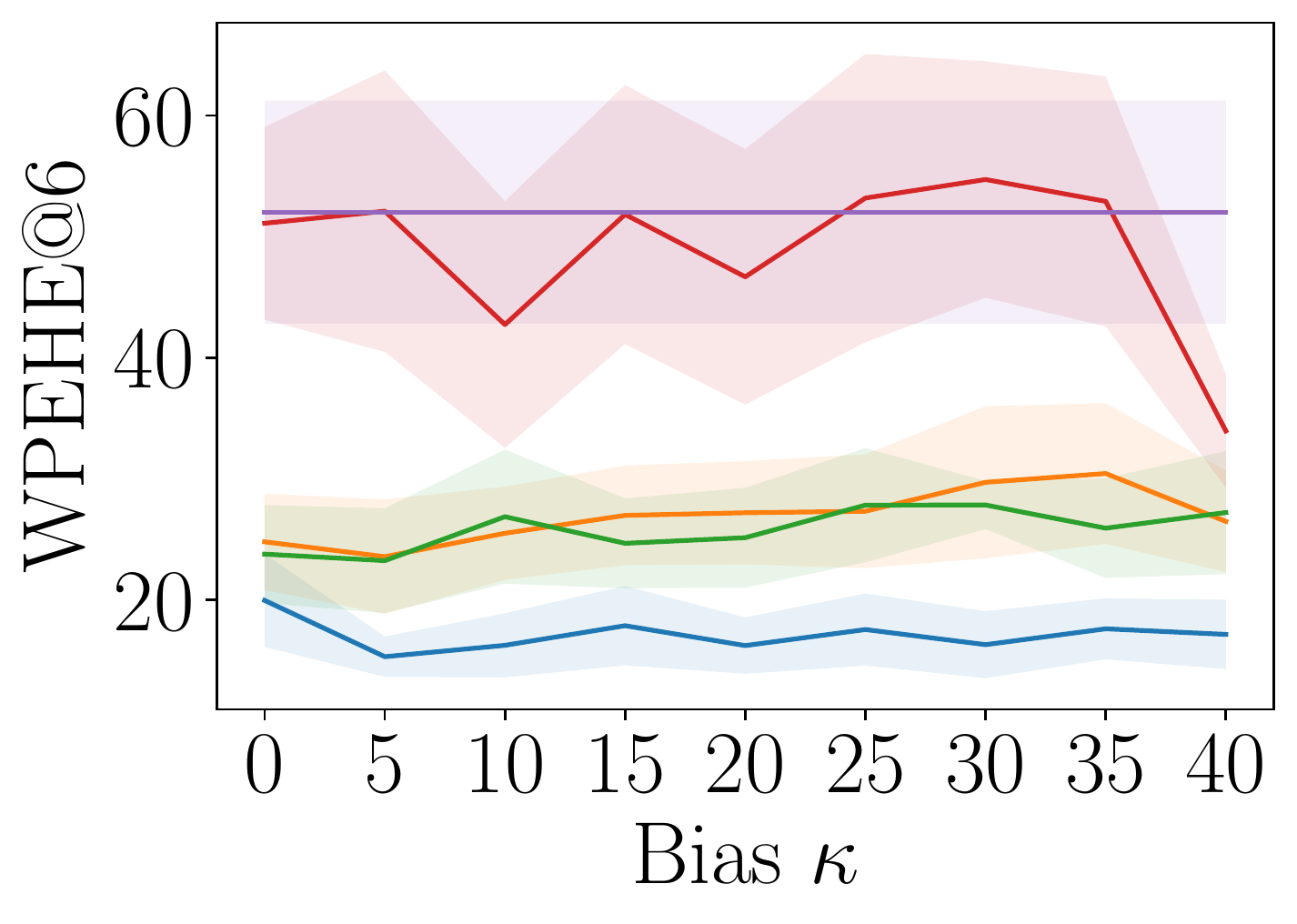}
    \end{subfigure}
    \begin{subfigure}[b]{.24\textwidth}
    \includegraphics[width=\textwidth]{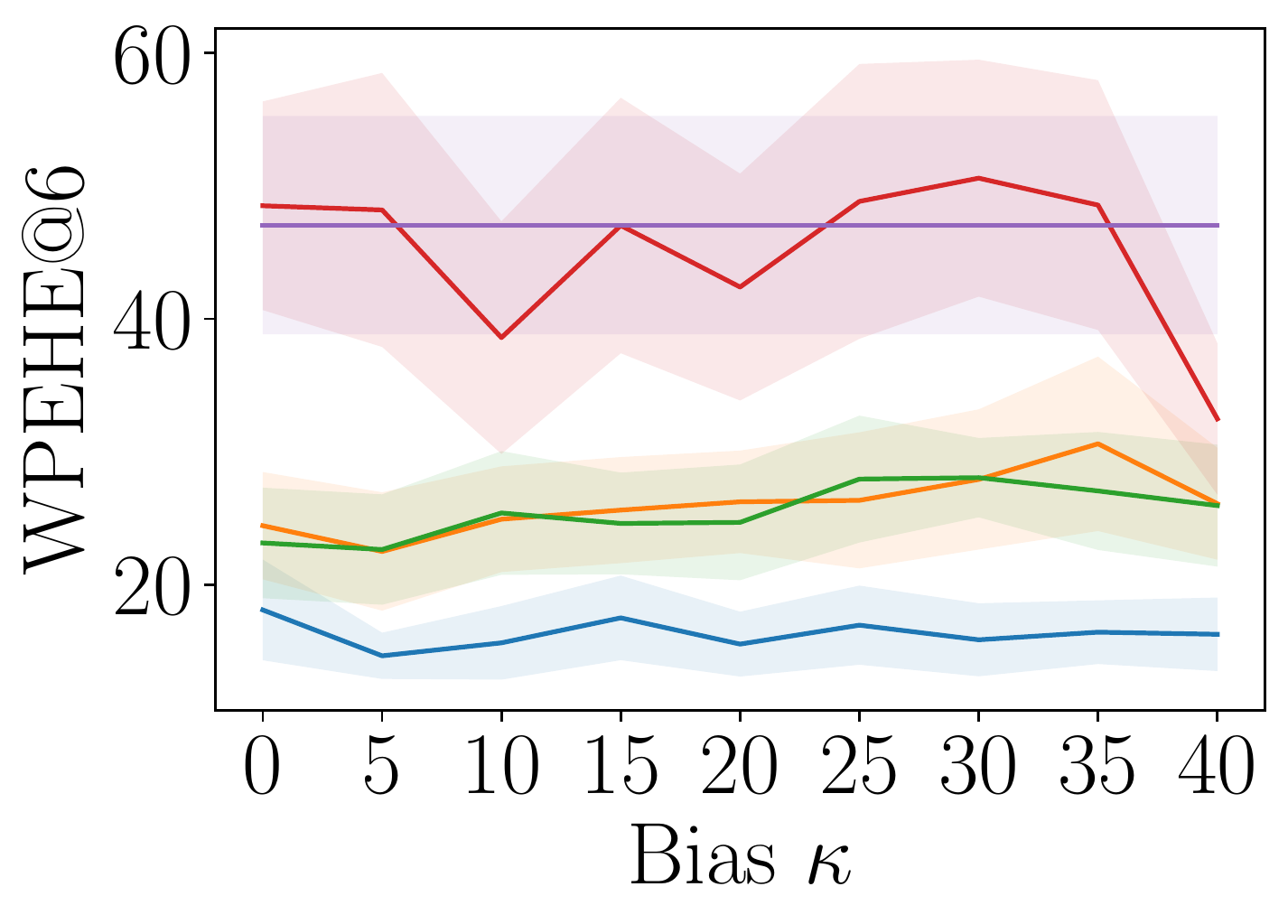}
    \end{subfigure}  
    \begin{subfigure}[b]{.24\textwidth}
    \includegraphics[width=\textwidth]{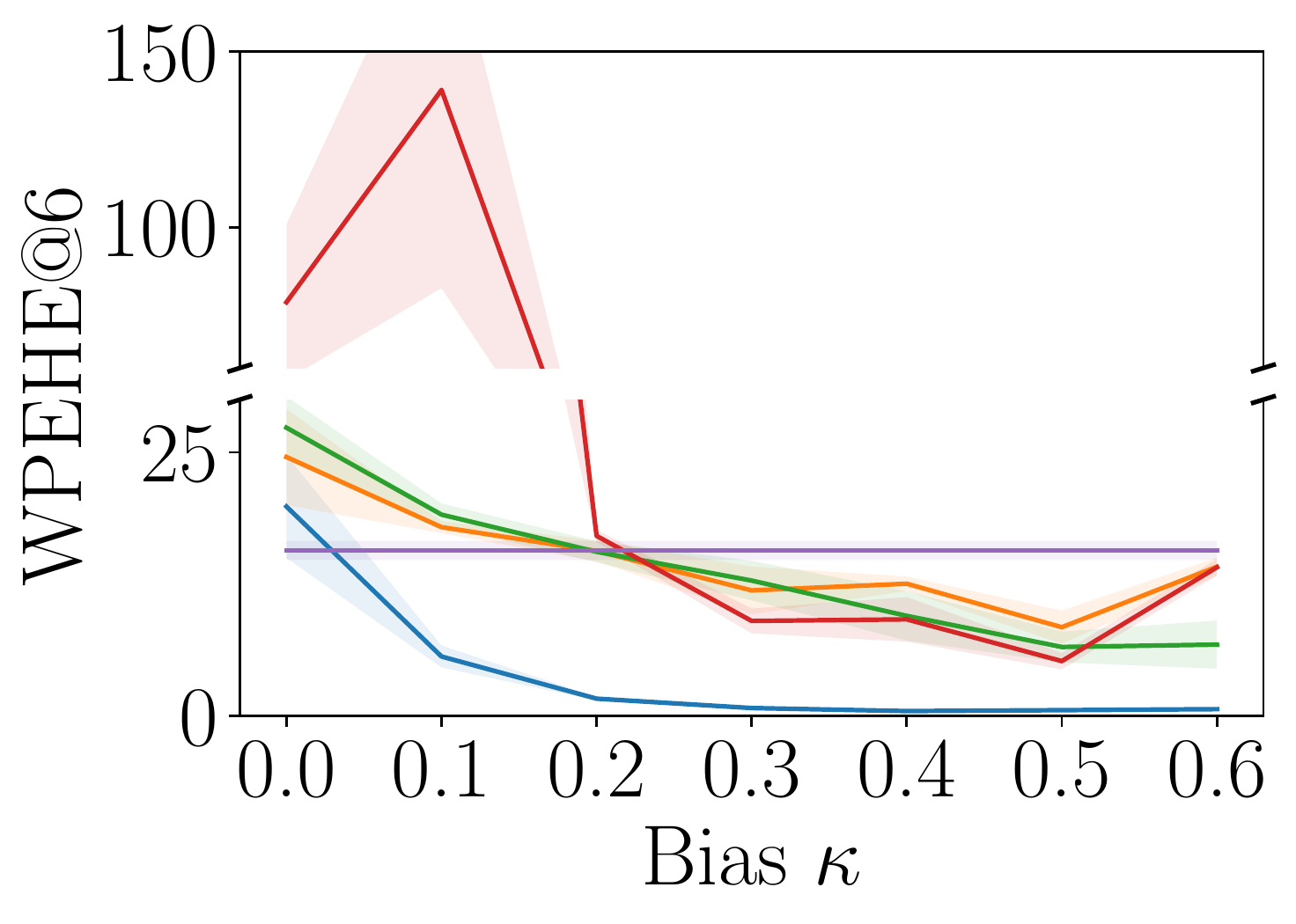}
    \end{subfigure}
    \begin{subfigure}[b]{.24\textwidth}
    \includegraphics[width=\textwidth]{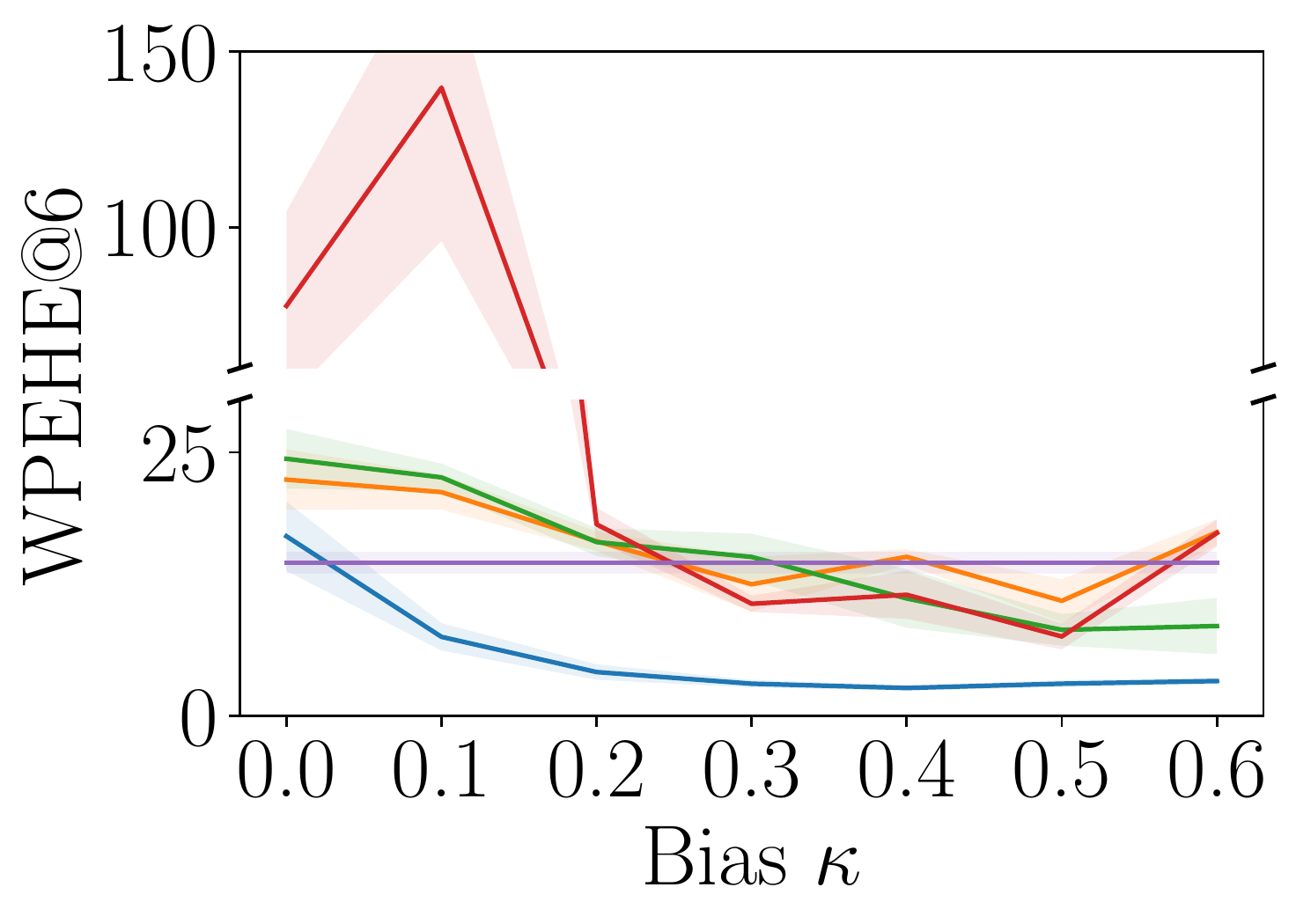}
    \end{subfigure}

    \end{figure}
    \begin{figure}[ht]\ContinuedFloat
    \centering

    \begin{subfigure}[b]{.24\textwidth}
    \centering
    \textbf{SW In-Sample} 
    \includegraphics[width=\textwidth]{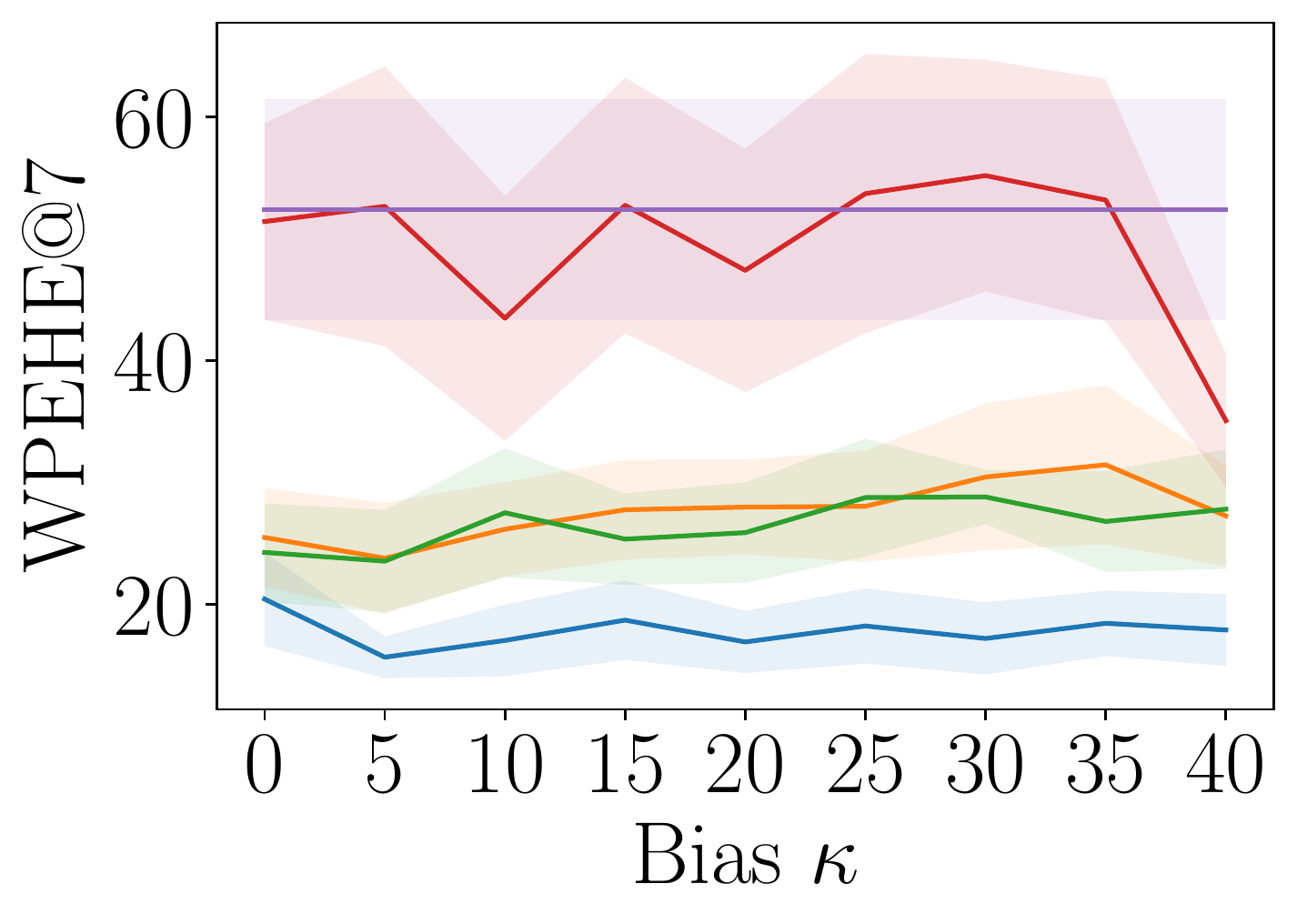}
    \end{subfigure}
    \begin{subfigure}[b]{.24\textwidth}
    \centering
    \textbf{SW Out-Sample} 
    \includegraphics[width=\textwidth]{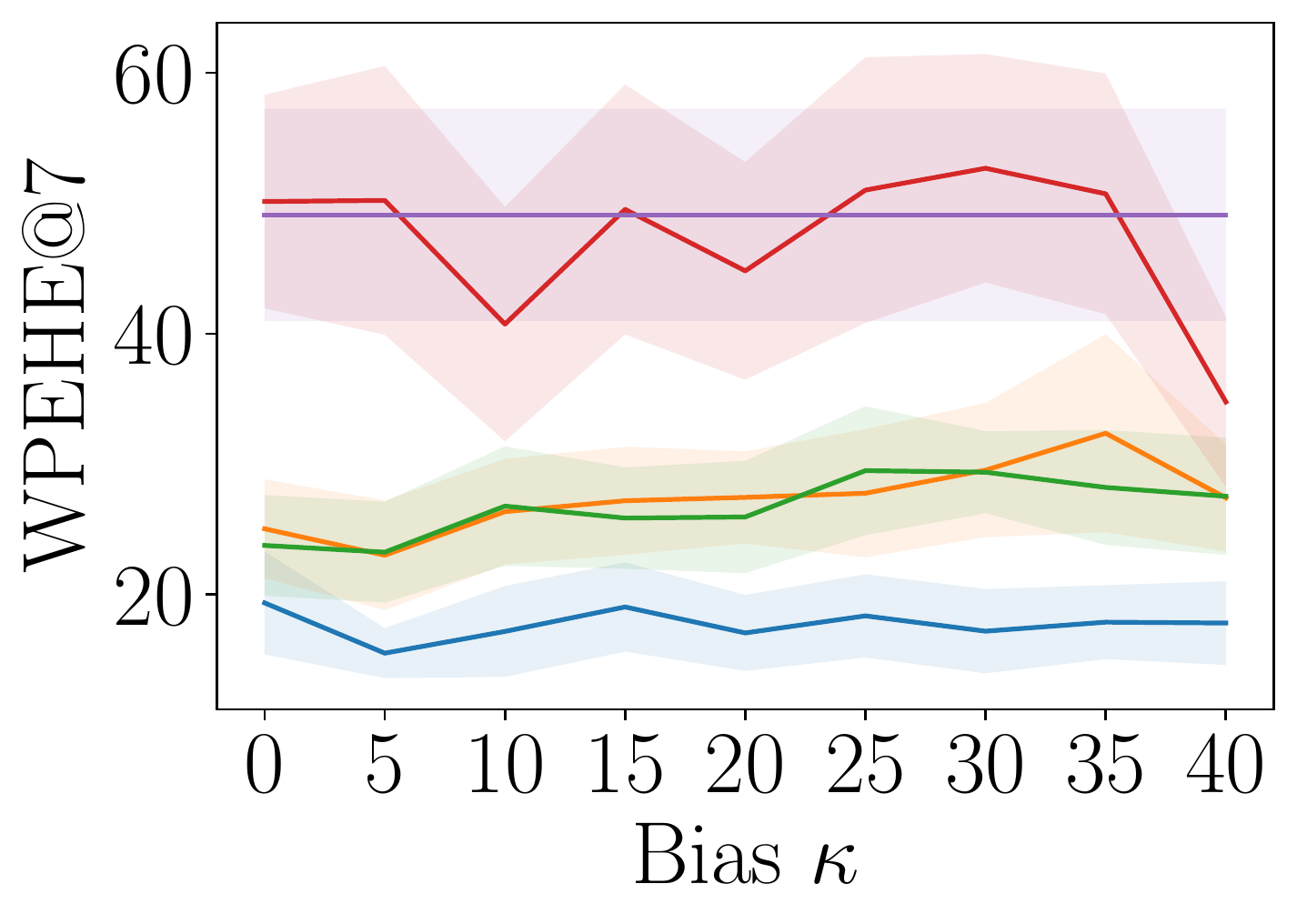}
    \end{subfigure}
    \begin{subfigure}[b]{.24\textwidth}
    \centering
    \textbf{TCGA In-Sample} 
    \includegraphics[width=\textwidth]{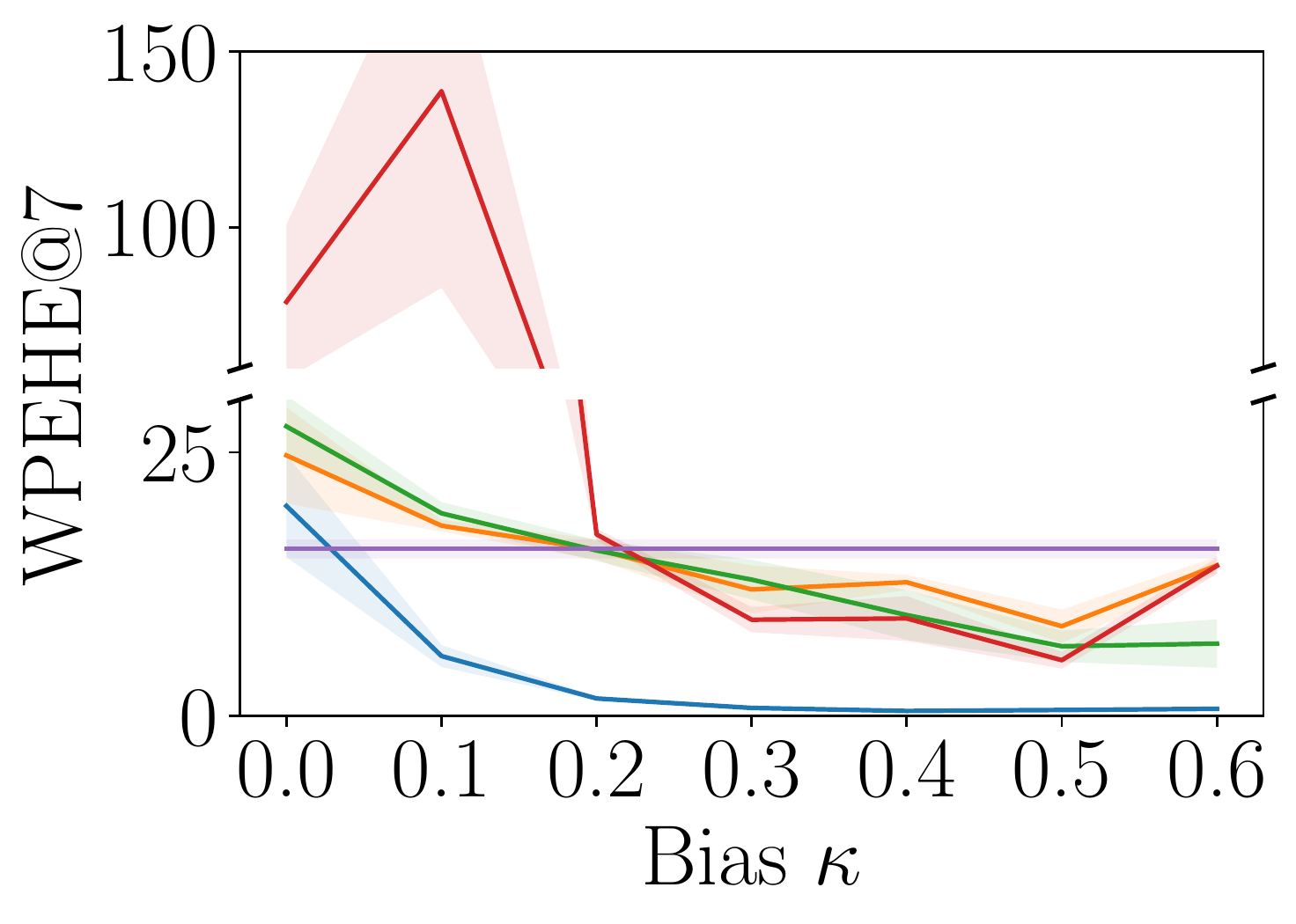}
    \end{subfigure}
    \begin{subfigure}[b]{.24\textwidth}
    \centering
    \textbf{TCGA Out-Sample} 
    \includegraphics[width=\textwidth]{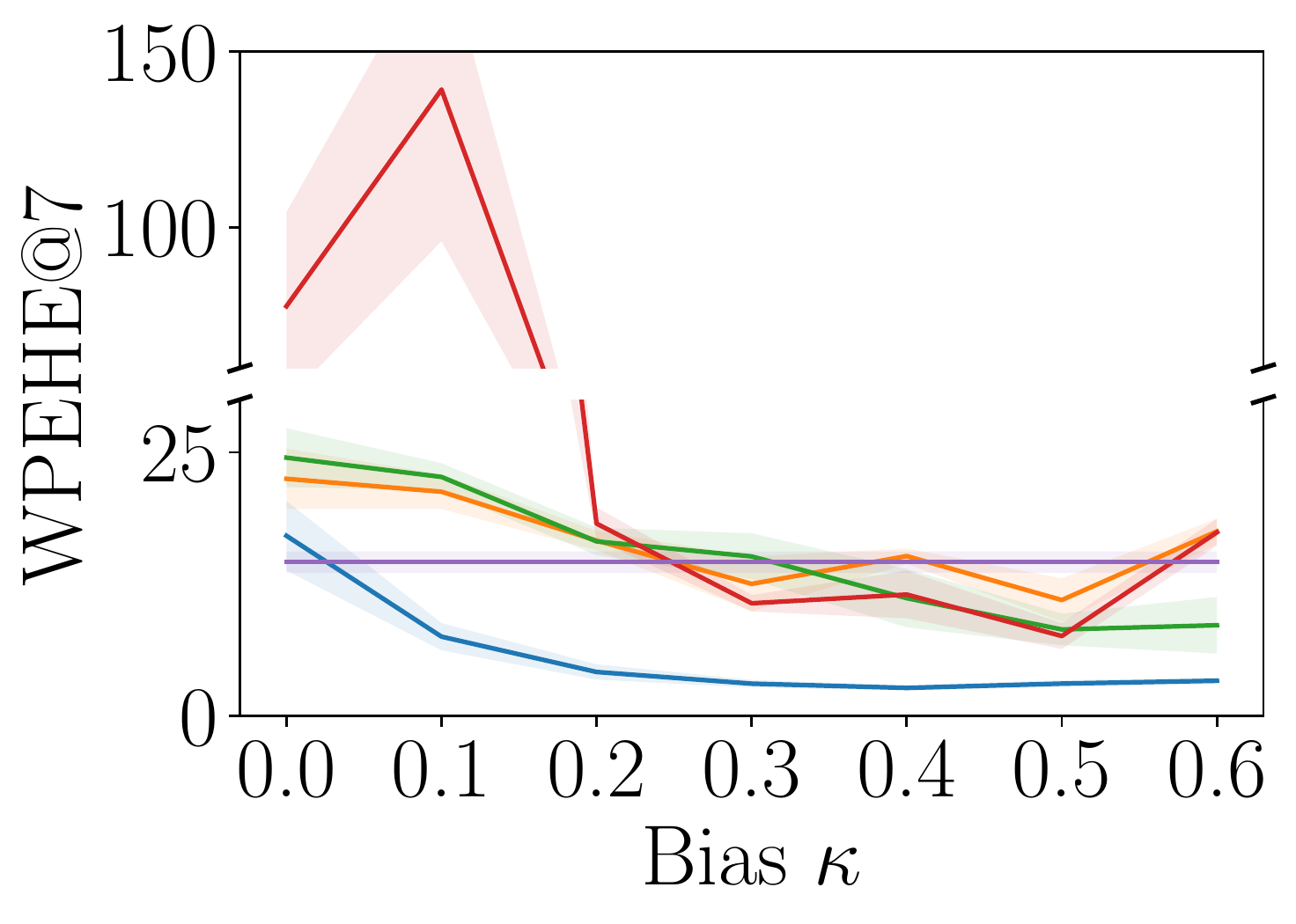}
    \end{subfigure}
    
    \begin{subfigure}[b]{.24\textwidth}
    \includegraphics[width=\textwidth]{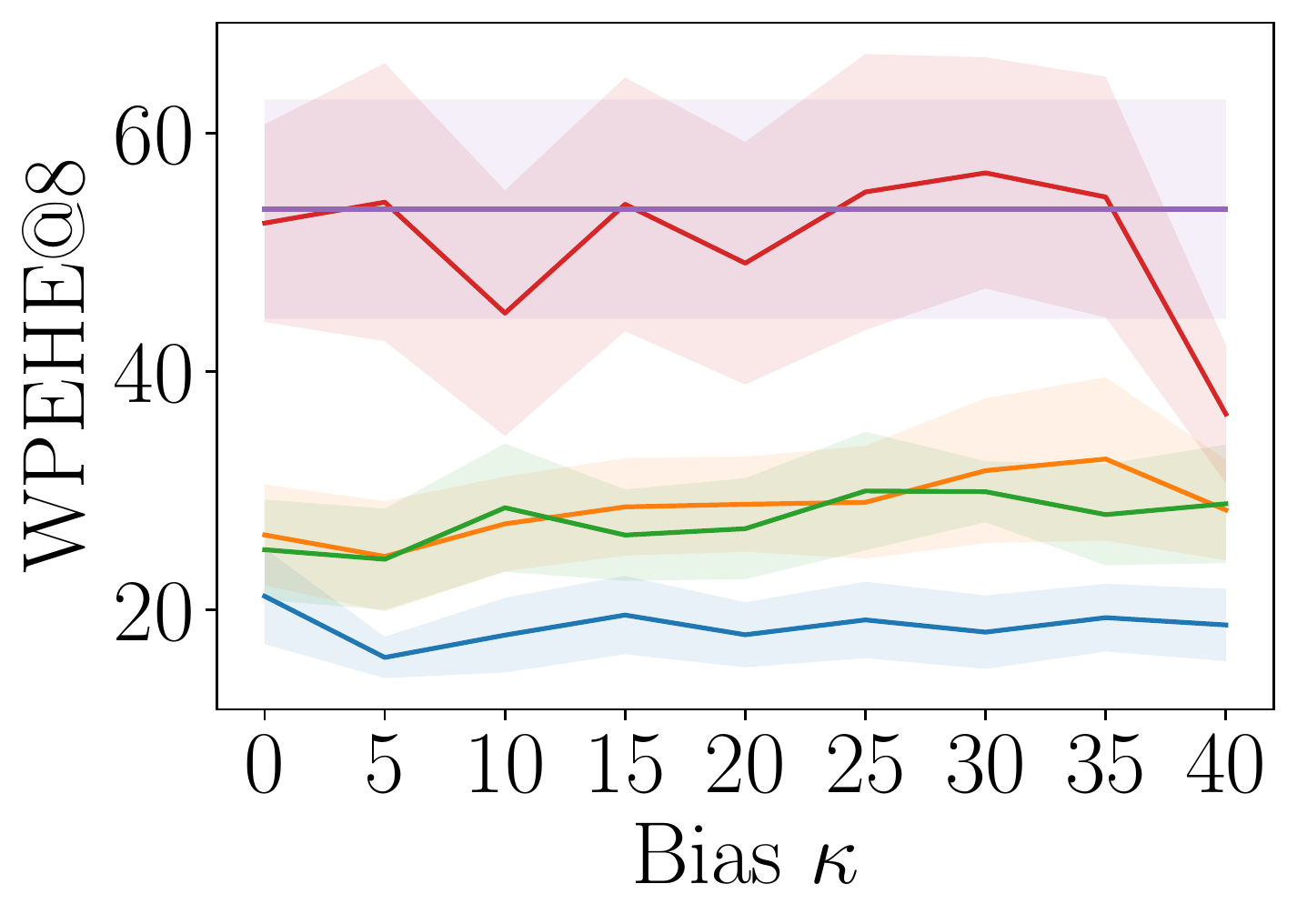}
    \end{subfigure}
    \begin{subfigure}[b]{.24\textwidth}
    \includegraphics[width=\textwidth]{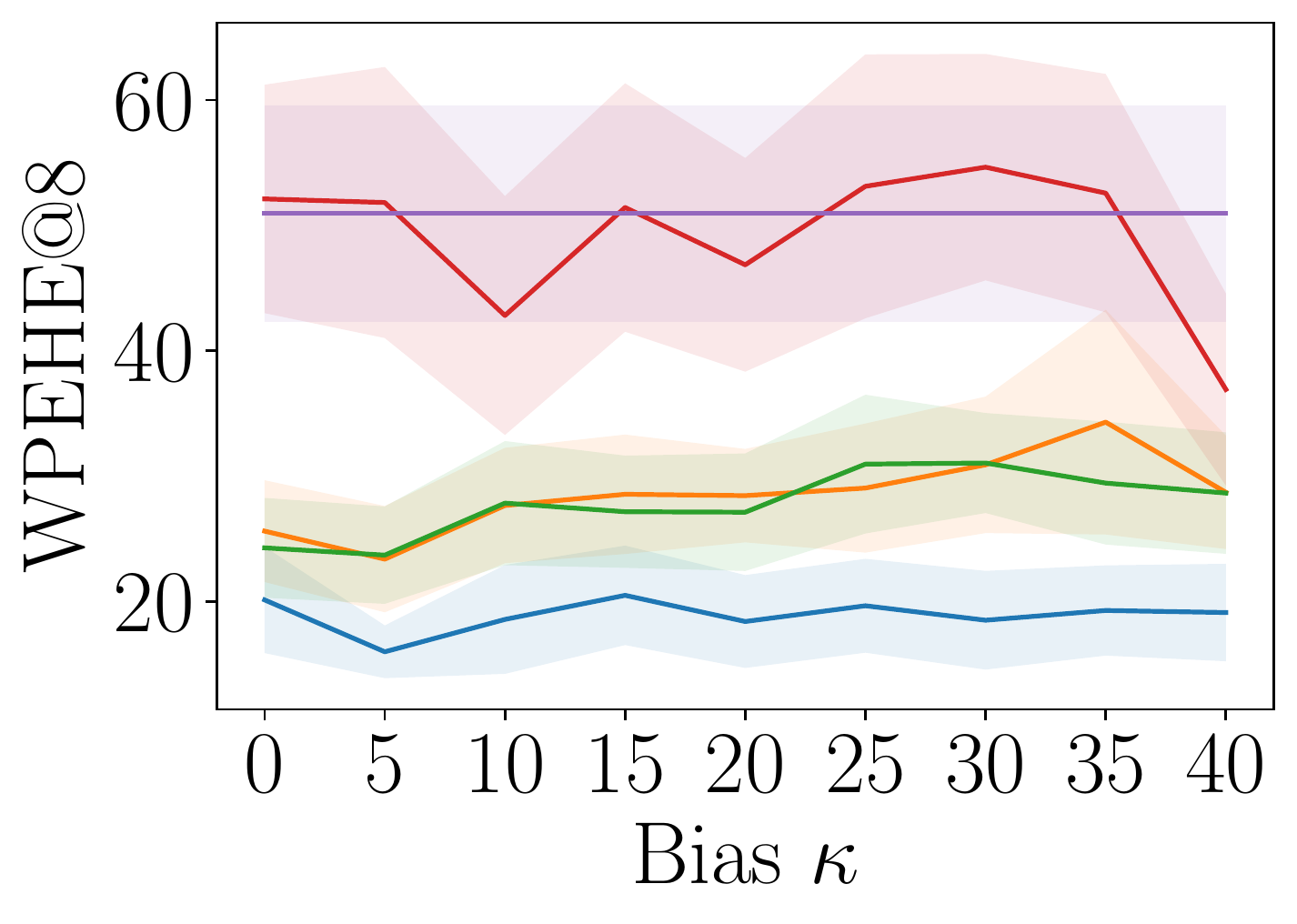}
    \end{subfigure}
    \begin{subfigure}[b]{.24\textwidth}
    \includegraphics[width=\textwidth]{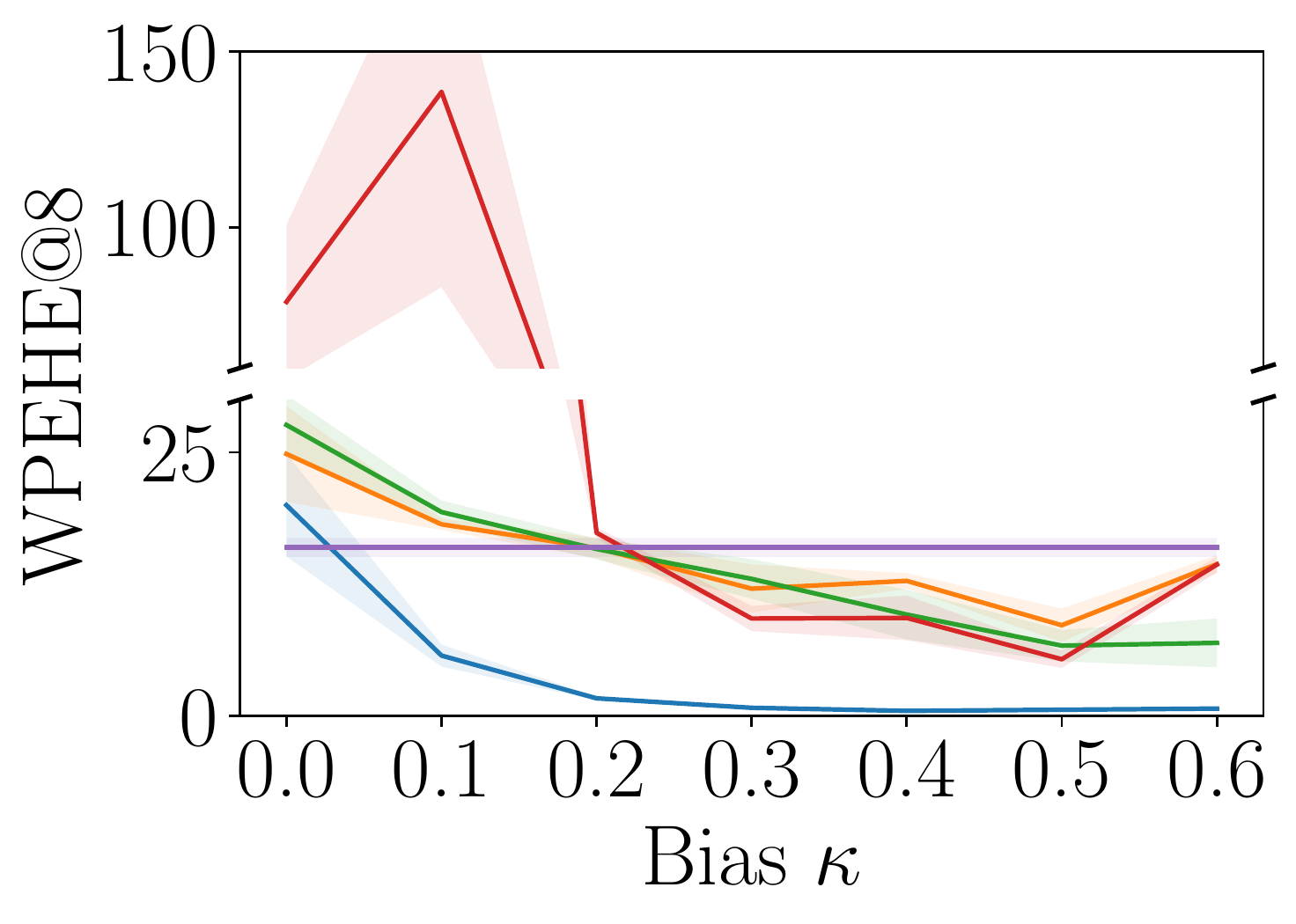}
    \end{subfigure}
    \begin{subfigure}[b]{.24\textwidth}
    \includegraphics[width=\textwidth]{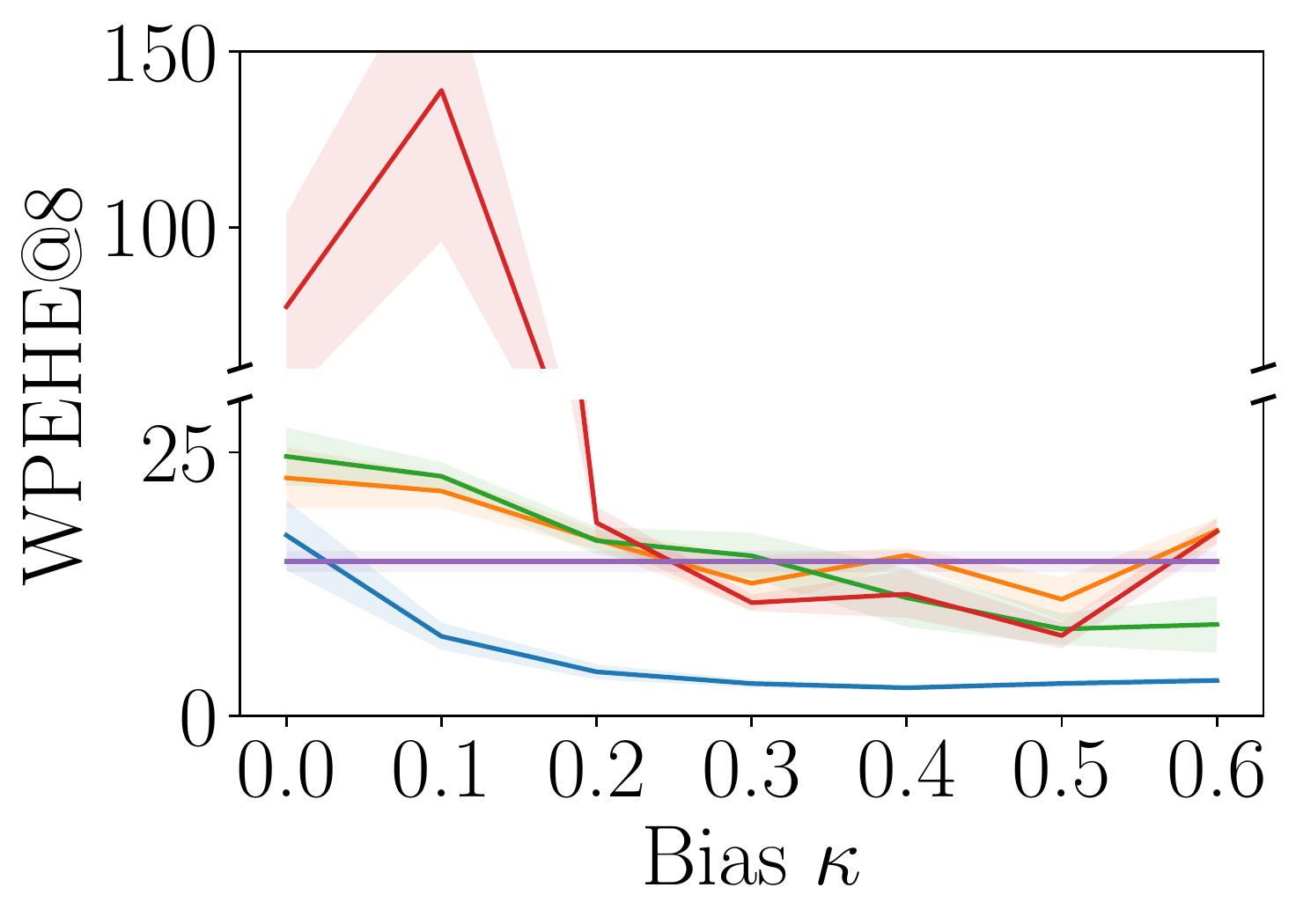}
    \end{subfigure}

    \begin{subfigure}[b]{.24\textwidth}
    \includegraphics[width=\textwidth]{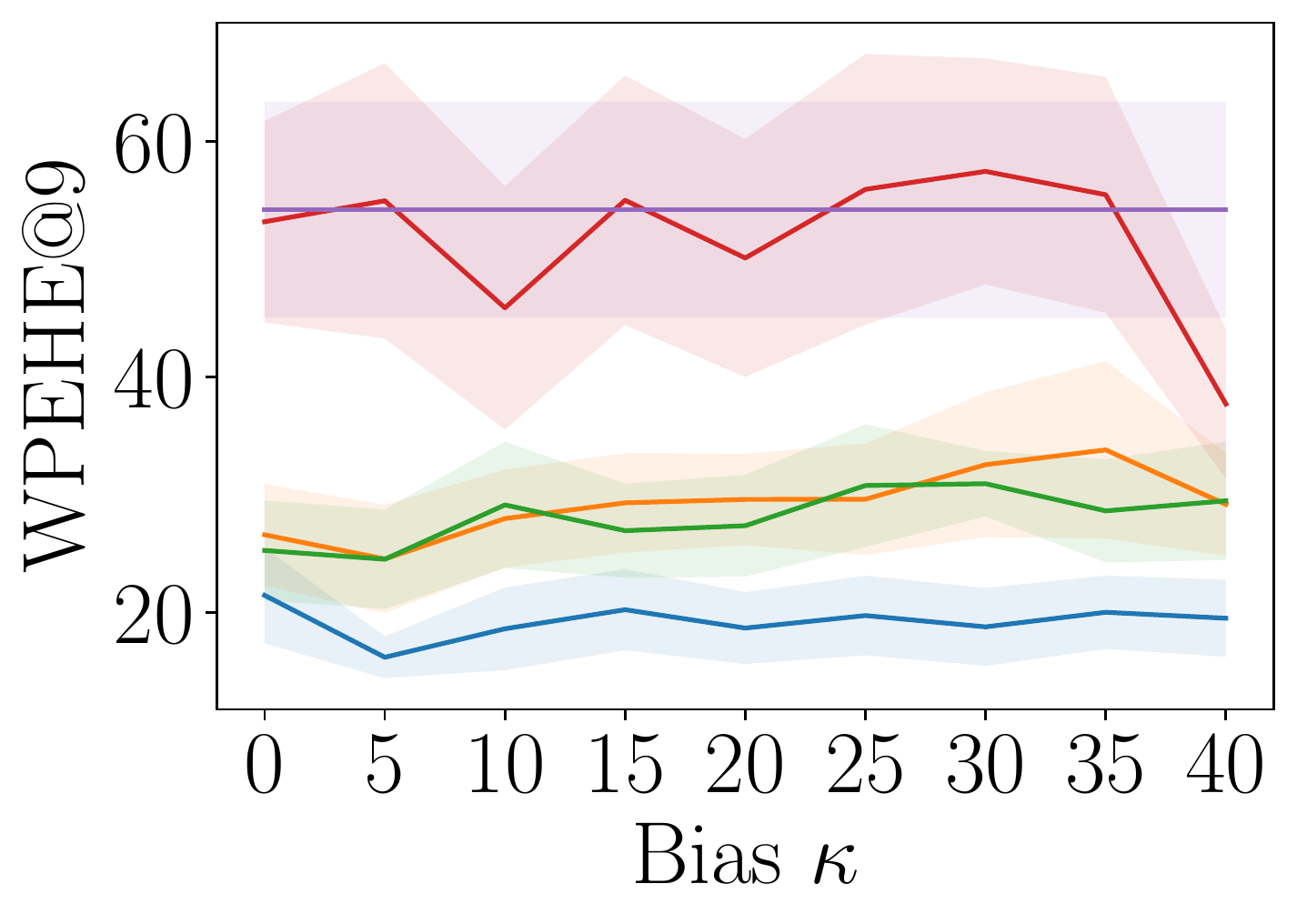}
    \end{subfigure}
    \begin{subfigure}[b]{.24\textwidth}
    \includegraphics[width=\textwidth]{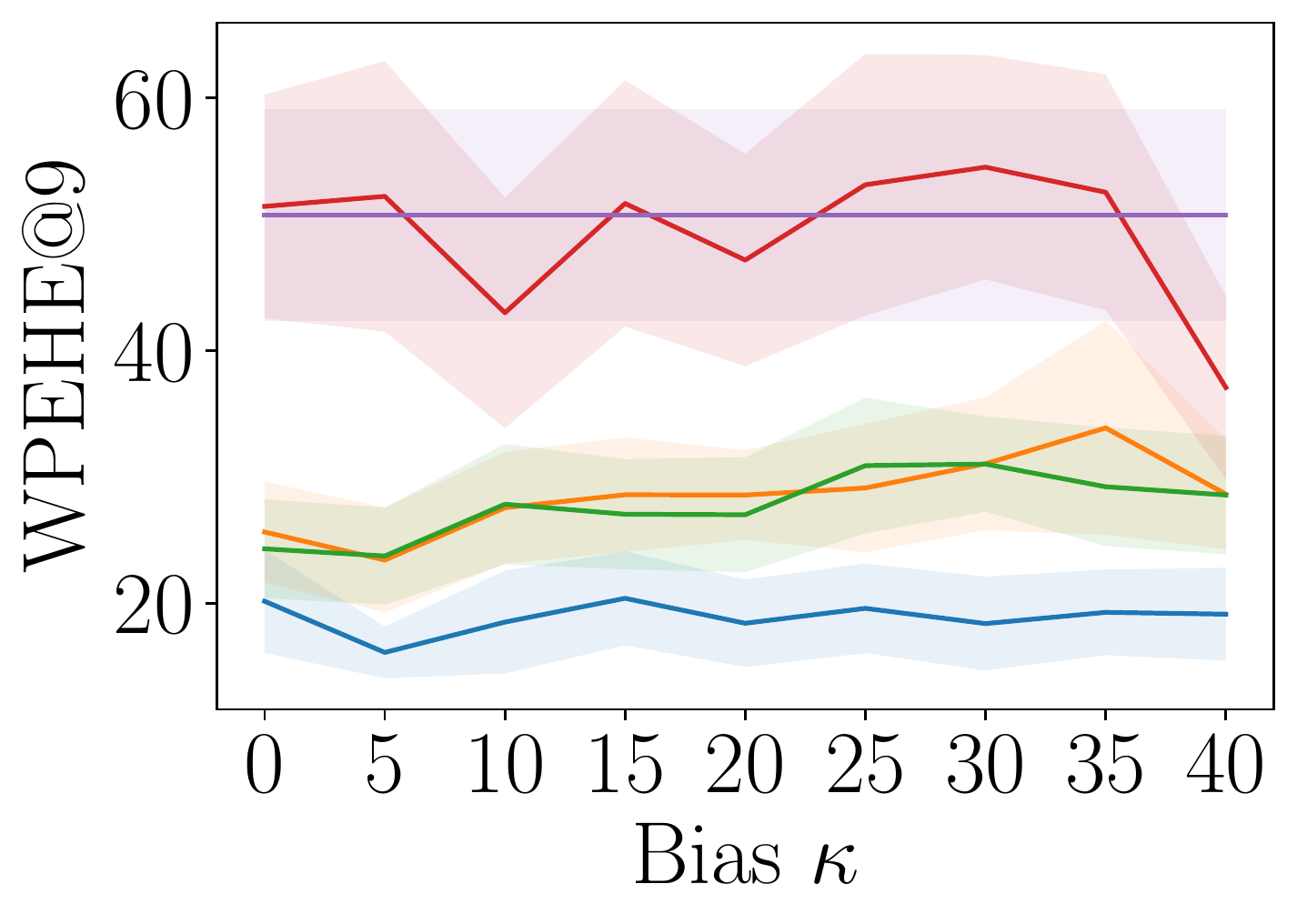}
    \end{subfigure}
    \begin{subfigure}[b]{.24\textwidth}
    \includegraphics[width=\textwidth]{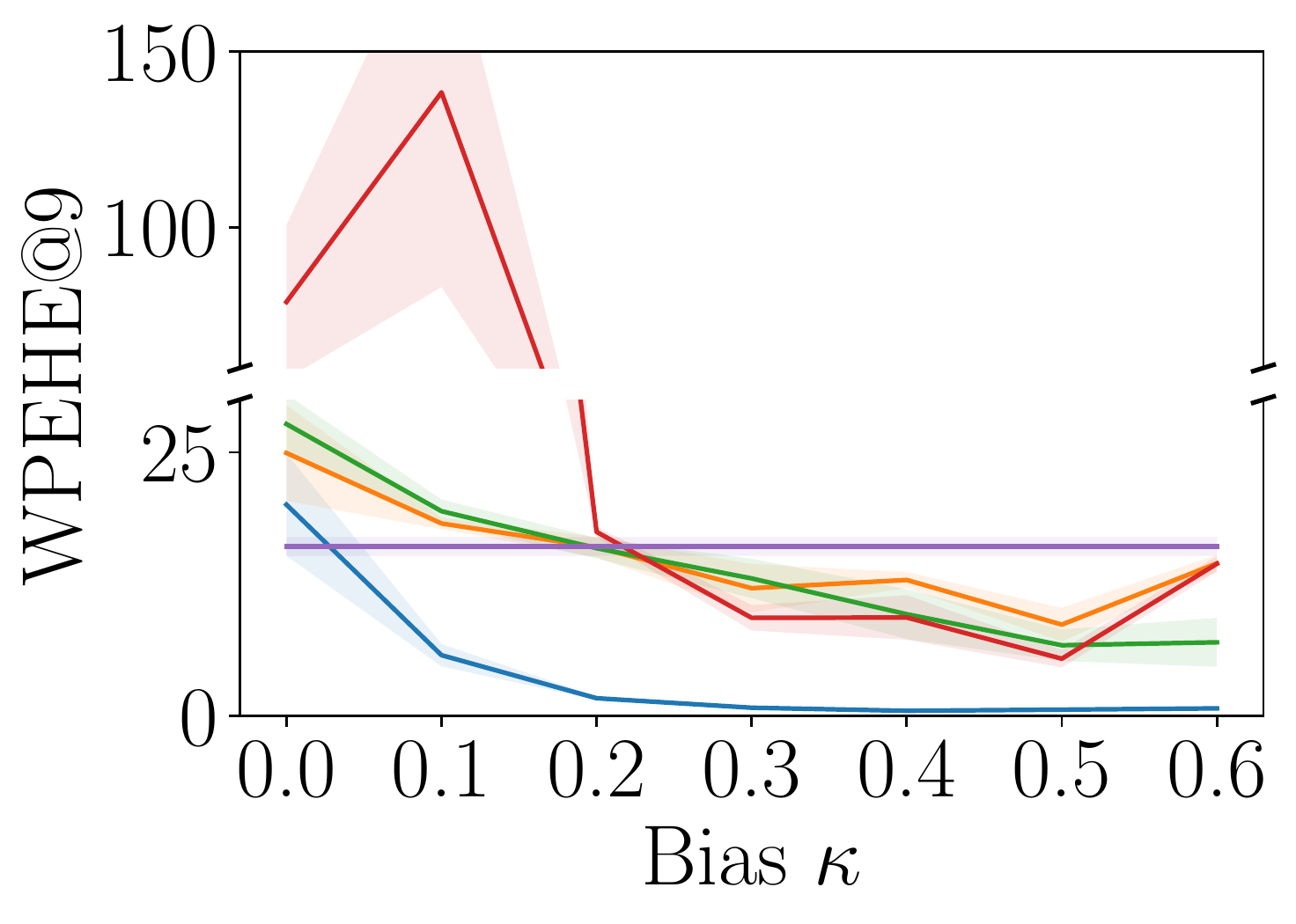}
    \end{subfigure}
    \begin{subfigure}[b]{.24\textwidth}
    \includegraphics[width=\textwidth]{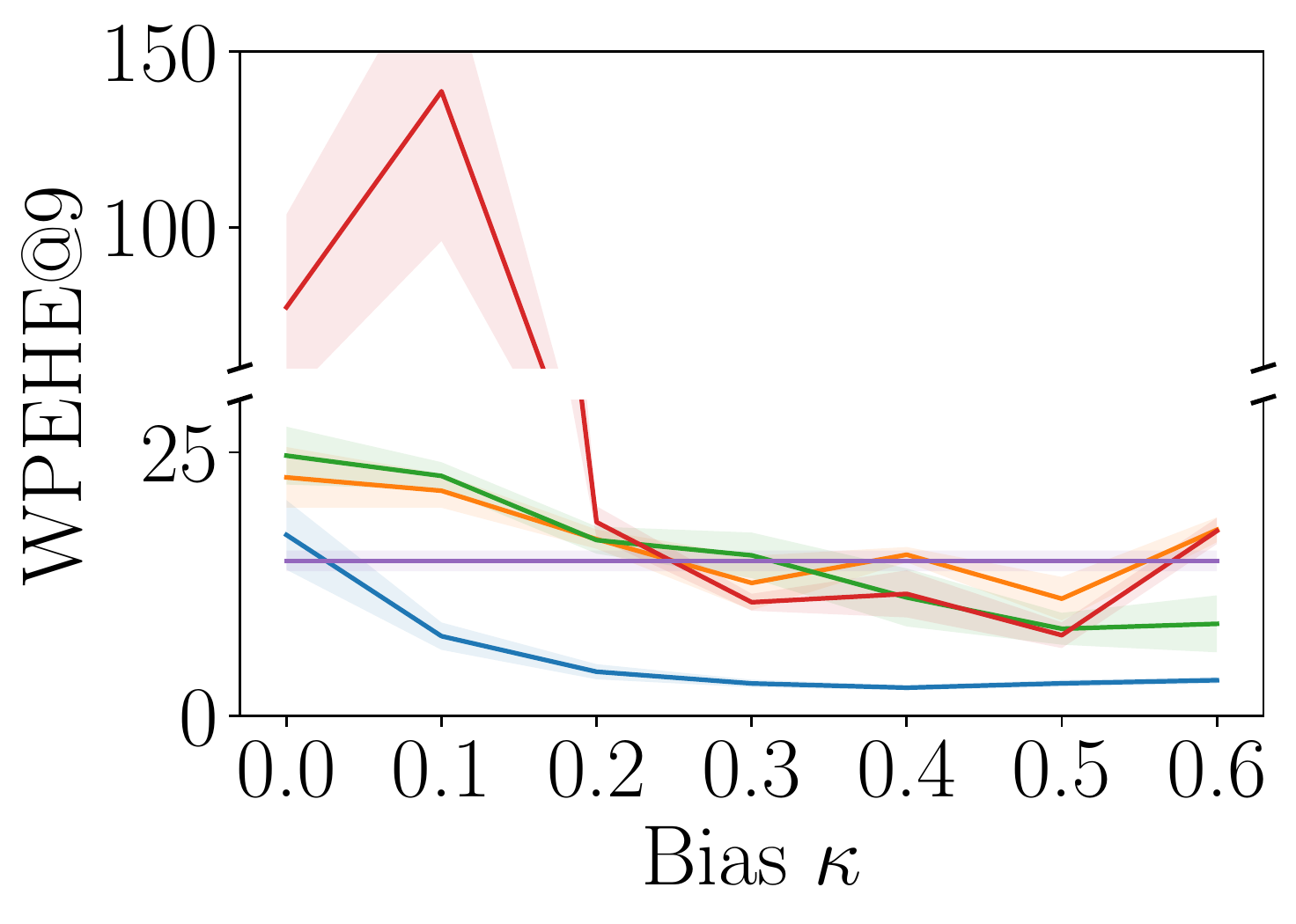}
    \end{subfigure}

    \begin{subfigure}[b]{.24\textwidth}
    \includegraphics[width=\textwidth]{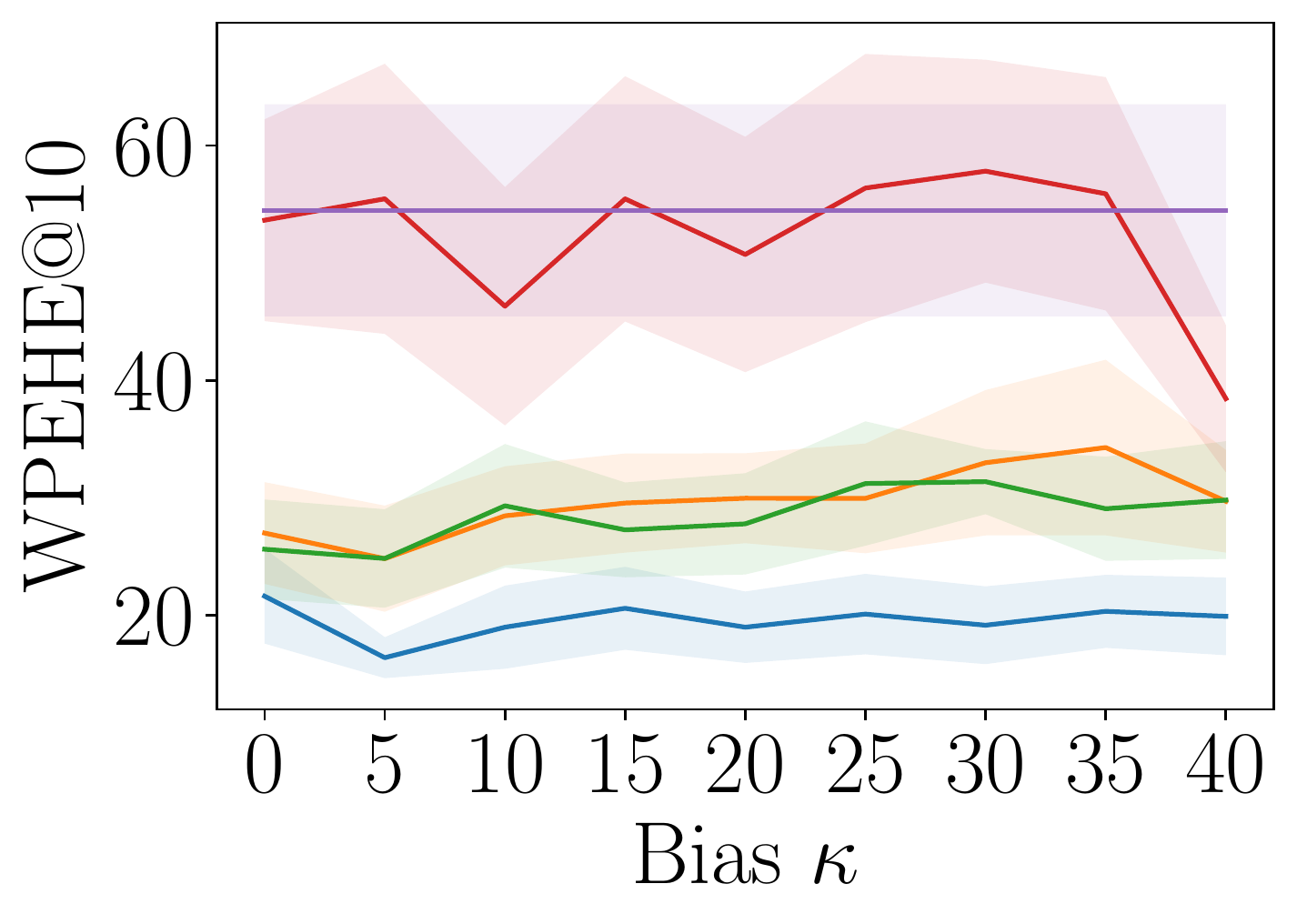}
    \end{subfigure}
    \begin{subfigure}[b]{.24\textwidth}
    \includegraphics[width=\textwidth]{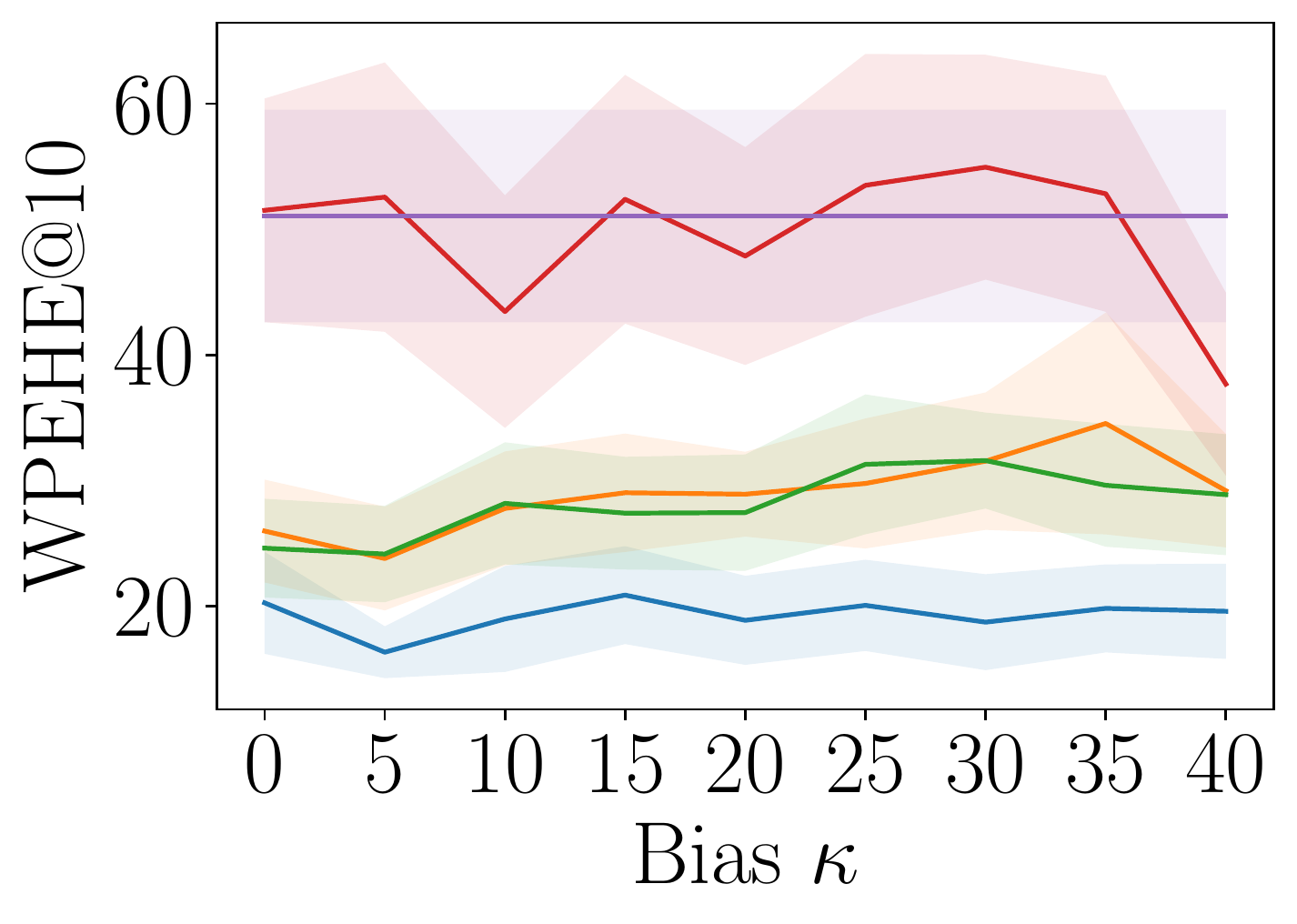}
    \end{subfigure}
    \begin{subfigure}[b]{.24\textwidth}
    \includegraphics[width=\textwidth]{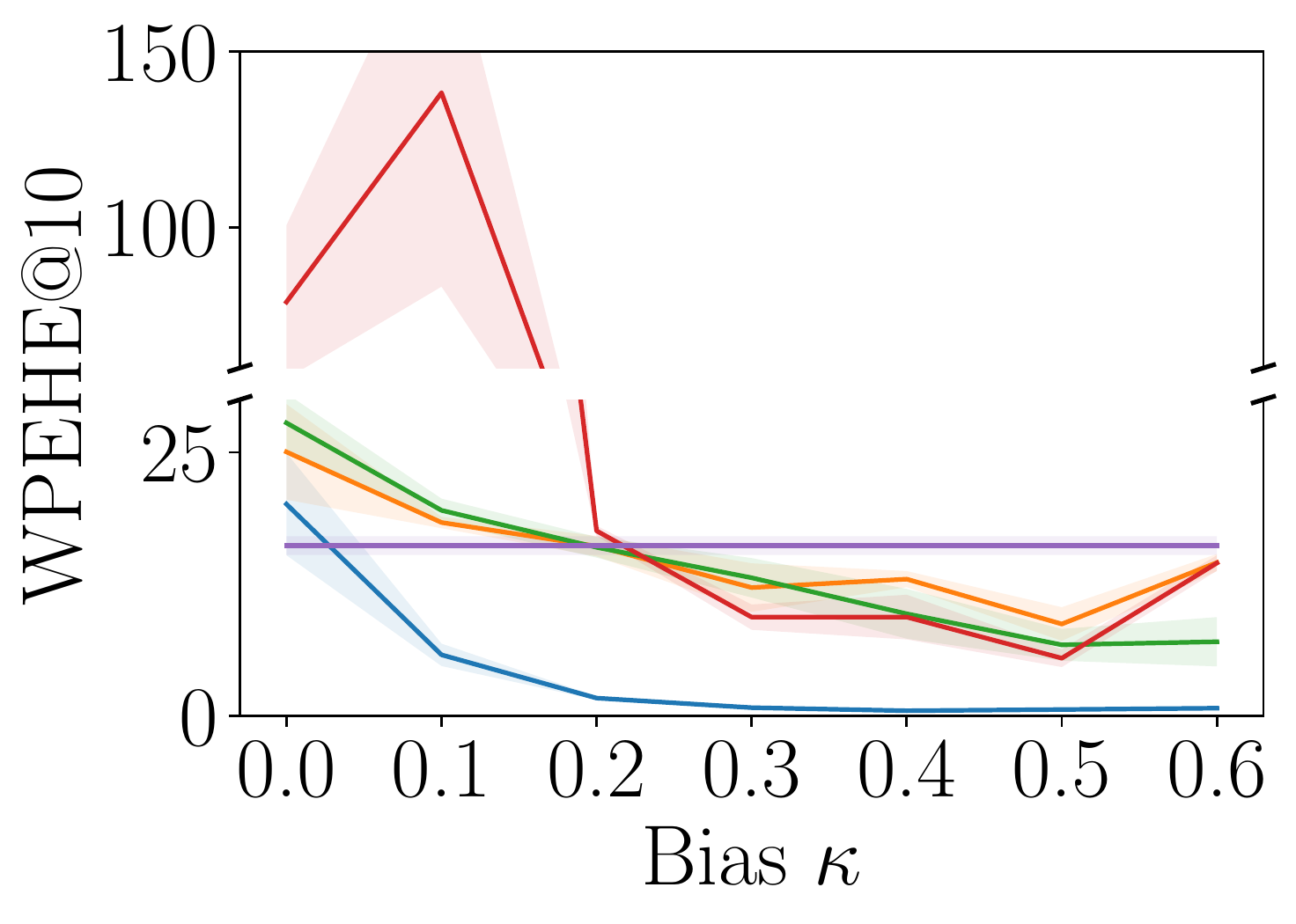}
    \end{subfigure}
    \begin{subfigure}[b]{.24\textwidth}
    \includegraphics[width=\textwidth]{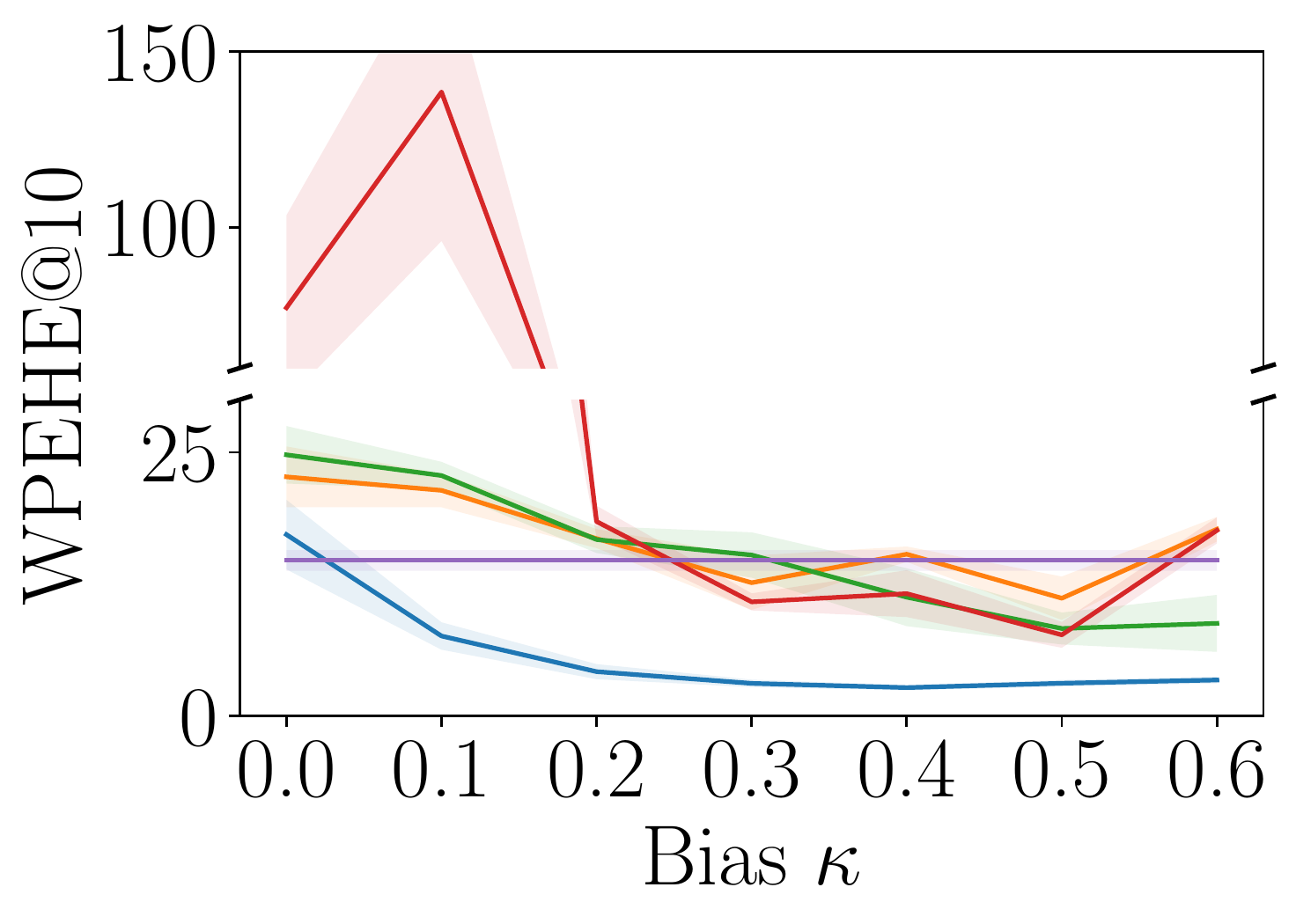}
    \end{subfigure}
    \includegraphics[width=0.8\textwidth]{figures/legend.pdf}
    \caption{WPEHE@$K$ over increasing bias strength $\kappa$ and varying $K$.}
\end{figure}
\section{Quasi-oracle rates for generalized R-Learner} \label{sec:rates}

The goal of this section is to establish error bounds for learning conditional average treatment effects (CATEs) when treatments are \textit{continuous}. To do so, we will assume that the response function, $\E[Y|\mx, \mt]$ can be written as follows,
\begin{align}
    \E\left[Y\mid\mx, \mt\right] = \boldsymbol{\alpha}\left(\matr X\right)^\top \Theta^* \boldsymbol{\beta}\left(\matr T\right),
\end{align}
where $\boldsymbol{\alpha}(\matr X) \in \mathbb{R}^{d_{\mx}}, \boldsymbol{\beta}(\matr T) \in \mathbb{R}^{d_{\mt}}$ are fixed, known basis functions\footnote{In deep learning jargon, each dimension of the basis functions, $\alpha_i, \beta_j$, is simply called a \textit{feature}.} (where $d_{\mx}, d_{\mt} < \infty$) and $\Theta^* \in \mathbb{R}^{d_{\mx} \times d_{\mt}}$ is unknown. We will show that we can learn $\Theta^*$ using the generalized Robinson decomposition in eq.~(\ref{eq:generalized_robinson_loss}) (i.e., the minimization is now over $\Theta$) with the same error rate as if we had known the true \emph{oracle} nuisance functions $m^*$ and $e^p$, provided our estimates of $m^*$ and $e^P$ converge to the ground truths at $O(n^{-1/4})$ rate. 

The reason we consider the above fixed basis setting instead of the more generic setup in the paper is because there are many things that make the analysis of a more general setup difficult:
\begin{itemize}[ wide = 5pt, leftmargin = *]
    \item There is a non-trivial dependence between estimators $m(\cdot), e(\cdot), g(\cdot), h(\cdot)$ created by fitting using the entire dataset (as opposed to using cross-fitting).
    \item Representation learning of the features typically involves non-convex loss functions; the convergence analysis of such is largely still an untackled question. 
    \item In the infinite-basis setting the problem becomes ill-posed (our current work provides insight into fixing this, in particular in Lemma \ref{lemma: overlap-extension}).
\end{itemize}
Addressing these issues is an interesting area of future work. Meanwhile, in this work, we focus on the scenario where the features (i.e. basis functions) are fixed. We first sketch our result without technical jargon as follows.

\begin{theorem*}[Sketch]
Write $m(\mathbf{x}) := \E\left[Y\mid\mx=\mathbf{x}\right]$ and $e^P(\mathbf{x}):=\E\left[\boldsymbol{\beta}(\mt) \mid \mx=\mathbf{x}\right]$. When the ground truths $m$ and $e^P$ are unavailable, we can still estimate $\E\left[Y \mid \mx, \mt\right]$ almost with rate $O\left(n^{-1/2}\right)$ using only estimates of $m$ and $e^P$, provided the estimates themselves converge at rate $O\left(n^{-1/4}\right)$.
\end{theorem*}

\subsection{Preliminaries} 

To specify the above formally, we follow e.g. \cite{rkhs_notes} to construct an RKHS for the hypothesis space of the response function $f$ as follows. Let $\mathcal{X}$ and $\mathcal{T}$ be compact metric spaces, endowed with finite Borel measures $\mathcal{P_X}$ and $\mathcal{P_T}$. Let $\{\alpha_i\}_{i=1}^{d_{\alpha}} \subset \mathcal{L}_2(\mathcal{X, P_X})$ and $\{\beta_i\}_{i=1}^{d_{\beta}} \subset \mathcal{L}_2(\mathcal{T, P_T})$ denote subsets of orthonormal functions in $\mathcal{L}_2(\mathcal{X, P_X})$ and $\mathcal{L}_2(\mathcal{T, P_T})$ which are feature maps for $\mx$ and $\mt$, respectively. Write $\boldsymbol{\alpha}, \boldsymbol{\beta} \in \R^{d_{\alpha}}, \R^{d_{\beta}}$ as the vectors of features on $\mx$ and $\mt$, with $\boldsymbol{\alpha}_{i}(\mathbf{x}) := \alpha_{i}(\mathbf{x})$ and $\boldsymbol{\beta}_j(\mathbf{t}) := \beta_j(\mathbf{t})$. Then define $k_{\mx}: \mathcal{X} \times \mathcal{X} \rightarrow \R $ as $k_{\mx}(\mathbf{x}_1, \mathbf{x}_2) = \langle \boldsymbol{\alpha}(\mathbf{x}_1), \boldsymbol{\alpha}(\mathbf{x}_2) \rangle_2$ where $\langle \cdot, \cdot \rangle_2$ is the standard Euclidean dot product in $\R^d$, and define similarly $k_{\mt}: \mathcal{T} \times \mathcal{T} \rightarrow \R $ as $k_{\mt}(\mathbf{t}_1, \mathbf{t}_2) = \langle \boldsymbol{\beta}(\mathbf{t}_1), \boldsymbol{\beta}(\mathbf{t}_2) \rangle_2$. Then clearly $k_{\mx}$ and $k_{\mt}$ are positive definite functions and by Moore-Aronsajn \cite[Section 4]{rkhs_notes} there exist unique RKHSes $\mathcal{H_X}, \mathcal{H_T}$ with kernels $k_{\mx}$ and $k_{\mt}$. 

For the readers familiar with \cite{r-learner}, we can connect the setup to that of \cite{r-learner} as follows: following e.g. \cite{mendelson}, an element $g$ in $\mathcal{H_X}$ can be represented by $g(\mathbf{x}) = \langle \theta, \boldsymbol{\alpha}(\mathbf{x})\rangle_2 = \langle g, \boldsymbol{\alpha}(\mathbf{x})\rangle_{\mathcal{H_X}}$.  Following \cite{rkhs_notes}, we can define an integral operator based on the kernel $k_{\mx}$:
\begin{align}
    S_{k_{\mx}}: \mathcal{L}_2(\mathcal{X}; \mathcal{P_X}) &\rightarrow \mathcal{C(X)} \hspace{0.3cm} \text{where $\mathcal{C(X)}$ are the continuous functions on $\mathcal{X}$.}\\
    (S_{k_{\mx}} f)(\mx) &= \int k_{\mx}(\mathbf{x}_1, \mathbf{x}_2)f(\mathbf{x}_2)d\mathcal{P_X}(\mathbf{x}_2), \hspace{0.5cm} f \in \mathcal{L}_2(\mathcal{X; P_X})\\
    T_{k_{\mx}} &= I_{k_{\mx}} \circ S_{k_{\mx}}\\
    \text{with the inclusion } I_{k_{\mx}}: \mathcal{C(X)}&\hookrightarrow \mathcal{L}_2(\mathcal{X; P_X})
\end{align}
Clearly the eigenfunctions of $T_{k_{\mx}}$ are the orthonormal functions $\{\alpha_i\}_{i=1}^{d_{\alpha}}$ and the non-zero eigenvalues are $\{\sigma_i=1\}_{i=1}^{d_{\alpha}}$. 

$\mathcal{H_T}$ can be dealt with similarly to $\mathcal{H_X}$. Since $\mathcal{H_{X\times T}}$ is isometrically isomorphic to $\mathcal{H_X} \times \mathcal{H_T}$ 
, we can identify the basis functions on $\mathcal{H_{X \times T}}$ as $\{\alpha_i \beta_j\}_{i,j=1}^{d_{\alpha}, d_{\beta}}$, the eigenvalues as $\{\sigma_{ij}=1\}_{i,j=1}^{d_{\alpha}, d_{\beta}}$, and the inner product $\langle \cdot, \cdot \rangle_{\mathcal{H_{X\times T}}} = \langle \cdot, \cdot \rangle_{\mathcal{H_{X}}}\langle \cdot, \cdot \rangle_{\mathcal{H_{T}}}$ . By construction, the RKHS norm and the $L_2$ norm of $\mathcal{H_{X \times T}}$ are both equal to the matrix $2-$norm of the function representer, that is, for $f \in \mathcal{H}_{\mathcal{X \times T}}, f(\mathbf{x}, \mathbf{t}) = \langle \thet, \boldsymbol{\alpha}(\mathbf{x}) \otimes  \boldsymbol{\beta}(\mathbf{t})\rangle_{2}$,
\begin{equation}
    \|f\|_{\mathcal{H}_{\mathcal{X \times T}}} = \|f\|_{L_2}=\|\thet\|_2 \label{eq:l2-rkhs-equiv}
\end{equation}

Trivially, for all $0 < p < 1$, the eigenvalues $\sigma_{ij}$ satisfy $G = \sup_{i, j \geq 1}(i+d_{\mx}(j-1))^{1/p} \sigma_{ij}$ for some constant $ G < \infty$, which was posed as an assumption in \cite{r-learner}.

\begin{remark}
We did not need to require $\mathcal{X}$ and $\mathcal{T}$ as compact metric spaces. Requiring them to be measurable spaces on which we can define $\mathcal{L}_2$ functions should be enough. But compact metric spaces also include most spaces of practical concern, including graph spaces, so we choose it since it satisfies the conditions of Mercer's theorem.
\end{remark}

\subsection{Problem set-up}
We assume that the true response function lie in $\mathcal{H_{X \times T}}$:
\begin{customassp}{\ref{assump:true_fun_approx}}
The true response function $f^*(\mathbf{x}, \mathbf{t}) = \E[Y\mid\mx=\mathbf{x}, \mt=\mathbf{t}]$ can be written as $f^*(\mathbf{x}, \mathbf{t})= \boldsymbol{\alpha}^{\top}(\mathbf{x})\thet^* \boldsymbol{\beta}(\mathbf{t})$ for some matrix of coefficients $\thet^*$.
\end{customassp}

First we write down the population and empirical loss functions we consider. In order to assert that every element of $\mathcal{H_{X \times T}}$ can be uniquely represented by some $\thet$, we use $f_{\thet}$ to denote $f_{\thet} := \boldsymbol{\alpha}(\mathbf{x})^T \thet \boldsymbol{\beta}(\mathbf{t}) \in \mathcal{H_{X \times T}}$.

The expected loss of $f_{\thet}$ is defined by:
\begin{align}
    L(f_{\thet}) = L(\thet) = \E\left[\left\{(Y - m^*(\mx)) - \boldsymbol{\alpha}(\mx)^T \thet (\boldsymbol{\beta}(\mt) - e^P(\mx))\right\}^2\right]
\end{align}

The oracle (empirical) loss is defined by:
\begin{align}
    \tilde{L}_n(f_{\thet}) = \tilde{ L}_n(\thet) = \sum_{l=1}^n\left[\left\{(Y - m^*(\mx_l)) - \boldsymbol{\alpha}(\mx_l)^T \thet (\boldsymbol{\beta}(\mt_l) - e^P(\mx_l))\right\}^2\right]
\end{align}

The feasible (empirical) loss is defined by:
\begin{align}
    \hat{L}_n(f_{\thet}) = \hat{L}_n(\thet) = \sum_{l=1}^n\left[\left\{(Y - \hat{m}(\mx_l)) - \boldsymbol{\alpha}(\mx_l)^T \thet (\boldsymbol{\beta}(\mt_l) - \hat{e}^P(\mx_l))\right\}^2\right]
\end{align}
Note that we use  $L(f_{\thet})$ and $L(\thet)$ interchangeably due to the bijection between $\thet \in \R^{d_{\mx} \times d_{\mt}}$ and $\mathcal{H}_{\mathcal{X \times T}}$.

The corresponding regret functions are defined by 
\begin{align}
    R(\thet) &= L(\thet) - L(\thet^*)\\
    \tilde{R}_n(\thet)&= \tilde{L}_n(\thet) - \tilde{L}_n(\thet^*)\\
    \hat{R}_n(\thet)&= \hat{L}_n(\thet) - \hat{L}_n(\thet^*)
\end{align}

We now formally state the assumptions we need to derive the result in Theorem \ref{theorem:main}.

\begin{customassp}{\ref{assump:overlap_mainbody}}[Overlap] 
The marginal distribution of features $\mathcal{P}_{\boldsymbol{\alpha}(\mathcal{X})\boldsymbol{\beta}(\mathcal{T})}$ is positive, i.e. $\operatorname{supp}[\mathcal{P}_{\boldsymbol{\alpha}(\mathcal{X})\boldsymbol{\beta}(\mathcal{T})}]= \boldsymbol{\alpha}(\mathcal{X})\boldsymbol{\beta}(\mathcal{T})$.
\end{customassp}

\begin{assumption}[Boundedness] \label{assump:boundedness} Without loss of generality, we assume that for all $\mx \in \mathcal{X}, \mt \in \mathcal{T}$, $\sup_{i}\|\alpha_i(\mx)\|_{\infty}, \sup_{j}\|\beta_j(\mt)\|_{\infty} \leq A < \infty$. We also assume that the outcome $Y$ are almost surely bounded, i.e. $\mathbb{P}(|Y|<B<\infty) = 1$.
\end{assumption}

For clarity, we list all notations we use here.
\paragraph{Notation.}
\begin{itemize}
    \item $\mathcal{H}$: A Product Reproducing Kernel Hilbert Space with finite number of basis functions, with $\boldsymbol{\alpha}$ the features of $X$ and $\boldsymbol{\beta}$ the features of $T$.
    \item $\Theta$: The matrix of coefficients for a given function in $\mathcal{H}$.
    \item $f_{\Theta}$: $f_{\Theta}(X,T) := \boldsymbol{\alpha}(X)^{\top}\Theta \boldsymbol{\beta}(T)$.
    \item $\mathcal{H}_c$: The subset of $\mathcal{H}$ which is the ball of radius $c$.
    \item $\Theta_c$: $f_{\Theta_c}$ is a minimiser of the loss in $\mathcal{H}_c$.
    \item $R(f_{\Theta})$: $L(f_{\Theta}) - L(f^*)$.
    \item $R(f_{\Theta}; c)$: $L(f_{\Theta}) - L(f_{\Theta_c})$
\end{itemize}

\paragraph{Convention.} Throughout, we will use capital letters $A, B, C,...$, possibly with subscripts and superscripts, e.g. $A_1, B^{(2)}$, etc. to denote constants. We may overload notation and use the same letter to denote different constants. 

\subsection{Proof strategy}
Here we lay forward the detailed proof for the quasi-oracle convergence rate for a featurized continuous heterogeneous treatment effect estimation algorithm with Robinson decomposition. Our proof extends the structure of \citet{r-learner}. To make the proof self-contained while simultaneously highlighting the differences with \citet{r-learner}, we present a complete version of the proof, where we will pause to describe any difference and its significance where it appears.

The high-level idea of showing `quasi-oracle' error rate is as follows. First, we show that both the feasible loss and the oracle loss satisfy the same (quasi-)isomorphism with the true loss, where the tightness of the quasi-isomorphism increases as sample size increases. The quasi-isomorphism with the true loss then leads us to bound the feasible and oracle losses by the same quantity, which decreases to 0 as sample size grows indefinitely. To show the (quasi-)isomorphism for the oracle learner can be done by leveraging on the standard least-squares regression ideas \cite{mendelson}; to achieve the same for the feasible learner relies on the fact that the feasible loss differs from the oracle loss by only a small amount relative to the true loss, which constitutes the bulk of the proof.

We start with stating the formal lemma which connects quasi-isomorphism with loss bounds.

\subsection{From quasi-isomorphism to regret bound}

\begin{definition}[loss function]
A function is a \textbf{loss function} if it maps from a hypothesis class $\mathcal{H}$, to the real numbers $\mathbb{R}$. 
\end{definition}
\begin{lemma}\label{lemma:risk_bound}
 Let $\check{L}(\fest \in \hc)$ be a loss function, and $\check{R}(\fest; c) = \check{L}(\fest) - \check{L}(\fc)$ be the associated c-regret. Suppose $\rho(r)$ is a positive, continuous, increasing function. If, $\forall$ $1\leq c \leq C$ and some $k > 1$, the following inequality holds for all $\fest \in \hc$:
\begin{align}
    \frac{1}{k}\check{R}(\fest;c) - \rho(c) \leq R(\fest; c) \leq k\check{R}(\fest; c) + \rho(c) \label{eq:quasi-iso-general}
\end{align}
Then, writing $\kappa_1 = 2k + \frac{1}{k}$ and $\kappa_2 = 2k^2 +3$, any solution to the regularized minimization problem with $\Lambda(c) \geq \rho(c)$,
\begin{equation}
    f_{\check{\thet}} \in \argmin_{\fest \in \hC}\{\check{L}(\fest)+\kappa_1\Lambda(\fest)_{\mathcal{H}}\}
\end{equation}
also satisfied the following risk bound:
\begin{equation}
    L(f_{\check{\thet}}) \leq \inf_{\fest \in \hC} \{L(\fest) + \kappa_2\Lambda(\fest)_{\mathcal{H}}\label{eq:loss_bound}
\end{equation}
\end{lemma}
\begin{proof}
Notice that $\{\hc; c \geq 1\}$ is an ordered set. Thus the same argument as \cite{r-learner} applies.
\end{proof}
Lemma \ref{lemma:risk_bound} tells us that if we have a quasi-isomorphism of the regrets in the form of \ref{eq:quasi-iso-general}, we immediately can bound the expected risk of the (regularized) minimizer of the corresponding loss, $\check{L}$ as in \ref{eq:loss_bound}.

\subsection{A concrete instance of $\rho(c)$ satisfying \ref{eq:quasi-iso-general}}

By setting $\check{R}$ to $\tilde{R}_n$, \ref{lemma:risk_bound} gives us a way to bound the oracle regret, but we still need a concrete formulation of $\rho(c)$ to derive the oracle convergence rate. To this end, we may use the result of \citet{mendelson}, but first we must show that their results can be applied to our setting.

\citet{mendelson} consider the optimization over a space of RKHS functions with the least-squares loss. Our oracle case can be thought of in the same way as follows: since $m^*$ is an oracle quantity, $Y- m^*(\mx)$ can be thought of as the labels, the space $\overline{\mathcal{H}} = \{f_{\thet}: \mathcal{X} \times \mathcal{T} \rightarrow R; \text{for some $\thet \in \R^{d_{\mx} \times d_{\mt}}$}, f_{\thet}(\mathbf{x}, \mathbf{t}) = \boldsymbol{\alpha}(\mathbf{x})^{\top} \thet (\boldsymbol{\beta}(\mathbf{t}) - e^P(\mathbf{x}))\}$ is an RKHS with features $\boldsymbol{\alpha}(\mx) \otimes (\boldsymbol{\beta}(\mt) - e^P(\mx))$. Thus, our setting can be thought of as a least-squares optimization over the RKHS $\overline{\mathcal{H}}$ and the results from \cite{mendelson} applies. To use the results of \cite{mendelson}, we still need the following technical result which decomposes $\bar{\mathcal{H}}$ into an \textit{ordered, parameterized hierarchy}.

\begin{definition}[Ordered, parameterized hierarchy]
As defined in \cite{mendelson}, let $\mathcal{F}$ be a class of functions and suppose that there is a collection of subsets $\{\mathcal{F}_r; r \geq 1\}$ with the following properties:
\begin{enumerate}
    \item $\{\mathcal{F}_r: r\geq 1\}$ is monotone (i.e. whenever $r \leq s, \mathcal{F}_r\subseteq \mathcal{F}_s$);
    \item for every $r \geq 1$, there exists a unique element $f_r^* \in \mathcal{F}_r$ such that $L(f_r^*) = \inf_{f\in \mathcal{F}_r} L(f)$;
    \item the map $r \rightarrow L(f_r^*)$ is continuous;
    \item for every $r_0 \geq 1$, $\bigcap_{r \leq r_0} \mathcal{F}_r = \mathcal{F}_{r_0}$;
    \item $\bigcup_{r\leq 1} \mathcal{F}_r = \mathcal{F}$.
\end{enumerate}
Given a class of functions $\mathcal{F}$, we say that $\{\mathcal{F}_r; r \geq 1\}$ is an \textbf{ordered, parameterized hierarchy of $\mathcal{F}$} if the above conditions 1-5 are satisfied. 
\end{definition}

\begin{lemma}
Define \begin{align}
    \overline{\hc} &:= \{\fest: \spacex \times \spacet \rightarrow \R: \exists \thet, \|\thet\|_2 \leq c, \\ &\text{s.t.} \quad \fest(\mx, \mt) = \boldsymbol{\alpha}(\mx)^{\top}\thet(\boldsymbol{\beta}(\mt)-\propfea)\},
\end{align} then $\left\{\overline{\hc}\right\}_{1\leq c \leq C}$ is an ordered parameterized hierarchy.
\end{lemma}
\begin{proof}
The first, fourth and fifth properties follow immediately. $\overline{\hc}$ is clearly convex. It is compact because every sequence $\{f_{\thet_i}\}_{i} \subset \overline{\hc}$ is induced by $\{\thet_i\}_{i} \subset \R^n, \|\thet_i\|_2\leq c$, and by Bolzano-Weierstrass theorem in $\R^n$, every bounded sequence has a convergent subsequence $\{\thet_k\}_{k} \subset \{\thet_i\}_i$ (w.r.t. the Euclidean norm). Thus pick the $N$ such that for all $k \geq N$ where $\|\thet_k - \thet_{N}\|_2 \leq \epsilon$, and then 
\begin{align}
    \norm{f_{\thet_k} - f_{\thet_{N}}}_{L_2{(P(\mathcal{X,T}))}} &= \norm{f_{\thet_k - \thet_{N}}}_{L_2{(P(\mathcal{X,T}))}} = \E\left[\langle\thet_k - \thet_{N}, \boldsymbol{\alpha}(\mx)\otimes (\boldsymbol{\beta}(\mt) - e^p(\mx))\rangle^2\right]^{1/2}\\
    & \leq \E\left[ \norm{\thet_k - \thet_{N}}_2 \norm{\boldsymbol{\alpha}(\mx)\otimes (\boldsymbol{\beta}(\mt) - e^p(\mx))}_2\right]^{1/2}\\
    & \leq \epsilon \E\left[\norm{\boldsymbol{\alpha}(\mx)\otimes (\boldsymbol{\beta}(\mt) - e^p(\mx))}_2\right]^{1/2} \leq \epsilon B, 
\end{align} where $\norm{\boldsymbol{\alpha}(\mx)\otimes (\boldsymbol{\beta}(\mt) - e^p(\mx))}_2 \leq B$ by Assumption~\ref{assump:boundedness} for some constant B. The second property now follows from the fact that $\overline{\hc}$ is convex and compact. The third property follows by the same argument as \cite{mendelson}.
\end{proof}

\citet{mendelson} thus provides a formulation of $\rho$ which, with some constant $U(\epsilon)$, for large enough $n$ and probability at least $1-\epsilon$, satisfies \ref{eq:quasi-iso} for the oracle loss function $\tilde{R}_n$ with $k=2$:

\begin{align}
\rho_{n}(c) = U\left(\epsilon\right) \left\{1+\log \left(n\right)+\log \left ( \log \left(c+e\right)\right)\right\}\left(\frac{(c+1)^{p} \log (n)}{\sqrt{n}}\right)^{2 /(1+p)} \label{eq:rho}
\end{align}

Thus, we may now realize the convergence rate for the oracle learner as follows.

\subsection{Oracle convergence rate.} \label{subsec:oracle_rate}

With \ref{eq:rho}, Lemma \ref{lemma:risk_bound} immediately implies that penalized regression over $\mathcal{H}_C$ with the oracle loss function $\tilde{L}_n(\cdot)$ and regularizer $\kappa_1 \rho_n(c)$ satisfies the bound below with high probability:
\begin{equation}
    R(\tilde{\thet}_n) = L(\tilde{\thet}_n) - L(\thet^*) \leq \inf_{\thet \in \mathcal{H}_C} \{L((\thet) + \kappa_2\rho_n(\|\thet\|_{\mathcal{H}})\} - L(\thet^*)
\end{equation}

Furthermore, Corollary 2.7 in \cite{mendelson} gives that for any $1 < c < C$, 
\begin{equation}
    \inf_{\thet \in \mathcal{H}_C} \{L(\thet) + \kappa_2 \rho_n(\|\thet\|_{\mathcal{H}})\} \leq L(\thet^*) + \{L(\thet_c^*) - L(\thet^*)\} + \kappa_2 \rho_n(c)
\end{equation}

Finally, note that for large enough $c$,
\begin{align}
    \left\{L\left(\thet_c^*\right) - L\left(\thet^*\right)\right\} &= 0,
\end{align}
so the error is dominated by $\rho_{n}(c)$, at
\begin{align}
    R\left(\tilde{\thet}_n\right) &= \mathcal{O}\left(\left(\log(n)\right)^{\frac{3+p}{1+p}}n^{-\frac{1}{1+p}}\right)=\tilde{\mathcal{O}}(n^{-\frac{1}{1+p}}),
\end{align}
where $\tilde{\mathcal{O}}$ notation ignores the logarithmic factors. 

\subsection{Bridging $\hat{R}_n$ and $\tilde{R}_n$}
Now that we have the oracle convergence rate, we show a bridging result which will let us conclude that \ref{eq:quasi-iso-general} holds for $\hat{R}_n$ as well, and thus the oracle rate also holds for $\hat{R}_n$.

To yield that bridging result, we first need to leverage the assumption of overlap to relate the $L_2$ difference between $f_{\thet}$ and $\fc$, i.e. $\E\left[\left(f_{\thet}(\mx, \mt)- \fc (\mx, \mt)\right)^2\right]$, with the $c-$regret $R(\thet; c)$. We first show that the $L_2$ difference is always upper bounded by the regret up to a constant.

\begin{lemma} \label{lemma: jensen&finitedim}
$\exists \epsilon > 0$ s.t. for all $f \in \hc$, $\E_{\alpha}[\langle f, \alpha \rangle^2] \geq \epsilon\|f\|_{L_2}$ where $\alpha$ is a $r.v.$ taking values in $\hc$ and the support of $\alpha$ is of Lebesgue-measure non-zero in $\hc$. 
\end{lemma}
\begin{proof}
Let $S = \{f \in \hc: \|f\|_{\hc}=1\}$, and define $g: S \rightarrow \R^+$ as $g(f) = \E_{\boldsymbol{\alpha}}[\langle f, \alpha \rangle^2]$. By Jensen's inequality, $\E_{\boldsymbol{\alpha}}[\langle f, \alpha \rangle^2] \geq 0$ since $\langle f, \cdot \rangle^2: \alpha \mapsto \langle f, \alpha \rangle^2$ is a convex function in $\alpha$. Moreover, whenever $supp[\mathcal{P}_{\boldsymbol{\alpha}}]$ is Lebesgue-measure non-zero in $\hc$, $\langle f, \cdot\rangle^2$ is non-linear on $supp[\mathcal{P}_{\boldsymbol{\alpha}}]$, so the inequality is strict: 
\begin{equation}
    \E_{\boldsymbol{\alpha}}[\langle f, \alpha \rangle^2] > 0 \label{eq:strict_jensen}.
\end{equation}
Now since $\hc$ is finite-dimensional, $S$ is compact. Since $g$ is continuous in $f$, and the continuous image of a compact set is compact, we have that $g(S)$ is compact, and therefore closed. 

Note, at this point, that $g(S)$ is the set of values achieved by $\E_{\boldsymbol{\alpha}}[\langle f, \alpha\rangle^2]$ at various values of $f$. By \eqref{eq:strict_jensen}, $g(S) \not\ni 0$. Since $g(S)$ is compact, its complement thus contains 0. Moreover, since $\R^+ \setminus g(S) \ni 0$, $\exists$ a ball around 0 of radius $\tilde \epsilon > 0$ s.t. $[0, \tilde \epsilon) \subset \R^+ \setminus g(S)$. Therefore, $g(S) \subset \R^+$ is lower bounded by $\tilde \epsilon > 0$. 

Therefore, \begin{align}
    \forall f \in \hc, \E_{\boldsymbol{\alpha}}\left[\left\langle f, \alpha \right\rangle^2\right] = \norm{f}_{\hc}^2 \E_{\boldsymbol{\alpha}}\left[\left\langle \frac{f}{\|f\|_{\hc}^2}, \alpha \right \rangle^2 \right] \geq \epsilon \|f\|_{\hc}^2 = \epsilon \|f\|^2_{L_2},
\end{align} for some $\epsilon > 0$. The last inequality is due to \ref{eq:l2-rkhs-equiv}.
\end{proof}

\begin{lemma}[Usage of the overlap condition in the multiple treatment setting]\label{lemma: overlap-extension}
Under Assumption \ref{assump:overlap_mainbody}, i.e. we have overlap on the features, that is $supp[\mathcal{P}_{\boldsymbol{\alpha}(\mathcal{X}) \times \boldsymbol{\beta}(\mathcal{T})}] = \boldsymbol{\alpha}(\mathcal{X}) \times \boldsymbol{\beta}(\mathcal{T})$, then $\exists A \in \mathbb{R}$ s.t.
\begin{equation}
    \mathbb{E}[\left(f_{\Theta}(X,T) - f_{\Theta_c}(X,T)\right)^2] < A R(\Theta;c)
\end{equation}
\end{lemma}
\begin{proof}
Within $\mathcal{H}_c$, we seek to upper bound excess $L_2$ risk of $f_{\Theta}$ by its c-regret $R(\Theta; c)$; $R(\Theta; c) = L(\Theta) - L(\Theta_c)$.

First we write down the expected loss functional again:
\begin{align}
    L(\thet) &= \E[(\{Y - m^*(\mx)\} - \{\fest(\rmX. \rmT) - \E[\fest(\rmX, \rmT)|\rmX]\})^2] \\
    &= \E[\var\{Y - m^*(\mx)|\mx, \mt\}] + \E[\{(f^*(\mx, \mt) - \fest(\mx, \mt)) - \E[f^*(\mx, \mt) - \fest(\mx, \mt)\mid\mx]\}^2]
\end{align}
Thus the regret of $\thet$, which is defined as $L(\thet) - L(f^*)$, is:
\begin{align}
    R(\thet) &= \E\left[\left\{(f^*(\mx, \mt) - \fest(\mx, \mt) - \E[f^*(\mx, \mt) - \fest(\mx, \mt)\mid\mx]\right\}^2\right]\\
             &= \E[\{(\fest(\mx, \mt) - \fc(\mx, \mt)) - \E[\fest(\mx, \mt) - \fc(\mx, \mt)\mid\mx] \nonumber \\
             &\quad+ (\fc(\mx, \mt) - f^*(\mx, \mt)) - \E[\fc(\mx, \mt) - f^*(\mx, \mt)\mid\mx]\}^2]\\
             &= \E[\{\fest(\mx, \mt) - \fc(\mx, \mt)) - \E[\fest(\mx, \mt) - \fc(\mx, \mt)\mid\mx]\}^2] \nonumber\\
             &\quad+ \E[\{\fc(\mx, \mt) - f^*(\mx, \mt) - \E[\fc(\mx, \mt) - f^*(\mx, \mt)\mid\mx]\}^2] \nonumber \\
             &\quad+ 2\E[\{(\fest(\mx, \mt) - \fc(\mx, \mt)) - \E[\fest(\mx, \mt) - \fc(\mx, \mt)\mid\mx]\} \nonumber \\
             &\quad\cdot \{(\fc(\mx, \mt) - f^*(\mx, \mt)) - \E[\fc(\mx, \mt) - f^*(\mx, \mt)\mid\mx]\}] \label{eq:regret_decomp}
\end{align}
Note that, by definition the c-regret of $\thet$ is just the difference between the regret of $\thet$ and $\thet_c$. And the regret of $\thet_c$ is the second term in \eqref{eq:regret_decomp}. Thus, the c-regret of $\thet$ is the first and third term of \eqref{eq:regret_decomp}.

Now, note that the third term is non-negative because $\hc$ is convex. To see this, note that it is equal to \begin{align}
    \frac{\partial}{\partial \epsilon} R\left(\thec + \epsilon\left(\thet - \thec\right)\right)|_{\epsilon=0},
\end{align} which must be non-negative for any $\thet \in \hc$ since otherwise there will be another point in $\hc$ which has a smaller regret than $\thec$. 

Therefore, 
\begin{align}
R(\thet; c) &\geq \E\left[\left\{\fest(\mx, \mt) - \fc(\mx, \mt)) - \E\left[\fest(\mx, \mt) - \fc(\mx, \mt)\mid\mx\right]\right\}^2\right]\\
&= \E\left[\left \{\boldsymbol{\alpha}(\mx)^{\top}(\thet - \thec)(\boldsymbol{\beta}(\mt)-\propfea)\right\}^2\right]\\
&= \E\left[\left \langle \thet - \thec, \boldsymbol{\alpha}(\mx) \otimes (\boldsymbol{\beta}(\mt) - \propfea) \right \rangle^2\right]
\end{align}
Now, we would like to show that $\E\left[\left \langle \thet - \thec, \boldsymbol{\alpha}(\mx) \otimes \left(\boldsymbol{\beta}(\mt) - \propfea\right) \right \rangle^2\right]$ is bounded below by the norm of $\thet - \thec$ up to some multiplicative constant. We do so using Lemma \ref{lemma: jensen&finitedim}. Under the context of Lemma \ref{lemma: jensen&finitedim}, set $f := \thet - \thec$, and $\alpha := \boldsymbol{\alpha}(\mx) \otimes (\boldsymbol{\beta}(\mt) - \propfea)$. To check that the support of $\alpha$ is not of measure 0, we first note that the support of $\boldsymbol{\alpha}(\mx)$ is not measure 0 by assumption; secondly, the support of $\boldsymbol{\beta}(\mt) - \propfea$ is not measure 0 provided that $P(\boldsymbol{\beta}(T)\mid X)$ is a positive measure for any $X$. Then by Lemma \ref{lemma: jensen&finitedim}, we have that $\exists \epsilon > 0$
\begin{equation}
    R(\thet ; c) \geq \epsilon \|f_{\thet} - f_{\thec}\|_{L_2}
\end{equation}
\end{proof}

Immediately after Lemma \ref{lemma: overlap-extension}, we derive a bound on the infinity norm using the regret function which we will repeatedly use later.

\begin{corollary}
Following from \ref{eq:l2-rkhs-equiv} and Lemma \ref{lemma: overlap-extension}, 
\begin{equation}
    \norm{\thet - \thec}_{\infty} \leq const(p)\norm{f_{\thet} - f_{\thec}}_{\mathcal{H}}^p\norm{f_{\thet} - f_{\thec}}^{1-p}_{L_2} \leq const(p) c^p R(\thet; c)^{\frac{1-p}{2}} \label{eq:infinity_norm}
\end{equation}
where we note that the second inequality follows from combining Lemma \ref{lemma: overlap-extension} with the fact that for $f_{\thet} \in \mathcal{H}_c$, $\norm{f_{\thet} - f_{\thec}} \leq 2c$ by the triangle inequality.
\end{corollary}
\begin{proof}
Immediate from \ref{eq:l2-rkhs-equiv} and Lemma \ref{lemma: overlap-extension}.
\end{proof}

Using Lemma \ref{lemma: overlap-extension}, we can further show that the $L_2$ difference between two constrained optima only depends on the $L_2$ norm of the one with the weaker constraint.
\begin{corollary}
 \label{lemma:rlearnerl6-ext}
Suppose we have overlap, i.e. Assumption \ref{assump:overlap_mainbody}. Then with a positive constant $const. > 0$, the following holds for $1 < c < c'$.
\begin{equation}
    \|\fc - f_{\thet_{c'}}\|_{L_2} \leq const.\|f_{\thet_{c'}}\|_{L_2}
\end{equation}
\end{corollary}
\begin{proof}
We have shown that 
\begin{equation}
    R(\thet; c) \geq \epsilon\|f_{\thet} - f_{\thet_c}\|^2_{L_2}
\end{equation}
Then following \cite{r-learner}, we check that
\begin{align}
    \|\thec - \frac{c}{c'}\thet_{c'}\|^2_{L_2} &\leq \epsilon R(\frac{c}{c'}\thet_{c'}; c)\\
    &= \epsilon \left( L(\frac{c}{c'}\thet_{c'}) - L(\thec) \right)\\
    &\leq \epsilon\left(L(\frac{c}{c'}\thet_{c'}) - L(\thet_{c'})\right)\\
    &= \epsilon\left(R(\frac{c}{c'}\thet_{c'}) - R(\thet_{c'})\right)
\end{align}
To bound $R\left(\frac{c}{c'}\thet_{c'}\right) - R\left(\thet_{c'}\right)$, note
\begin{align}
    R(\thet) &= \E\left[\{(\festxt - f_{\thet_{c'}}(\mx, \mt)) - \E[\festxt - f_{\thet_{c'}}(\mx, \mt)\mid\mx]\}^2\right]\nonumber\\
    &\quad + \E\left[\{f_{\thet_{c'}}(\mx, \mt) - f^*(\mx, \mt) - \E[f_{\thet_{c'}}(\mx, \mt) - f^*(\mx, \mt)\mid\mx]\}^2\right]\nonumber\\
    &\quad + 2\E\Bigg[\left\{(\fest(\mx, \mt) - f_{\thet_{c'}}(\mx, \mt)) - \E\left[\festxt - f_{\thet_{c'}}(\mx, \mt)\mid\mx\right]\right\} \nonumber \\
    &\quad\quad \cdot \left\{(f_{\thet_{c'}}(\mx, \mt) - f^*(\mx, \mt)) - \E\left[f_{\thet_{c'}}(\mx, \mt) - f^*(\mx, \mt)\mid\mx\right]\right\}\Bigg] \label{eq:37}
\end{align}
so $R(\thet_{c'})$ is just the second term of \eqref{eq:37}, which we drop when considering $R(\frac{c}{c'}\thet_{c'}) - R(\thet_{c'})$
\begin{align}
    R\left(\frac{c}{c'}\thet_{c'}\right) - R(\thet_{c'}) &= \E\left[\{(\frac{c}{c'} - 1)f_{\thet_{c'}}(\mx, \mt) - \E[(\frac{c}{c'} - 1)f_{\thet_{c'}}(\mx, \mt)\mid\mx]\}^2\right]\nonumber\\
    &\quad + 2\E\Bigg[\left\{(\frac{c}{c'} - 1)f_{\thet_{c'}}(\mx, \mt) - \E\left[(\frac{c}{c'} - 1)f_{\thet_{c'}}(\mx, \mt)\mid\mx\right]\right\} \nonumber \\
    &\quad\quad \cdot \left\{(f_{\thet_{c'}}(\mx, \mt) - f^*(\mx, \mt)) - \E\left[f_{\thet_{c'}}(\mx, \mt) - f^*(\mx, \mt)\mid\mx\right]\right\}\Bigg]\\
    &= \E \Bigg[ \left\{\boldsymbol{\alpha}(\mx)^{\top}\left(\frac{c}{c'} - 1\right)\thet_{c'}\left(\boldsymbol{\beta}(\mt) - e^p(\mx)\right)\right\}^2\Bigg]\nonumber\\
    &\quad + 2\E\Bigg[ \left\{\boldsymbol{\alpha}(\mx)^{\top}\left(\frac{c}{c'} - 1\right) \thet_{c'} \left(\boldsymbol{\beta}(\mt) - e^p(\mx)\right)\right\} \nonumber \\
    &\quad\quad \cdot \left\{(f_{\thet_{c'}}(\mx, \mt) - f^*(\mx, \mt)) - \E\left[f_{\thet_{c'}}(\mx, \mt) - f^*(\mx, \mt)\mid\mx\right]\right\}\Bigg]
\end{align}
Denote the two terms $E_1$ and $E_2$. By the same argument as Lemma \ref{lemma: jensen&finitedim}, where the Lebesgue-measure-non-zero condition is satisfied by Assumption \ref{assump:overlap_mainbody}, there exist a constant $const. > 0$ such that $E_1 \geq \left(\frac{c}{c'} - 1\right)^2const.\|f_{\thet_{c'}}\|_{L_2} \rightarrow const.\|f^*\|_{L_2}$ as $c' \rightarrow \infty$. But for $E_2$, note that $\|f_{\thet_{c'}}- f^*\|_{L_2} \rightarrow 0$ as $c' \rightarrow \infty$. So $E_2 = o(E_1)$, and under mild conditions there exists a constant $F > 0$ such that for all $c, c'$,
\begin{align}
    R\left(\frac{c}{c'}\thet_{c'}\right) - R(\thet_{c'}) \leq F \E\Bigg[ \left\{\boldsymbol{\alpha}(\mx)^{\top}\thet_{c'} \left(\boldsymbol{\beta}(\mt) - e^p(\mx)\right)\right\}^2\Bigg]
\end{align}. Then note:
\begin{align}
    &\E\Bigg[ \left\{\boldsymbol{\alpha}(\mx)^{\top} \thet_{c'} \left(\boldsymbol{\beta}(\mt) - e^p(\mx)\right)\right\}^2\Bigg] \nonumber\\
    &= \E\Bigg[ \langle\thet_{c'}, \boldsymbol{\alpha}(\mt)\otimes (\boldsymbol{\beta}(\mt) - e^p(\mx))\rangle^2\Bigg]\\
    &= \E\Bigg[\langle\thet_{c'}\otimes \thet_{c'}, (\boldsymbol{\alpha}(\mt)\otimes (\boldsymbol{\beta}(\mt) - e^p(\mx))) \otimes (\boldsymbol{\alpha}(\mt)\otimes (\boldsymbol{\beta}(\mt) - e^p(\mx)))\rangle\Bigg]\\
    &= \langle\thet_{c'}\otimes \thet_{c'}, \E\Bigg[(\boldsymbol{\alpha}(\mt)\otimes (\boldsymbol{\beta}(\mt) - e^p(\mx))) \otimes (\boldsymbol{\alpha}(\mt)\otimes (\boldsymbol{\beta}(\mt) - e^p(\mx)))\Bigg]\rangle\\
    &\leq \norm{\thet_{c'}\otimes \thet_{c'}}\norm{\E\left[(\boldsymbol{\alpha}(\mt)\otimes (\boldsymbol{\beta}(\mt) - e^p(\mx))) \otimes (\boldsymbol{\alpha}(\mt)\otimes (\boldsymbol{\beta}(\mt) - e^p(\mx)))\right]}\label{eq:42}\\
    &= \|\thet_{c'}\|^2 \norm{\underbrace{\E\left[(\boldsymbol{\alpha}(\mt)\otimes (\boldsymbol{\beta}(\mt) - e^p(\mx))) \otimes (\boldsymbol{\alpha}(\mt)\otimes (\boldsymbol{\beta}(\mt) - e^p(\mx)))\right]}_{constant}}\label{eq:43}\\
    &= const. \|f_{\thet_{c'}}\|^2_{\mathcal{H}_{c'}}\label{eq:44}\\
    &= const. \|f_{\thet_{c'}}\|^2_{L_2}\label{eq:45}
\end{align}
where Eq.~\ref{eq:42} is by Cauchy-Schwarz and the \eqref{eq:43} uses the fact that under Euclidean norms for finite dimensional real vectors $\mathbf{a}, \mathbf{b}$, $\|\mathbf{a}\otimes \mathbf{b}\|=\|\mathbf{a}\|\|\mathbf{b}\|$. \eqref{eq:44} is due to the vector 2-norm of $\thet$ is equal to the RKHS norm of $f_{\thet}$, and \eqref{eq:45} is due to the fact that in finite dimensions all norms are Lipschitz equivalent. Note that the constant factors in \ref{eq:44} and \ref{eq:45} may be different but that both positive.

Then finally by the triangle inequality,
\begin{align}
    \|\fc - f_{\thet_{c'}}\|_{L_2} &\leq \|f_{\thet_{c'}} - \frac{c}{c'}f_{\thet_{c'}}\|_{L_2} + \|\fc - \frac{c}{c'}f_{\thet_{c'}}\|_{L_2} \\
    &\leq \left(1-\frac{c}{c'}\right)\|f_{\thet_{c'}}\|_{L_2} + constant.\|f_{\thet_{c'}}\|_{L_2}\\
    &\leq const. \|f_{\thet_{c'}}\|_{L_2}
\end{align}
again for a positive constant factor in the last equality.
\end{proof}

Now we have arrived at the position to bound the difference between the oracle and feasible regrets by functions of the true regret. We first present Lemma \ref{lemma:l8-ext} which bounds the difference between $\hat{R}_n$ and $\tilde{R}_n$ in terms of $R$. Then, we leverage the result by \cite{r-learner} to linearize the dependence on $R$.
\begin{lemma}
\label{lemma:l8-ext}
Suppose that the propensity estimate $\propfeax$ is uniformly consistent, \begin{equation}
    \sup_{\mathbf{x} \in \mathcal{X}}\|\estpropfeax - \propfeax\| \rightarrow_p 0
\end{equation}
and the $L_2$ errors converge at rate
\begin{equation}
    \E\left[\{\estmean - \mean\}^2\right], \E\left[\|\estpropfea - \propfea\|^2\right] = \mathcal{O}(a^2_n) \label{eq:cf_rates}
\end{equation}
for some sequence $a_n \rightarrow 0.$ Suppose, moreover, Assumptions \ref{assump:overlap_mainbody}, \ref{assump:boundedness} and \ref{assump:true_fun_approx} hold. Then, for any $\epsilon > 0$, there exists a constant $U(\epsilon)$ such that the regret functions induced by the oracle learner and the feasible learner are coupled with probability at least $1 - \epsilon$ as 
\begin{align}
\begin{array}{c}
\left|\widehat{R}_{n}(\thet ; c)-\widetilde{R}_{n}(\thet ; c)\right| \leq U(\varepsilon)\left\{c^{p} R(\thet ; c)^{(1-p) / 2} a_{n}^{2}+c^{2 p} R(\thet ; c)^{1-p} \frac{1}{\sqrt{n}} \log (n)\right. \\
+c^{2 p} R(\thet ; c)^{1-p} \frac{1}{n} \log \left(\frac{c n^{1 /(1-p)}}{R(\thet ; c)}\right)+c^{p} R(\thet ; c)^{1-\frac{p}{2}} \frac{1}{\sqrt{n}} \sqrt{\log \left(\frac{c n^{1 /(1-p)}}{R(\thet ; c)}\right)} \\
\left.+c^{p} R(\thet ; c)^{(1-p) / 2} a_{n} \frac{1}{\sqrt{n}} \sqrt{\log \left(\frac{c n^{1 /(1-p)}}{R(\thet ; c)}\right)}+\xi_{n} R(\thet ; c)\right\}
\end{array}
\end{align}
simultaneously for all $1 \leq c \leq \log(n)$.
\end{lemma}

\begin{proof}
Following \cite{r-learner}, we start by decomposing the feasible loss function $\hat{L}_n(\thet)$ into the oracle loss together with additional terms as follows:
\begin{align}
    \hat{L}_n(\thet) &= \frac{1}{n}\sum_{l=1}^n \left( (Y_l - \mhatcf(\mx_l)) - \boldsymbol{\alpha}(\mx_l)^{\top}\thet(\boldsymbol{\beta}(\mt_l) - \ehatcf(\mx_l)) \right)^2\\
    &= \frac{1}{n}\sum_{l=1}^n\left[(Y_l - m^*(\mx_l)) + \{m^*(\mx_l) - \mhat(\mx_l)\} - \boldsymbol{\alpha}(\mx_l)^{\top}\thet(\boldsymbol{\beta}(\mt_l) - e^p(\mx_l))  \right. \nonumber\\ & \hspace{3cm} - \left. \boldsymbol{\alpha}(\mx_l)^{\top}\thet(e^p(\mx_l) - \ehatcf(\mx_l))\right]^2\\
    &= \frac{1}{n}\sum_{l=1}^n\left[\{Y_l - m^*(\mx_l)\} - \boldsymbol{\alpha}(\mx_l)^{\top}\thet(\boldsymbol{\beta}(\mt_l) - e^p(\mx_l))\right]^2 \nonumber\\
    &\quad + \frac{1}{n}\sum_{l=1}^n [\{m^*(\mx_l) - \mhat(\mx_l)\} - \boldsymbol{\alpha}(\mx_l)^{\top}\thet(e^p(\mx_l) - \ehatcf)]^2\nonumber\\
    &\quad + \frac{2}{n} \sum_{l=1}^n\left[\{Y_l - m^*(\mx_l)\} - \boldsymbol{\alpha}(\mx_l)^{\top}\thet(\boldsymbol{\beta}(\mt_l) - e^p(\mx_l))\right] \nonumber \\ 
    &\cdot\left[\{m^*(\mx_l)-\mhatcf(\mx_l)\} - \boldsymbol{\alpha}(\mx_l)^{\top}\thet(e^p(\mx_l) - \ehatcf(\mx_l))\right]\\
    &= \frac{1}{n} \sum_{l=1}^n\left[\{Y_l - m^*(\mx_l)\} - \boldsymbol{\alpha}(\mx_l)^{\top}\thet (\boldsymbol{\beta}(\mt_l) - e^p(\mx_l))\right]^2\nonumber\\
    &\quad +\frac{1}{n}\sum_{l=1}^n\left[\{m^*(\mx_l) - \mhatcf(\mx_l)\} - \boldsymbol{\alpha}(\mx_l)^{\top}\thet(e^p(\mx_l) - \ehatcf(\mx_l))\right]^2\nonumber\\
    &\quad - \frac{2}{n} \sum_{l=1}^n\{Y_l - m^*(\mx_l)\}\boldsymbol{\alpha}(\mx_l)\thet(e^p(\mx_l) - \ehatcf(\mx_l)) \nonumber \\
    &\quad - \frac{2}{n} \sum_{l=1}^n \boldsymbol{\alpha}(\mx_l)^{\top}\thet (\boldsymbol{\beta}(\mt_l) - e^p(\mx_l))\{m^*(\mx_l) - \mhatcf(\mx_l)\}\nonumber\\
    &\quad + \frac{2}{n} \sum_{l=1}^n \boldsymbol{\alpha}(\mx_l)^{\top}\thet (\boldsymbol{\beta}(\mt_l) - e^p(\mx_l))\boldsymbol{\alpha}(\mx_l)^{\top}\thet(e^p(\mx_l) - \ehatcf(\mx_l))
\end{align}
Furthermore, we may verify that some terms cancel out when we restrict our attention to the main objective of interest
\begin{equation}
    \feasr\left(\thet; c\right) - \oracler\left(\thet;c\right) = \hat{L}_n\left(\thet\right) - \hat{L}_n\left(\thet_c\right) - \tilde{L}_n\left(\thet\right) + \tilde{L}_n\left(\thet_c\right)
\end{equation}
In particular, note that the first term in the decomposition above is exactly $\tilde{L}_n(\thet)$. Thus
\begin{align}
    &\feasr(\thet; c) - \oracler(\thet; c)\nonumber\\
    &= -\frac{2}{n} \sum_{l=1}^n \{m^*(\mx_l) - \mhatcf(\mx_l)\}\boldsymbol{\alpha}(\mx_l)^{\top}(\thet - \thet_c)(e^p(\mx_l) - \ehatcf(\mx_l))\nonumber\\
    &\quad + \frac{1}{n}\sum_{l=1}^n \{\boldsymbol{\alpha}(\mx_l)^{\top}\thet(e^p(\mx_l) - \ehatcf(\mx_l))\}^2 - \{\boldsymbol{\alpha}(\mx_l)^{\top}\thec(e^p(\mx_l) - \ehatcf(\mx_l))\}^2 \nonumber\\
    &\quad -\frac{2}{n} \sum_{l=1}^n \{Y_l - m^*(\mx_l)\}\boldsymbol{\alpha}(\mx_l)^{\top}(\thet - \thec)(e^p(\mx_l) - \ehatcf(\mx_l))\nonumber\\
    &\quad -\frac{2}{n} \sum_{l=1}^n \boldsymbol{\alpha}(\mx_l)^{\top}(\thet - \thec)(\boldsymbol{\beta}(\mt_l) - e^p(\mx_l))\{m^*(\mx_l) - \mhatcf(\mx_l)\}\nonumber\\
    &\quad + \frac{2}{n} \sum_{l=1}^n \boldsymbol{\alpha}(\mx_l)^{\top}\thet(\boldsymbol{\beta}(\mt_l) - e^p(\mx_l))\boldsymbol{\alpha}(\mx_l)^{\top}\thet( e^p(\mx_l) - \ehatcf(\mx_l))\nonumber\\
    &\quad - \frac{2}{n} \sum_{l=1}^n \boldsymbol{\alpha}(\mx_l)^{\top}\thec(\boldsymbol{\beta}(\mt_l) - e^p(\mx_l))\boldsymbol{\alpha}(\mx_l)^{\top}\thec( e^p(\mx_l) - \ehatcf(\mx_l))\nonumber\\
\end{align}
Letting $A_1^c(\thet)$, $A_2^c(\thet)$, $B_1^c(\thet)$ and $B_3^c(\thet)$ denote these 5 summands respectively, we seek to bound each of the terms in terms of $R(\thet; c)$. 
Starting with $A^c_1(\thet)$, we extract $\thet - \thec$ by its infinity norm and by Cauchy-Schwarz,
\begin{align}
    |A_1^c(\thet)| &\leq 2\sqrt{\frac{1}{n}\sum_{l=1}^n \left\{m^*(\mx) - \mhatcf(\mx_l)\right\}^2}\nonumber\\
    &\quad \cdot \sqrt{\frac{1}{n}\sum_{l=1}^n\left\|\boldsymbol{\alpha}(\mx_l)\otimes\left(e^p(\mx) - \ehatcf(\mx_l)\right)\right\|^2} \cdot \|\thet - \thec\|_{\infty}\\
\end{align}
Using the fact that $\|\mathbf{a}\otimes \mathbf{b}\| = \|\mathbf{a}\|\|\mathbf{b}\|$ for $\mathbf{a}$ and $\mathbf{b}$ some (finite dimensional) vector, we may separate out the norm of $\boldsymbol{\alpha}(\mx)$ and we know $\|\boldsymbol{\alpha}(\mx)\|^2$ is uniformly bounded by Assumption \ref{assump:boundedness}. By \eqref{eq:cf_rates} and Markov's inequality, the mean squared errors of the $m-$ and $e-$models decay at rate $O_P(a_n)$. Therefore, applying \ref{eq:infinity_norm} to bound the infinity-norm discrepancy $\|\thet - \thec\|_{\infty}$, we find that simultaneously for all $c \geq 1$, 
\begin{equation}
    \sup\{c^{-p}R(\thet; c)^{-\frac{1-p}{2}}|A^c_1(\thet)|: \fest \in \hc, c \geq 1\} = O_P(a^2_n)
\end{equation}
Following \cite{r-learner} and using a similar argument to extract $\|\boldsymbol{\alpha}(\mx)\|$ and bound $\thet - \thec$ by the c-regret $\|R(\thet; c)\|$, we get that 
\begin{equation}
    |A^c_2| = \mathcal{O}_P\left(\left(c^pR(\thet; c)^{\frac{1-p}{2}} + c^{2p}R(\thet; c)^{1-p}\right)a_n^2\right)
\end{equation}
In order to bound $B^c_1(\thet)$, decomposing it with respect to the cross fitting structure, we consider
\begin{equation}
    B_{1,q}^c(\thet) = \frac{\sum_{\{l: q(l)\}}2\{Y - m^*(\mx)\}\boldsymbol{\alpha}(\mx_l)^{\top}(\thet - \thec)(e^p(\mx_l) - \ehatcf(\mx_l))}{|\{l: q(l)=q\}|},
\end{equation}
noting that $|B^c_1(\thet)| \leq \sigma_{q=1}^Q|B_{1,q}^c(\thet)| $. In particular, we bound its supremum $\sup B_{1,q}^c(\thet)$. To proceed, we bound this quantity over sets indexed by $c$ and $\delta$ such that $\|f_{\thet} - f_{\thec}\|_{L^2} \leq \delta$:
\begin{equation}
    \sup_{\thet \in \hc} \left\{B^c_{1,q}(\thet): \|f_{\thet} - f_{\thec}\|_{L^2} \leq \delta\right\}.
\end{equation}
Letting $\noqfold = \{\mx_l, \mt_l, Y_l: q(l) \neq q\}$ denote the set of data points excluded in the $q-$fold, using a similar procedure to \cite{r-learner}, we can check that the conditional expectation $\E\left[B^c_{1,q}\mid\noqfold\right] = 0$.
By conditioning on $\noqfold$, the summands in $B^c_{1,q}(\thet)$ become independent, as $\estpropfea(\mx_l)$ is now only random in $\mx$. 

Now, the next step in \cite{r-learner} is to bound the expectation of the supremum of $B^c_{1,q}$ using \cite[Lemma 5]{r-learner} and \cite[Eq. (36)]{r-learner}. Since we work with a vector of propensity features instead of a single propensity score unlike in \cite{r-learner}, we need to apply \cite[Lemma 5]{r-learner} $d$ times where $d$ is the dimension of $e^p(\mx)$:
\begin{align}
    B_{1,q}^c(\thet) &= \frac{\langle(\thet - \thec), \sum_{\{l: q(l)\}}2\{Y - m^*(\mx)\}\boldsymbol{\alpha}(\mx_l)\otimes(e^p(\mx_l) - \ehatcf(\mx_l))\rangle}{|\{l: q(l)=q\}|}\\
    &= \frac{\sum_{ij}(\thet - \thec)_{ij}, \sum_{\{l: q(l)\}}2\{Y - m^*(\mx)\}\alpha_i(\mx_l)(e^p_j(\mx_l) - \hat{e}^p_{(-q(l)), j}(\mx_l))}{|\{l: q(l)=q\}|},
\end{align}
so 
\begin{equation}
    \sup_{f_{\thet} \in \hc }\{B_{1,q}^c(\thet)\} \leq \frac{\sum_{ij} \sup_{f_{\thet} \in \hc}\sum_{\{l: q(l)\}}2\{Y - m^*(\mx)\}\alpha_i(\mx_l)(e^p_j(\mx_l) - \hat{e}^p_{(-q(l)), j}(\mx_l))(\thet - \thec)_{ij}}{|\{l: q(l)=q\}|}
\end{equation}
So bounding each term indexed by $ij$ using Lemma 5 of \cite{r-learner} and \eqref{eq:cf_rates}, we will get the same bound as in \cite{r-learner} because the sum over $ij$ is finite and $\boldsymbol{\alpha}$ is bounded. 

Then, using a similar argument to \cite{r-learner}, we may obtain that for any fixed $c, \delta, \epsilon > 0$, there exists a different constant B such that with probability at least $1 - \epsilon$, 
\begin{align}
\begin{aligned}
&\sup _{\tau \in \mathcal{H}_{c}}\left\{B_{1, q}^{c}(\tau) \mid \mathcal{I}^{(-q)}:\left\|f_{\thet}-f_{\thec}\right\|_{L_{2}} \leq \delta\right\} \\
&\quad <B\left\{c^{p} \delta^{1-p} a_{n} \frac{\log (n)}{\sqrt{n}}+\frac{c^{p} \delta^{1-p} a_{n}}{\sqrt{n}} \sqrt{\log \left(\frac{1}{\varepsilon}\right)}+\frac{1}{n} c^{p} \delta^{1-p} \log \left(\frac{1}{\varepsilon}\right)\right\}
\end{aligned},
\end{align}
which holds unconditionally of $\noqfold$.
In order to establish the bound for all values of $c$ and $\delta$ simultaneously, we may proceed with the same argument as \cite{r-learner}; instead of \cite[Lemma 6]{r-learner}, we replace with our Lemma \ref{lemma:rlearnerl6-ext}, which is our extension to the multidimensional setting. 
$B_2^c(\thet)$ may be bounded similarly. 

To bound $B_3(\thet)$, the argument of \cite{r-learner} is easily extended as well, using the decomposition which we detail below.

To simplify notation, write 
\begin{align}
\mathbf{a}_l &= \boldsymbol{\alpha}(\mx_l) \otimes \left(\boldsymbol{\beta}(\mt_l) - e^p(\mx_l)\right)\\
\mathbf{b}_l &= \boldsymbol{\alpha}(\mx_l) \otimes \left(e^p(\mx_l) - \ehatcf(\mx_l)\right)    
\end{align}
Note: 
\begin{align}
    B^c_3 &= \frac{2}{n}\sum_{l=1}^n\langle \thet,\mathbf{a}_l\rangle \langle \thet, \mathbf{b}_l\rangle  -\frac{2}{n}\sum_{l=1}^n \langle \thec, \mathbf{a}_l \rangle \langle \thec, \mathbf{b}_l \rangle\\
    &= \frac{2}{n}\sum_{l=1}^n \Bigg\{ 2\langle \thet,\mathbf{a}_l\rangle \langle \thet, \mathbf{b}_l\rangle - \langle \thet,\mathbf{a}_l\rangle \langle \thet, \mathbf{b}_l\rangle \nonumber\\
    &\quad\quad\quad\quad- \langle \thec,\mathbf{a}_l\rangle \langle \thet, \mathbf{b}_l\rangle + \langle \thec,\mathbf{a}_l\rangle \langle \thet, \mathbf{b}_l\rangle\nonumber\\ &\quad\quad\quad\quad-\langle \thet,\mathbf{a}_l\rangle \langle \thec, \mathbf{b}_l\rangle +\langle \thet,\mathbf{a}_l\rangle \langle \thec, \mathbf{b}_l\rangle \nonumber\\
    &\quad\quad\quad\quad- \langle \thec,\mathbf{a}_l\rangle \langle \thec, \mathbf{b}_l\rangle \Bigg\}\\
    &= \frac{2}{n}\sum_{l=1}^n \Bigg\{
    \langle \thet - \thec,\mathbf{a}_l\rangle \langle \thet, \mathbf{b}_l\rangle
    +\langle \thet,\mathbf{a}_l\rangle \langle \thet-\thec, \mathbf{b}_l\rangle \\
    &\quad\quad\quad
    - \langle \thet - \thec,\mathbf{a}_l\rangle \langle \thet - \thec,\mathbf{b}_l\rangle \Bigg\}\\
    &\leq \left| \frac{2}{n}\sum_{l=1}^n 
    \langle \thet - \thec,\mathbf{a}_l\rangle \langle \thet, \mathbf{b}_l\rangle \right|\nonumber\\
    &\quad + \left| \frac{2}{n}\sum_{l=1}^n \langle \thet,\mathbf{a}_l\rangle \langle \thet-\thec, \mathbf{b}_l\rangle 
     \right|\nonumber\\
    &\quad + \frac{2}{n}\sum_{l=1}^n \| \thet - \thec\|^2_2\|\mathbf{a}_l\|_2\|\mathbf{b}_l\|_2
\end{align}
where the last term of the last inequality follows by Cauchy-Schwarz.

The first two terms can be bounded similarly to the argument used for bounding $B_1^c(\thet)$. For the last term, we note that $\|\thet - \thec\|_2 = \|\fest - \fc\|_{L_2}$ since by construction the RKHS norm and the $L_2$ norms are equal. Therefore, the last term is bounded by $\xi_n \|\fest - \fc\|_{L_2}$ where 
\begin{align}
    \xi_n = \|\boldsymbol{\alpha}(\mx_l)\otimes (\boldsymbol{\beta}(\mt_l) - e^p(\mx_l))\|_{\infty}\|\boldsymbol{\alpha}(\mx_l)\otimes (e^p(\mx_l) - \ehatcf(\mx_l))\|_{\infty} = o(1).
\end{align} Note that we do not need the lower order terms present in \cite{r-learner} which followed from \cite[Lemma 7]{r-learner}.

Thus the desired result follows.
\end{proof}

By \cite[Lemma 2]{r-learner}, Lemma \ref{lemma:l8-ext} implies that under Assumptions \ref{assump:boundedness} to \ref{assump:true_fun_approx}, and the conditions in Lemma \ref{lemma:l8-ext} and Lemma \ref{lemma:risk_bound}, where the $(a_n)$ in Lemma \ref{lemma:l8-ext} is such that $a_n = O(n^{-\kappa})$ with $\kappa > \frac{1}{4}$, then
\begin{equation}
    \left|\hat{R}_n(\thet; c) - \tilde{R}_n(\thet; c)\right| \leq 0.125 R(\thet; c) + o(\rho_n(c)) \label{eq:rlearner-lem2-result}
\end{equation}
with probability at least $1-\epsilon$, for all $\thet \in \hc$, $1 \leq c \leq \log(n)$ for large enough n.

Thus we have finally bridged $\hat{R}_n$ and $\tilde{R}_n$ with respect to the expected regret $R$. We are ready to prove our main theorem which concerns the regret bound of $\hat{R}_n$.

\subsection{Using the bridge result to derive feasible regret bound}
\begin{customthm}{\ref{theorem:main}}
Under Assumptions \ref{assump:overlap_mainbody}, \ref{assump:boundedness}, \ref{assump:true_fun_approx} and the conditions in Lemma \ref{lemma:l8-ext} and Lemma \ref{lemma:risk_bound}, where the $(a_n)$ in Lemma \ref{lemma:l8-ext} is such that $a_n = O(n^{-\kappa})$ with $\kappa > \frac{1}{4}$, and suppose that we obtain $\hat \thet$ via a penalized basis function regression variant of the generalized R-learner, with a properly chosen penalty of the form $\Lambda_n(\|\hat \thet\|_2)$ that grows faster than $\rho_n(\|\hat \thet\|_2)$ in \ref{eq:rho} . Then $\hat \thet$ satisfies the same regret bound as $\tilde \thet$, $R(\hat{\thet}_n) = \tilde{O}(n^{\frac{1}{1+p}})$.
\end{customthm}
\begin{proof}
We have established that when we set $\rho_n$ as 
\begin{align}
\rho_{n}(c) = U(\epsilon) \{1+\log (n)+\log \log (c+e)\}\left(\frac{(c+1)^{p} \log (n)}{\sqrt{n}}\right)^{2 /(1+p)} \nonumber,
\end{align}
we have that for every $\epsilon$ there exist a constant $U(\epsilon)$ such that for large enough $n$ the following is satisfied with probability at least $1-\epsilon$:

\begin{align}
    \frac{1}{2}\tilde{R}_n(\fest;c) - \rho_n(c) \leq R_n(\fest; c) \leq 2\tilde{R}_n(\fest; c) + \rho_n(c) \label{eq:quasi-iso}
\end{align}
Subsection \ref{subsec:oracle_rate} argued that this leads to a rate of $\tilde{\mathcal{O}}(n^{-\frac{1}{1+p}})$ for $R(\tilde{\thet})$.

Now to show that feasible learner matches the rate of the oracle learner, 

Eq. \ref{eq:rlearner-lem2-result} implies that 
\begin{align}
    R(\thet; c) &\leq 2 \tilde{R}_n(\thet; c) + \rho_n(c)\\
    &\leq 2 \hat{R}_n(\tau; c) + 0.25k R(\tau; c) + k\rho_n(c)
\end{align}
Rearranging the inequality implies that 
\begin{align}
    R(\thet; c) \leq \frac{8}{3}\hat{R}_n(\thet; c) + 2\rho_n(c)
\end{align}
for large $n$ for all $1< c< \log(n)$, with probability at least $1-2\epsilon$. It can then be checked following a symmetrical argument, that 
\begin{align}
    \frac{3}{8}\hat{R}_n(\thet;c) - 2\rho_n(c) \leq R(\thet; c) \leq \frac{8}{3}\hat{R}_n(\thet; c) + 2 \rho_n(c)
\end{align}
for $n$ large enough for all $1\leq c \leq \log(n)$ with probability at least $1-4\epsilon$. 

Then, following the same argument as \cite{r-learner}, we find that the feasible minimizer has the same regret bound as the oracle minimizer: $R(\hat{\thet}_n)=\tilde{\mathcal{O}}\left(n^{-\frac{1}{1+p}}\right).$

This is to say:
\begin{empheq}[box=\eqbox]{align}
R(\hat{\thet}_n) = O(r_n^2),\; r_n = n^{-\frac{1}{2(1+p)}}
\end{empheq}
\end{proof}

\end{document}